\documentclass[journal,twoside,web]{ieeecolor}
\usepackage{generic}
\usepackage{cite}
\usepackage{amsmath,amssymb,amsfonts}
\usepackage{algorithmic}
\usepackage{graphicx}
\usepackage{algorithm,algorithmic}
\usepackage{hyperref}
\usepackage{textcomp}
\usepackage{subcaption}
\usepackage{booktabs}

\newtheorem{asm}{Assumption}	%Assumption
\newtheorem{Rem}{Remark}	%Theorem
\newtheorem{Def}{Definition}	
\newtheorem{Thm}{Theorem}	%Theorem
\newtheorem{Lem}{Lemma}	%Lemma
\def\BibTeX{{\rm B\kern-.05em{\sc i\kern-.025em b}\kern-.08em
		T\kern-.1667em\lower.7ex\hbox{E}\kern-.125emX}}
\begin{document}
	\title{Output-Feedback Boundary Control of Thermally and Flow-Induced Vibrations in Slender Timoshenko Beams}
	\author{Chengyi Wang and Ji Wang
		\thanks{C. Wang and J. Wang are with the School of Aerospace Engineering, Xiamen University, Xiamen 361102, P. R. China (e-mail: 23220231151773@stu.xmu.edu.cn, jiwang@xmu.edu.cn). }}
	%\thanks{Third C. Author is with 
		%the Electrical Engineering Department, University of Colorado, Boulder, CO 
		%80309 USA, on leave from the National Research Institute for Metals, 
		%Tsukuba, Japan (e-mail: author@nrim.go.jp).}}

\maketitle

\begin{abstract}
	This work is motivated by the engineering challenge of suppressing vibrations in turbine blades of aero engines, which often operate under extreme thermal conditions and high-Mach aerodynamic environments that give rise to complex vibration phenomena, commonly referred to as thermally-induced and flow-induced vibrations. Using Hamilton’s variational principle, the system is modeled as a rotating slender Timoshenko beam under thermal and aerodynamic loads, described by a coupled system of $2\times2$ hyperbolic PIDEs, parabolic PDE, and ODEs, where the nonlocal terms exist in the hyperbolic PDE domain, and where the external disturbance (heat flux) flows into one boundary of the heat PDE. For the general form of such mixed systems, we present the state-feedback control design based on the PDE backstepping method, and then design an extended state observer for the unmeasurable distributed states and external disturbances using only available boundary measurements. In the resulting output-feedback closed-loop system, the state of the uncontrolled boundary, i.e., the furthest state from the control input, is proved to be exponentially convergent to zero, and all signals are proved to be uniformly ultimately bounded. Moreover, if the external disturbance vanishes, the exponential stability of the overall system is obtained. The proposed control design is validated on an aero-engine flexible blade under extreme thermal and aerodynamic conditions.
\end{abstract}


%\label{sec:introduction}
%\IEEEPARstart{T}{his} document is a template for \LaTeX. If you are 
%reading a paper or PDF version of this document, please download the 
%template from the IEEE Web site at \underline
%{http://ieeeauthorcenter.ieee.org/create-your-ieee-article/
	%use-}\discretionary{}{}{}\underline{authoring-tools-and-ieee-article-templates/ieee-articletemplates/} so you can use it to prepare your manuscript. If 
%you would prefer to use \LaTeX, download IEEE's \LaTeX style and sample files 
%from the same Web page. You can also explore using the Overleaf editor at 
%\underline
%{https://www.overleaf.com/blog/278-how-to-use-overleaf-with-}\discretionary{}{}{}\underline
%{ieee-collabratec-your-quick-guide-to-getting-started\#.}\discretionary{}{}{}\underline{xsVp6tpPkrKM9}
%
%If your paper is intended for a conference, please contact your conference 
%editor concerning acceptable word processor formats for your particular 
%conference. 
%
%IEEE will do the final formatting of your paper. If your paper is intended 
%for a conference, please observe the conference page limits. 
\section{Introduction}
\subsection{Motivation}
The turbine blades in aero-engines often operate under extreme thermal
conditions and high-Mach aerodynamic environments that lead to complex vibration behaviors. These vibrations, commonly known as thermally-induced vibrations and flow-induced vibrations, are key challenges for modern aeronautical engineering. 

Thermally-induced vibrations arise primarily due to temperature gradients across the flexible structure, causing the thermal moments that lead to oscillatory behavior \cite{bib6}, \cite{bib1}, \cite{bib2}, \cite{bib3}. Flow-induced vibrations, on the other hand, occur when aerodynamic forces interact with the flexible structure. In the case of aero engines, high-speed airflow over the surface of the turbine blades generates unsteady pressures, often resulting in aeroelastic instability \cite{bib00}. The flow-induced and thermally-induced vibrations can lead to material fatigue, crack propagation, and eventual failure of aero-engine blades, compromising safety, efficiency, and reliability. 

This work is motivated by suppressing the flow-induced and thermally-induced vibrations in the turbine blades of aero engines, where, by Hamilton’s variational
principle, the thermal blade is modeled as a mixed system of a slender Timoshenko beam and a heat PDE, and where the flow pressure acting on the blade is approximated as in-domain instability sources according to ''piston theory'' \cite{bib49}. 

\subsection{Boundary control of beams}
Several approaches have been developed for the boundary control of beams. A boundary damper method requiring spatial and temporal derivatives at the beam tip was introduced in \cite{bib11}. To address the challenge of measuring tip velocity, a more sophisticated dynamic feedback approach was later proposed in \cite{bib12}. In \cite{bib13}, an energy shaping controller was developed, leveraging tip actuation for stabilization. Additionally, \cite{bib14} introduced an output-feedback control law that employed a Lyapunov-based framework combined with a disturbance observer. The work in \cite{bib15} achieved two-dimensional robust output tracking for a Timoshenko beam by utilizing an observer-based error feedback control strategy. Further, \cite{bib10} explored stabilizing a rotating Timoshenko beam through force application at the free end and torque at the pivoted end, while \cite{bib16} demonstrated stabilization of a rotating beam system requiring only torque control at the pivoted end. As a notable application, \cite{bib17} successfully stabilized a one-link flexible arm interacting with a soft environment by applying a damping force and a carefully designed controller.

The backstepping method, recognized as a powerful tool for controlling infinite-dimensional systems, was first applied to beam equations in \cite{bib18}, \cite{bib19}. This method was later extended in \cite{bib20} to design an output feedback control strategy for a slender Timoshenko beam model, where actuation occurs at the beam’s base and sensing at the tip. Some studies have focused on controlling the shear beam and Euler-Bernoulli beam, as evidenced in \cite{bib21}, \cite{bib22}. Expanding on these foundations, \cite{bib24} presented a boundary control design for coupled hyperbolic equations with nonlocal terms, demonstrating its applicability to stabilizing the shear beam model. A significant development in the field was made in \cite{bib23}, where a Riemann transformation was introduced to reformulate the Timoshenko beam as a $(2+2) \times (2+2)$ system of first-order transport equations, enabling rapid stabilization of a Timoshenko beam exhibiting anti-damping and anti-stiffness effects at the uncontrolled boundary. The aforementioned results have not taken into account the influence of thermal dynamics. Using the Riemann transformation, a thermal Timoshenko beam can be reformulated as a mixed system consisting of parabolic PDEs and coupled hyperbolic PDEs, which is introduced next.
\subsection{Control of coupled hyperbolic-parabolic PDEs}
The control design for this problem is particularly challenging due to the fundamentally different dynamic characteristics of the coupled subsystems. The initial application of the backstepping method to mixed PDE systems was introduced in \cite{bib25}, where a boundary control strategy was developed for an unstable reaction-diffusion PDE coupled with an arbitrarily long delay modeled as a transport PDE. The control of the two types of PDEs coupled within their spatial domain, presenting new theoretical and technical difficulties, was initially explored in \cite{bib26,bib27}. Furthermore, the authors of \cite{bib28} designed a backstepping-based controller aimed at stabilizing a class of hyperbolic-parabolic PDE systems that include interior mixed-coupling terms. A more general class of mixed PDE systems was subsequently addressed in \cite{bib29}, where state feedback controllers and observers were constructed using an advanced multi-step backstepping procedure. Moreover, the control of the mixed PDE systems involving the domain coupling of the parabolic PDE is addressed in \cite{bib53} through a combined strategy based on the backstepping method and spectral-reduction techniques.
\subsection{Main Contribution}

1) In contrast to \cite{bib24}, which addresses full-state feedback control of coupled hyperbolic PDEs with nonlocal terms, we focus on output-feedback control for these types of hyperbolic PDEs, which are additionally coupled with a parabolic PDE.

2) Compared with \cite{bib14}, our control design effectively tackles the instabilities arising in the spatial domain and the uncontrolled boundary of the beams. It is, therefore, capable of addressing aeroelastic instability that occurs in severe high-speed conditions and suppressing the thermally-induced vibrations distributed along the beam in extreme thermal environments, such as those encountered in aero engines.

3) In comparison to previous studies about backstepping control of Timoshenko beam---the case with a small Kelvin-Voigt damping or undamping in \cite{bib18}, and the case with
antidamping and antistiffness at the uncontrolled boundary in \cite{bib28}, we investigate a more challenging scenario where the beam is subject to thermal moments and aerodynamic pressure. This means that the anti-damping and anti-stiffness, as well as heat PDE states, are present in the beam PDE domain. 

4) Compared with the existing work on boundary control on hyperbolic-parabolic mixed PDEs \cite{bib53,bib26,bib27}, \cite{bib28}, \cite{bib29}, physically motivated by the problem of suppressing thermally-induced vibrations in
	slender Timoshenko beams, our work introduces the novel feature of incorporating both nonlocal terms and ODE states within the PDE spatial domain, broadening the problem setting of hyperbolic-parabolic mixed PDEs and posing new theoretical challenges that have not been addressed in the prior literature. To the best of our knowledge, this is the first result about boundary control of a class of mixed systems consisting of hyperbolic PIDE, parabolic PDE, and ODEs. The theoretical result is new, even if the disturbance rejection and the observer design are not incorporated.
\subsection{Organization}
The rest of the article is organized as follows. The problem formulation is presented in Section \ref{problemf}. State-feedback control design is proposed in Section \ref{staefd}. The design of the observer is shown in Section \ref{observerds}. The result of the output-feedback closed-loop system is given in Section \ref{outputfd}. The numerical simulation results are shown in Section \ref{simulation}. Section \ref{conclusion} provides the conclusion and future work.
\subsection{Notation}
\begin{itemize}
	\item  Let $U \subseteq \mathbb{R}^n$ be an open set and let $\Omega \subseteq \mathbb{R}$ be a set. By $C^0(U;\Omega)$, we denote the class of continuous mappings on $U$, which take values in $\Omega$. By $C^k(U;\Omega)$, where $k \geq 1$, we denote the class of continuous functions on $U$, which have continuous derivatives of order $k$ on $U$ and take values in $\Omega$. Also, by $L^{\infty}(U,\Omega)$, we denote the class of (Bochner) measurable mappings on $U$, whose $\Omega$-norms are essentially bounded.
	\item We use the notation $L^2(0,1)$ for the standard space of the equivalence class of square-integrable, measurable functions $f:[0,1]\rightarrow \mathbb{R}$, with $\| f \|^2:=\int_{0}^{1}f(x)^2dx<+\infty$ for $f\in L^2(0,1)$. For an integer $k \geq 1$, $H^k(0,1)$ denotes the Sobolev space of functions in $L^2(0,1)$ with all its weak derivatives up to order $k$ in $L^2(0,1)$, with $\| f\|^2_{H^k}=\int_{0}^{1}(f(x)^2+f^{(1)}(x)^2+\cdot\cdot\cdot+f^{(k)}(x)^2)dx$.
	\item The notation $|\cdot|$ denotes Euclidean norm. The notation $\dot{z}(t)$ denotes the time derivative of $z$. The notation $c^{(i)}(x)$ denote the i times derivatives of $c$.
	\item Suppose that $f(x)$, $F(x,y)$ are functions defined in domain $\mathcal{D}$ where $\mathcal{D}=\lbrace (x,y)\in \mathbb{R}^2 \vert 0 \leq y \leq x \leq 1 \rbrace$, then the norm $| \cdot |_{\infty}$ is defined by $| f(x) |_{\infty}=\sup\limits_{x\in\mathcal{D}}\lvert f(x) \lvert$ and the norm $\| \cdot\|_{\infty}$ is defined by $\| F(x,y)\|_{\infty}=\sup\limits_{x,y\in\mathcal{D}}|F(x,y)|$.
\end{itemize}
\section{Rotating Slender Timoshenko Beams in Thermo-Aerodynamic Environments}\label{problemf}
The flexible blade in the aero-engine is under extreme thermal and aerodynamic conditions that generate thermally-induced and flow-induced vibrations. The control objective is to quickly suppress the overall vibration energy of the flexible blade and make sure the blade tip moves as the given reference, i.e., the vibration displacement of the tip quickly converges to zero, which is crucial for maintaining smooth airflow and improving the engine's operational efficiency.

\begin{figure}
	\centerline{\includegraphics[width=\columnwidth]{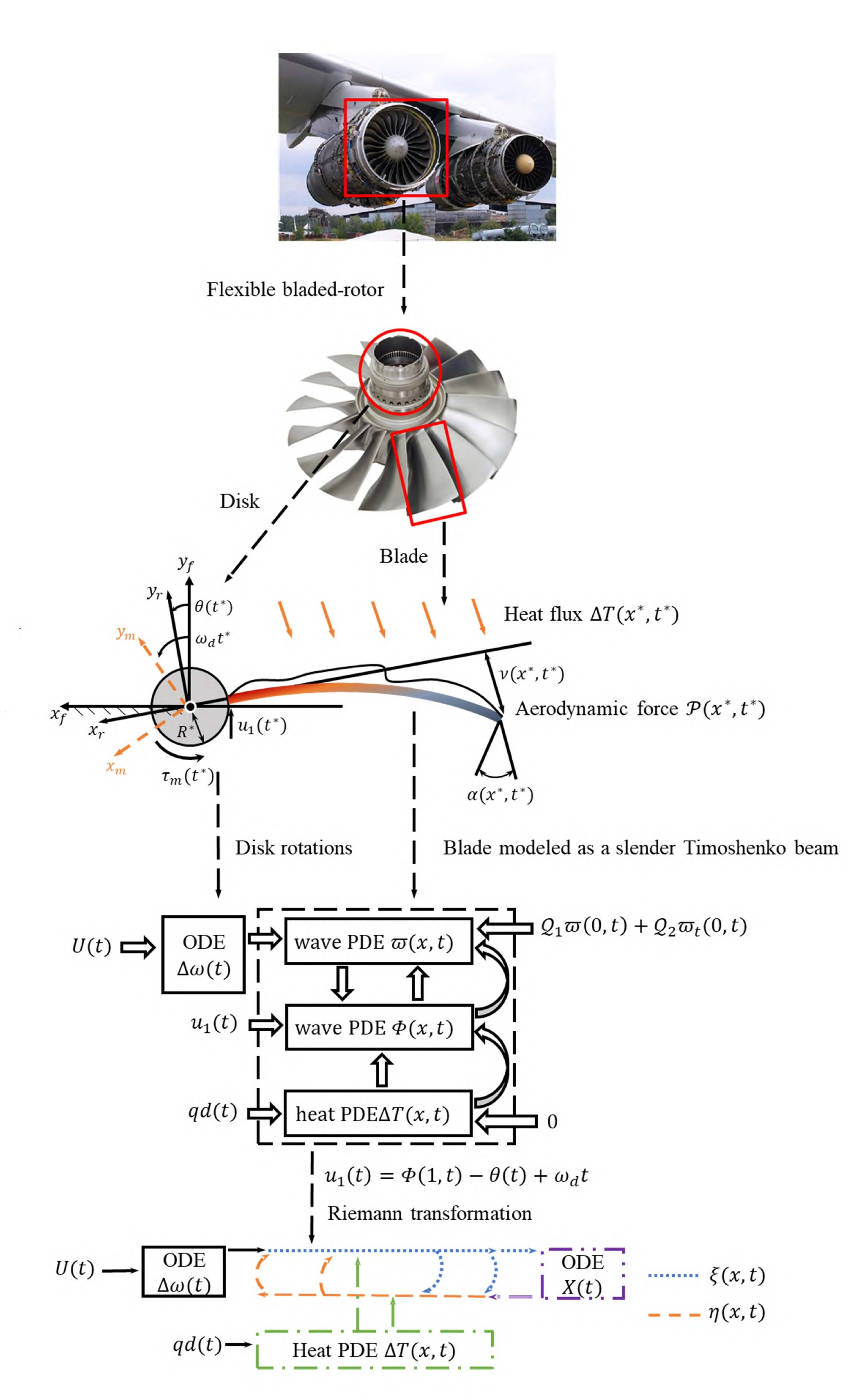}}
	\caption{Rotating turbine blade in aero-engine: from the physical model to the mathematical plant.}
	\label{Fg1}
\end{figure}
Fig.\ref{Fg1} illustrates a rotating turbine blade in thermo-aerodynamic environments. One end of the blade is fixed to the disk that is driven by the control input. Besides, the overall blade bears unstable aerodynamic forces and thermal moments. The $x_fy_f$ coordinate system is a fixed frame, the $x_my_m$ coordinate system is a moving frame rotating with the disk, and the $x_ry_r$ coordinate system is a reference frame that rotates as the reference trajectory whose angular velocity is $w_d$.
\subsection{Modeling}
Here, the model is derived using Hamilton's variational principle. The potential energy of the flexible blade due to bending and temperature distribution is expressed as follows
\begin{align}
	PE^\ast_{bending}=&\frac{1}{2} \int_{0}^{L^\ast} E^\ast I^\ast \bigg(\frac{\partial \alpha(x^{\ast},t^\ast)}{\partial x^{\ast}}\bigg)^2dx^\ast \notag \\
	&+\frac{1}{2}\int_{0}^{L^\ast}E^\ast M_T(x^\ast,t^\ast) \frac{\partial \alpha(x^\ast,t^\ast)}{\partial x^\ast}dx^\ast 	\label{Pbend}
\end{align}
where $L^\ast$ is the length of the blade, $E^\ast$ is the modulus of elasticity, $I^\ast$ is the area moment of inertia of the cross-section about the neutral axis $x$, $\alpha(x^{\ast},t^\ast)$ is the angle of rotation of the cross-section due to the bending moment. The function $M_T$ given by
$$	M_T(x^\ast,t^\ast)= \Delta T(x^\ast,t^\ast)\int_{0}^{L_w^\ast}dz\int_{0}^{L_h^\ast}\alpha_0 y dy$$
%\end{align}
represents the thermal moment \cite{bib1}, \cite{bib6}, where \(\alpha_0\) is the thermal expansion coefficient, while \(L_w^\ast\) and \(L_h^\ast\) are the height and width of the blade, respectively. The function $\Delta T(x^\ast,t^\ast)$ is the temperature change.

The potential energy due to shear is given by
\begin{align}
	PE^\ast_{shear}=\frac{1}{2}\int_{0}^{L^\ast}k'G^\ast A^\ast\bigg(\frac{\partial \upsilon^\ast(x^{\ast},t^\ast)}{\partial x^\ast}-\alpha(x^{\ast},t^\ast)\bigg)^2 dx^\ast \label{Pshear}
\end{align}
where $\upsilon(x^{\ast},t^\ast)$ is the transverse displacement in the moving frame $x_my_m$, $A^\ast$ is the cross sectional area of the blade, $k'$ is the shear factor, $G^\ast$ is the shear modulus.

The kinetic energy due to the transverse deflection is expressed as follows
\begin{align}
	KE^*_{trans}=&\frac{1}{2} \int_{0}^{L^*} \rho^\ast A^\ast \bigg(\frac{\partial\theta(t^\ast)}{\partial t^\ast} (L^\ast+R^\ast-x^\ast)\notag \\
	&+\frac{\partial \upsilon^\ast(x^\ast,t^\ast)}{\partial t^\ast}\bigg)^2 dx^\ast\label{Ptrans}
\end{align}
where $\rho^*$ is the density of the beam, $\theta(t^\ast)$ and $R^\ast$ are the rotational angle of the blade and the radius of the rigid disk, respectively.

Additionally, the kinetic energy arising from the rotation of the cross-section is given by
\begin{align}
	KE^*_{rot}=\frac{1}{2} \int_{0}^{L^*} \rho^* I^*\bigg(-\frac{\partial\theta(t^\ast)}{\partial t^\ast} +\frac{\partial \alpha(x^\ast,t^\ast)}{\partial t^\ast}\bigg)^2 dx^\ast.\label{Pke}
\end{align}
The Lagrangian, defined by kinetic energy-potential energy, is obtained as follows
\begin{align}
	L=KE_{trans}^\ast
	+KE_{rot}^\ast-PE_{bending}^\ast-PE_{shear}^\ast+J^{\ast}\omega^2\label{eq:PL}
\end{align}
where $\omega$ and $J^\ast$ are the angular velocity of the moving reference $x_my_m$ and the inertia moment of the disk, respectively.

The virtual work is
\begin{align}
	\delta W=&u_1^{\ast}(t^\ast)\delta(\alpha(L^\ast,t^\ast))+(\tau_m^{\ast}(t^\ast)+c^{\ast}\dot{\theta}(t^\ast))\delta(\theta(t^\ast))\notag\\&+\int_{0}^{L^\ast} \mathcal{P}^{\ast}(x^\ast,t^\ast)\delta \upsilon^\ast(x^\ast,t^\ast) dx^\ast\label{eq:vw}
\end{align}
where $c^{\ast}$ is the damping coefficient, $\tau_m^{\ast}(t)$ is torque of the disk, $\mathcal{P}^{\ast}(x^\ast,t^\ast)$ is the fluid pressure acting on the blade when flow passing over it, and $u_1^{\ast}(t^\ast)$ is the input signal applied at the edge of the rigid disk, which is implemented with piezoelectric actuators \cite{bib43}.

Introducing the following dimensionless parameters
\begin{align}
	&x=\frac{x^\ast}{L^\ast},\quad I=\frac{I^\ast}{{L^\ast}^4},\quad t=t^\ast \omega_0^\ast,\quad A=\frac{A^\ast}{{L^\ast}^2},\quad  \upsilon=\frac{\upsilon^\ast}{L^\ast}, \notag \\
	&R=\frac{R^\ast}{L^\ast},\quad G=\frac{G^\ast {L^\ast}^4}{E^\ast I^\ast},\quad \rho=\frac{\rho^\ast {L^\ast}^6 {\omega_1^\ast}^2}{E^\ast I^\ast}, \quad I_0=\frac{I_0^\ast}{I^\ast}, \notag \\
	&u_1(t)=\frac{u_1^{\ast}L^{\ast}}{E^{\ast}I^{\ast}},\quad \tau_m=\frac{\tau_m^{\ast}L^{\ast}}{E^{\ast}I^{\ast}},\quad c=\frac{c^{\ast}L^{\ast}}{E^{\ast}I^{\ast}},\quad \mathcal{P}=\frac{\mathcal{P^\ast}{L^{\ast}}^3}{E^{\ast}I^{\ast}}
\end{align}
where $\omega_0^\ast$ is the first natural frequency and $I_0^\ast=\iint_{A^\ast}^{}\alpha_0 y dA^\ast$, applying Hamilton’s principle with the energy equations $L$ in \eqref{eq:PL} that consists of  \eqref{Pbend}--\eqref{Pke}, as well as the virtual work $\delta W$ in \eqref{eq:vw}, the dynamic model is derived as follows
\begin{align}
	&\rho A((1+R-x)\ddot{\theta}(t)+\upsilon_{tt}(x,t)) \notag \\
	&+k'GA(\alpha_x(x,t)-\upsilon_{xx}(x,t))=\mathcal{P}(x,t), \label{beam0} \\
	&\rho I(-\ddot{\theta}(t)+\alpha_{tt}(x,t))-\alpha_{xx}(x,t)\notag \\
	&+k'GA(\alpha(x,t)-\upsilon_x(x,t))=I_0 \Delta T_x(x,t),  \label{beam1} \\
	&J\dot{\omega}(t)=2c\omega(t)+2\tau_m(t) \label{beam2} 
	%  	&\Delta T_t(x,t)=\Delta T_{xx}(x,t) \label{thd}
\end{align}
with the boundary conditions to be satisfied 
\begin{align}
	&\alpha_x(1,t)=u_1(t),\quad \alpha_x(0,t)=I_0 \Delta T(0,t), \label{alp} \\
	&\upsilon(1,t)=0,\quad \upsilon_x(0,t)=\alpha(0,t). \label{htb} 
\end{align}
\subsection{Reformulation}
To rewrite the system \eqref{beam0}--\eqref{beam2}  into a form suitable for control design, we introduce new variables to place the model \eqref{beam0}--\eqref{htbd} in a reference frame $x_ry_r$ that aligns with the reference trajectory we want to track:  
\begin{align}
	\varpi(x,t)=&\upsilon(x,t)+(1+R-x)(\theta(t)-\omega_dt), \label{up0om} \\
	\Phi(x,t)=&\alpha(x,t)-(\theta(t)-\omega_dt), \label{up1om} \\
	\Delta \omega(t)=& \omega(t)-\omega_d, \label{up3om} \\
	\epsilon=&\frac{\rho}{k'G},\quad \mu=\rho I,\quad a=A\rho \label{up2om}
\end{align}
where the constant $\omega_d$ is the reference angle velocity. Then \eqref{beam0}--\eqref{beam2}  is rewritten as
\begin{align}
	\epsilon \varpi_{tt}(x,t)-\varpi_{xx}(x,t)=&-\Phi_x(x,t)+k_1\varpi_t(x,t)\notag \\
	&+k_2\varpi_x(x,t), \label{st0}\\
	\mu\Phi_{tt}(x,t)-\Phi_{xx}(x,t)=&-\frac{a}{\epsilon}(\Phi(x,t)-\varpi_x(x,t))\notag \\
	&+I_0 \Delta T_x(x,t), \label{st1}\\
	J\Delta\dot{\omega}(t)=&2c\Delta\omega(t)+U(t) \label{st3}  
\end{align}for $x\in(0,1)$, $t\geq0$, and $U(t)=2(\tau_m(t)+c\Delta\omega_d)$. Here, we use the popular "piston theory" \cite{bib49} to treat fluid-structure interaction, expressing the flow-induced pressure
\begin{align}\mathcal{P}=k_1^{\ast}\varpi_t(x,t)+k_2^{\ast}\varpi_x(x,t)\label{eq:pressure}\end{align}  where $k_1^{\ast}$, $k_2^{\ast}$ satisfy $k_1=\frac{k_1^{\ast}}{k'GA}$, $k_2=\frac{k_2^{\ast}}{k'GA}$, and are determined by Mach number, the free-stream density of the fluid and the velocity of the fluid \cite{bib47}, which presents the
anti-damping and anti-stiffness instability in the PDE domain.

 Moreover, tip-leakage flow in the blade-tip region of high-pressure turbines not only increases heat transfer but also creates a significant pressure difference between the aerodynamic surfaces of the blade. This is primarily due to the small clearance between the blade tip and the stationary outer \cite{bib51}. Therefore, it is necessary to deal with the heat difference caused by thermal forces, which are related to blade vibrations \cite{bib1}.
	The temperature difference along the axis $x$ is described by the following heat PDE
	\begin{align}
		&\Delta T_t(x,t)=\acute{\kappa}\Delta T_{xx}(x,t), \label{ht} \\
		& \Delta T(0,t)=qd(t)-\frac{S_0}{\beta^{\ast}}\varpi(0,t), \quad \Delta T(1,t)=0 \label{htbd}
	\end{align}    
	where $\acute{\kappa}$ is the thermal diffusivity, $qd(t)$ is the heat flux on the actuated side of the blade, $S_0$ is the average heat flux level and $\beta^{\ast}$ is a parameter determined from a heat flux versus distance test \cite{bib1}.
Additionally, the blade tip experiences a force from the working fluid (air in gas turbines) that is represented as $-\mathcal{Q}_1\varpi(0,t)-\mathcal{Q}_2\varpi_t(0,t)$, i.e., the
anti-damping and anti-stiffness effects at the uncontrolled boundary. Thus, the boundary conditions are 
\begin{align}
	\varpi_t(1,t)=&R\Delta\omega(t), \\
	\varpi_x(0,t)=&\Phi(0,t)-\mathcal{Q}_1\varpi(0,t)-\mathcal{Q}_2\varpi_t(0,t), \\
	\Phi_x(1,t)=&u_1(t), \label{eq:Phi1}\\
	\Phi_x(0,t)=&I_0 \Delta T(0,t) \label{eq:Phi0}
\end{align}
where $\mathcal{Q}_1$, $\mathcal{Q}_2$ $\in$ $\mathbb{R}$ and $\mathcal{Q}_2$ satisfies the Assumption 1 in \cite{bib23}, i,e., $\mathcal{Q}_2\neq \sqrt{\epsilon}$. 

Due to the turbine blade's thinness, we consider it a slender Timoshenko Beam, which implies that the parameter \(\mu\) is very small. Following the approximation approach in \cite{bib18}, we set \(\mu = 0\) in \eqref{st0}, \eqref{st1}. Consequently, we have 
\begin{align}
	\epsilon \varpi_{tt}-\varpi_{xx}+\Phi_x=&k_1\varpi_t(x,t)+k_2\varpi_x(x,t), \label{armw} \\
	\frac{a} {\epsilon}(\Phi-\varpi_x)-\Phi_{xx}=&I_0 \Delta T_x. \label{ODE}
\end{align}
Using the Laplace transformation for \eqref{ODE}, the solution  $\Phi(x,t)$ can be obtained as 
\begin{align}
	\Phi(x,t)=&\cosh(bx)\Phi(0,t)+\frac{1}{b}\sinh(bx)\Phi_x(0,t) \notag \\
	&-b\int_{0}^{x}\sinh(b(x-y))\varpi_y(y,t)dy \notag \\
	&-\frac{1}{b}\int_{0}^{x}\sinh(b(x-y))I_0 \Delta T_y(y,t)dy
\end{align}
where $b=\sqrt{\frac{a}{\epsilon}}$. The constant $\Phi(0,t)$ can be expressed in terms of $\Phi_x(1,t)$ in the following way
\begin{align}
		\Phi(0,t)=&\frac{1}{b\sinh(b)} \bigg[ \Phi_x(1,t)-\cosh(b)\Phi_x(0,t)\notag \\
		&+b^2\int_{0}^{1}\cosh(b(1-y))\varpi_y(y,t)dy \notag \\
		&+\int_{0}^{1}\cosh(b(1-y))I_0 \Delta T_y(y,t)dy \bigg]. \label{ps0}
\end{align}
Inserting \eqref{eq:Phi1} into \eqref{ps0}, choosing the control input $u_1(t)$ as
\begin{align}
		u_1(t)=&\cosh(b)\Phi_x(0,t)-b^2\int_{0}^{1}\cosh(b(1-y))\varpi_y(y,t)dy \notag \\
		&-\int_{0}^{1}\cosh(b(1-y))I_0 \Delta T_y(y,t)dy
\end{align}which leads to $\Phi(0)=0$. With the boundary condition \eqref{eq:Phi0}, we can get
\begin{align}
		\Phi(x,t)=&\frac{I_0}{b}\sinh(bx)\Delta T(0,t)-b\int_{0}^{x}\sinh(b(x-y))\varpi_y(y)dy\notag \\
		&-\frac{1}{b}\int_{0}^{x}\sinh(b(x-y))I_0 \Delta T_y(y,t)dy. \label{eq:phix}
\end{align}
Differentiating $\Phi(x)$ with respect $x$ and substituting the result into \eqref{armw}, we get
\begin{align}
		\epsilon\varpi_{tt}(x,t)=&\varpi_{xx}(x,t)-2I_0\cosh(bx)(qd(t)-\frac{S_0}{\beta^{\ast}}\varpi(0,t)) \notag \\
		&+k_1\varpi_t(x,t)+k_2\varpi_x(x,t)+I_0 \Delta T(x,t) \notag \\
		&+bI_0 \int_{0}^{x}\sinh(b(x-y))\Delta T(y,t)dy \notag \\
		&+b^2\int_{0}^{x}\cosh(b(x-y))\varpi_y(y,t)dy \label{shearb}
\end{align}
with boundary conditions to be satisfied
\begin{align}
	\varpi_t(1,t)=&R\Delta\omega(t), \\
	\varpi_x(0,t)=&-\mathcal{Q}_1\varpi(0,t)-\mathcal{Q}_2\varpi_t(0,t). \label{shearbon}
\end{align}
To facilitate the backstepping control design, we use the following Riemann transformation
\begin{align}
	\xi(x,t)&=\frac{\sqrt{\epsilon}}{\varphi_2(x)}\varpi_t(x,t)+\frac{1}{\varphi_2(x)}\varpi_x(x,t), \label{trm0} \\
	\eta(x,t)&=\frac{\sqrt{\epsilon}}{\varphi_1(x)}\varpi_t(x,t)-\frac{1}{\varphi_1(x)}\varpi_x(x,t), \label{trm1} \\
	X(t)&=\varpi(0,t) \label{trm3} 
\end{align}
where $\varphi_1(x)=e^{\int_{0}^{x}\frac{k_1-\sqrt{\epsilon}k_2}{2\sqrt{\epsilon}}ds}$, $\varphi_2(x)=e^{\int_{0}^{x}-\frac{k_1+\sqrt{\epsilon}k_2}{2\sqrt{\epsilon}}ds}$ \cite{bib48},  to convert the system   \eqref{shearb}--\eqref{shearbon} into a $2\times2$  coupled hyperbolic PDE-ODE system, 
\begin{align}
	\sqrt{\epsilon}\xi_t(x,t)=&\xi_x(x,t)+\frac{(k_1-\sqrt{\epsilon}k_2)\varphi_1(x)}{2\sqrt{\epsilon}\varphi_2(x)}\eta(x,t) \notag \\
	&+\frac{b^2}{2}\int_{0}^{x}\frac{\cosh(b(x-y))}{\varphi_2(x)}(\varphi_2(y)\xi(y,t) \notag \\ 
	&-\varphi_1(y)\eta(y,t))dy+\frac{I_0}{\varphi_2(x)} \Delta T(x,t) \notag \\
	&+bI_0 \int_{0}^{x}\frac{\sinh(b(x-y))}{\varphi_2(x)} \Delta T(y,t)dy\notag \\
	&-\frac{2I_0\cosh(bx)}{\varphi_2(x)}(qd(t)-\frac{S_0}{\beta^{\ast}}X(t)), \label{xi} \\
	\sqrt{\epsilon}\eta_t(x,t)=&-\eta_x(x,t)+\frac{(k_1+\sqrt{\epsilon}k_2)\varphi_2(x)}{2\sqrt{\epsilon}\varphi_1(x)}\xi(x,t) \notag \\
	&+\frac{b^2}{2}\int_{0}^{x}\frac{\cosh(b(x-y))}{\varphi_1(x)}(\varphi_2(y)\xi(y,t)\notag \\
	&-\varphi_1(y)\eta(y,t))dy+\frac{I_0}{\varphi_1(x)} \Delta T(x,t) \notag \\
	&+bI_0 \int_{0}^{x}\frac{\sinh(b(x-y))}{\varphi_1(x)} \Delta T(y,t)dy\notag \\
	&-\frac{2I_0\cosh(bx)}{\varphi_2(x)}(qd(t)-\frac{S_0}{\beta^{\ast}}X(t)),  \label{eta} \\
	\xi(1,t)=&-\frac{\varphi_1(1)}{\varphi_2(1)}\eta(1,t)+\frac{2\sqrt{\varepsilon}R}{\varphi_2(1)}\Delta\omega(t),\label{xib} \\
	\eta(0,t)=&\frac{\sqrt{\epsilon}+\mathcal{Q}_2}{\sqrt{\epsilon}-\mathcal{Q}_2}\xi(0,t)+\frac{2\sqrt{\epsilon}\mathcal{Q}_1}{\sqrt{\epsilon}-\mathcal{Q}_2}X(t), \label{etab} \\
	\dot{X}(t)=&\frac{\mathcal{Q}_1}{\sqrt{\epsilon}-\mathcal{Q}_2}X(t)+\frac{1}{\sqrt{\epsilon}-\mathcal{Q}_2}\xi(0,t). \label{Xt} 
\end{align}
Now, we obtain the model \eqref{ht}, \eqref{htbd}, \eqref{st3}, and \eqref{xi}--\eqref{Xt} that is ready for control design. The thermal influence on the blade dynamics, i.e., the above hyperbolic PDE-ODE \eqref{xi}--\eqref{Xt}, comes from the temperature gradient $\Delta T$ that is described by a heat PDE \eqref{ht}, \eqref{htbd}, which is driven by the output of the above hyperbolic PDE-ODE. The flow-induced pressure modeled by \eqref{eq:pressure}  represents the in-domain instabilities, i.e., coupling terms in the $2\times2$  coupled hyperbolic PDEs \eqref{xi}, \eqref{eta}. The model \eqref{ht}, \eqref{htbd}, \eqref{st3},  \eqref{xi}--\eqref{Xt} is indeed an ODE-PDE-ODE system, where the PDE in the middle is a coupled system of hyperbolic and parabolic PDEs, which is different from \cite{bib36,bib37,bib39,bib41,bib42}.
\subsection{Generalization}
To generalize the control design, we present the obtained model \eqref{ht}, \eqref{htbd}, \eqref{st3}, and \eqref{xi}--\eqref{Xt} in the general form:
\begin{align}
	\dot{z}(t)=&c_0z(t)+c_1\xi(1,t)+U(t), \label{gsz} \\
	\varepsilon_1\eta_t(x,t)=&-\eta_x(x,t)+c_1(x)\xi(x,t)+D_1(x)X(t) \notag \\
	&+\int_{0}^{x}f_{12}(x,y)\xi(y,t)dy+g_1(x)\xi(0,t) \notag \\
	&+\int_{0}^{x}f_{11}(x,y)\eta(y,t)dy+\mu_1(x)u(x,t) \notag \\
	&+\int_{0}^{x}f_{13}(x,y)u(y,t)dy+G_1(x)d(t),	\label{gseta} \\
	\varepsilon_2\xi_t(x,t)=&\xi_x(x,t)+c_2(x)\eta(x,t)+D_2(x)X(t)  \notag \\
	&+\int_{0}^{x}f_{22}(x,y)\xi(y,t)dy+g_2(x)\xi(0,t) \notag \\
	&+\int_{0}^{x}f_{21}(x,y)\eta(y,t)dy+\mu_2(x)u(x,t) \notag \\
	&+\int_{0}^{x}f_{23}(x,y)u(y,t)dy+G_2(x)d(t),  \label{gsxi} \\ 
	\xi(1,t)=&-q_1\eta(1,t)+q_0z(t), \label{gsxib} \\
	\eta(0,t)=&\xi(0,t)+CX(t),  \label{gsetab} \\
	\dot{X}(t)=&AX(t)+B\xi(0,t), \label{gsx} \\
	u_t(x,t)=&\kappa_0u_{xx}(x,t), \label{gxu} \\
	u(0,t)=&qd(t)+\acute{p}_2X(t), \quad u(1,t)=0, \label{gxub} \\
	\dot{d}(t)=&A_dd(t). \label{gxd}
\end{align}
$\forall(x,y)\in[0,1]\times[0,+\infty)$, where $X(t)$ $\in$ $\mathbb{R}^{n}$, $z(t)$ $\in$ $\mathbb{R}$ are ODE states, $\eta(x,t)$ $\in$ $\mathbb{R}$, $\xi(x,t)$ $\in$ $\mathbb{R}$ are hyperbolic PDE states, and $u(x,t)\in\mathbb{R}$ is the state of the heat PDE. The function $U(t)$ is the control input to be designed. The
plant initial conditions are taken as 
\begin{align}
	&	(z(0),\xi(x,0),\eta(x,0), u(x,0), X(0), d(0)) \notag \\
	&	\in \mathcal H:=\mathbb R\times H^{1}(0,1)^3  \times \mathbb R^n \times \mathbb R^m. \label{space:H}
\end{align}
The general function $c_i(x)$, $g_i(x)$, $\mu_i(x)\in C^1(0,1)$, $D_i(x)\in C^1(0,1)\to \mathbb R^{1\times n}$, $G_i(x)\in C^1(0,1)\to \mathbb R^{1\times m} $ $f_{i,j}(x,y)$ $C^1([0,1] \times [0,1])$, $i=1,2$, $j=1,2,3$. The vector  $\acute{p}_2$ $\in$ $\mathbb{R}^{1\times n}$ satisfies that the $\acute{p_2}$ is arbitrary. The parameters $q_1\neq0$, $c_0$, $c_1$, and without loss of generality, $\varepsilon_1$, $\varepsilon_2$, $\kappa_0$ are positive constants. 
The vector $d(t)\in\mathbb R^{m\times 1}$ satisfies \eqref{gxd} where
$A_d\in\mathbb R^{m\times m}$ satisfy the following assumption:
\begin{asm}\label{d(t)}
	The diagonalizable matrix $A_d\in\mathbb{R}^{n\times n}$ only contains eigenvalues on the imaginary axis.
\end{asm}

Assumption \ref{d(t)} implies that $d(t)$, which is produced by a neutrally stable system matrix $A_d$, is bounded. Denote the upper bound as the constant $\bar{D}_d$, which is arbitrary and unknown.
This assumption is practically reasonable. In practice, disturbances act only within a finite working horizon. By periodically extending such finite-horizon signals to 
	$[0,\infty)$, they can be regarded as periodic and thus represented by Fourier series as a combination of sinusoidal components, which can be generated by the system as Assumption \ref{d(t)}. Therefore, Assumption \ref{d(t)} covers a broad class of practically relevant disturbance signals.

 The matrix $A$ $\in$ $\mathbb{R}^{n\times n}$, $B$ $\in$ $\mathbb{R}^{n\times1}$, and $q$, $C$ $\in$ $\mathbb{R}^{1\times n}$ satisfy the following assumption:
	\begin{asm}\label{matrix}
		The pair $(A,B)$ is controllable, and the pairs $(A,C)$ as well as $(A_d,q)$ are observable.
\end{asm}
We give the solution notion below.
	\begin{Def}\label{def1}
		For a solution of \eqref{gsz}--\eqref{gxub}, we mean a tuple \begin{align}
			&(z(t),\xi(x,t),\eta(x,t), u(x,t), X(t))\in  L^{\infty}([0,\infty);\mathbb R)\notag\\&\times L^{\infty}([0,\infty);H^1(0,1))^3\times H^1([0,\infty);\mathbb R^n)\label{eq:space}
		\end{align}
		where the equations are satisfied in the weak sense.
\end{Def}
\section{State-Feedback Control}\label{staefd}
\subsection{State-Feedback Control Design}
We propose a backstepping transformation to map the original system \eqref{gsz}-\eqref{gxd} into a target system with the desired performance. Inspired by \cite{bib28}, we build the transformation as
\begin{align}
	&\beta(x,t)=\xi(x,t)+\gamma(x)X(t)-\int_{0}^{x}k(x,y)\xi(y,t)dy\notag \\
	&-\int_{0}^{x}l(x,y)\eta(y,t)dy-\int_{0}^{1}p(x,y)u(y,t)dy+\Upsilon(x)d(t) \label{cxi2bt}
\end{align}
whose inverse is
\begin{align}
	&\xi(x,t)=\beta(x,t)-\lambda(x)X(t)+\int_{0}^{x}\rho(x,y)\beta(y,t)dy\notag \\	&+\int_{0}^{x}\sigma(x,y)\eta(y,t)dy+\int_{0}^{1}\varrho(x,y)u(y,t)dy-\vartheta(x)d(t). \label{cbt2xi}
\end{align}
The kernels $k(x,y)$, $l(x,y)$ evolving in the domain $\mathcal{D}=\lbrace (x,y)\in \mathbb{R}^2 \vert 0 \leq y \leq x \leq 1 \rbrace$ and the kernel $\gamma(x)$ defined in $\lbrace 0\leq x\leq 1 \rbrace$ satisfy
\begin{align}
	k_x(x,y)&=-k_y(x,y)-f_{22}(x,y)+\int_{y}^{x}f_{22}(z,y)k(x,z)dz\notag \\
	&+\frac{\varepsilon_2}{\varepsilon_1}\int_{y}^{x}f_{12}(z,y)l(x,z)dz+\frac{\varepsilon_2}{\varepsilon_1}l(x,y)c_1(y), \label{kerK} \\
	l_x(x,y)&=\frac{\varepsilon_2}{\varepsilon_1}l_y(x,y)-f_{21}(x,y)+\int_{y}^{x}f_{21}(z,y)k(x,z)dz\notag \\
	&+\frac{\varepsilon_2}{\varepsilon_1}\int_{y}^{x}f_{11}(z,y)l(x,z)dz+k(x,y)c_2(y), \label{kerL} \\
	\gamma_x(x)&=\varepsilon_2\gamma(x)A-\frac{\varepsilon_2}{\varepsilon_1}\int_{0}^{x}l(x,y)D_1(y)dy-\frac{\varepsilon_2}{\varepsilon_1}l(x,0)C \notag \\
	&-\int_{0}^{x}k(x,y)D_2(y)dy+D_2(x)-\varepsilon_2\kappa_0p_y(x,0)\acute{p}_2, \notag  \\
	&-\frac{\varepsilon_2}{\varepsilon_1}l(x,0)C-\frac{\varepsilon_2}{\varepsilon_1}\int_{0}^{x}l(x,y)D_1(y)dy+D_2(x), \label{kerGma} \\
	k(x,0)&=\frac{\varepsilon_2}{\varepsilon_1}l(x,0)-\varepsilon_2\gamma(x)B+\int_{0}^{x}k(x,y)g_2(y)dy\notag \\
	&+\frac{\varepsilon_2}{\varepsilon_1}\int_{0}^{x}l(x,y)g_1(y)dy-g_2(x), \label{kerKb} \\
	l(x,x)&=-\frac{\varepsilon_1}{\varepsilon_1+\varepsilon_2}c_2(x), \quad \gamma(0)=-K \label{kerLGmab}
\end{align}
where $K\in\mathbb{R}^{1\times n}$ are control gains to be designed. The kernel $p(x,y)$ on $\lbrace 0\leq x,y\leq 1 \rbrace$ and $\Upsilon(x)$ on $\{ 0\leq x\leq 1\}$ here goes
\begin{align}
	p_x(x,y)=&\varepsilon_2 \kappa_0 p_{yy}(x,y)+h(x-y)H(x,y)\notag \\
	&-\delta(y-x)\mu_2(y), \label{kerP} \\
	\Upsilon_x(x)=&\varepsilon_2 \Upsilon(x)A_d-\varepsilon_2 \kappa_0 p_y(x,0)q+G_2(x), \label{kerUp} \\ 
	p(x,1)=&0,\quad p(x,0)=0,\quad p(0,y)=0, \label{kerPb} \\
	\Upsilon(0)=&0 \label{kerPUpb}
\end{align}
where
$H(x,y)=\frac{\varepsilon_2}{\varepsilon_1}l(x,y)\mu_1(y)+k(x,y)\mu_2(y)-f_{23}(x,y)+\int_{y}^{x}f_{23}(z,y)k(x,z)dz+\frac{\varepsilon_2}{\varepsilon_1}\int_{y}^{x}f_{13}(z,y)l(x,z)dz
,$
$h(x)$ is the step function satisfying $h(x)=1$, $x>0$ and $h(x)=0$, $x\leq0$, $\delta(x)$ is Dirac's function. Notice, the expressions of inverse kernels $\lambda(x)$, $\rho(x,y)$, $\sigma(x,y)$, $\varrho(x,y)$, $\vartheta(x)$ are similar to those of $\gamma(x)$, $k(x,y)$, $l(x,y)$, $p(x,y)$, $\Upsilon(x)$ and the well-posedness of them are equal to kernels in \eqref{cxi2bt}. Thus, both of them are omitted.

The well-posedness of \eqref{kerK}--\eqref{kerPUpb} is shown in the following lemma.
\begin{Lem}\label{kernel klp}
	The equation set \eqref{kerK}--\eqref{kerPUpb} has a weak solution $k, l\in H^1(\mathcal{D})$, $\gamma\in H^1(0,1)$, $p\in L^2(0,1;H^1(0,1))$ and $\Upsilon \in H^1(0,1)$ such that $\|k(x,\cdot)\|^2_{H^1(0,x)}+\|l(x,\cdot)\|^2_{H^1(0,x)}+\|p(x,\cdot)\|^2_{H^1(0,1)}+\|\gamma(\cdot)\|^2_{H^1(0,1)}+\|\Upsilon(\cdot)\|^2_{H^1(0,1)}\leq\bar{C}_2$ for some positive $\bar{C}_2$.
\end{Lem}
\begin{proof}
	The proof of Lemma \ref{kernel klp} is shown in Appendix\ref{pfk1}.
\end{proof}
Using the transformations \eqref{cxi2bt}, \eqref{cbt2xi}, the original  \eqref{gsxi}--\eqref{gsx} is converted to the following equations
\begin{align}
	\varepsilon_2\beta_t(x,t)=&\beta_x(x,t), \label{gtsbt} \\
	\varepsilon_1\eta_t(x,t)=&-\eta_x(x,t)+c_1(x)\beta(x,t)+g_1(x)\beta(0,t)  \notag \\
	&+\int_{0}^{x}\mathcal{F}_{11}(x,y)\eta(y,t)dy+\mu_1(x)u(x,t)  \notag \\
	&+\int_{0}^{x}\mathcal{F}_{12}(x,y)\beta(y,t)dy+\mathcal{D}_1(x)X(t) \notag \\
	&+\int_{0}^{1}\mathcal{F}_{13}(x,y)u(y,t)dy+\mathcal{D}_2(x)d(t),  \label{gtset} \\
	\eta(0,t)=&\beta(0,t)+(C+K)X(t), \label{gtsalb} \\
	\dot{X}(t)=&(A+BK)X(t)+B\beta(0,t) \label{gstx} 
\end{align}
where the design parameters $K\in\mathbb{R}^{1\times n}$ are chosen such that the system $A+BK$ is Hurwitz, and where $\mathcal{F}_{11}(x,y)$, $\mathcal{F}_{12}(x,y)$, $\mathcal{F}_{13}(x,y)$, $\mathcal{D}_1(x)$, $\mathcal{D}_2(x)$ are shown in Appendix\ref{F}.
According to \cite{bib38}, recalling the equations \eqref{gsz}, \eqref{gsxib} and the transformations \eqref{cxi2bt}, \eqref{cbt2xi}, through a lengthy calculation involving a change in the order of integration within a double integral, the dynamic right boundary, originating from the ODE in the input channel, of \eqref{gtsbt}--\eqref{gstx} is obtained as 
\begin{align}
	\beta_t(1,t)&=-q_1\eta_t(1,t)+h_1\beta(1,t)+h_2\beta(0,t)+h_3\eta(1,t) \notag \\
	&+h_4\eta(0,t)+h_5X(t)+h_6d(t)+q_0U(t) \notag \\
	&+\int_{0}^{1}H_7(y)\beta(y,t)dy+\int_{0}^{1}H_8(y)\eta(y,t)dy \notag \\
	&+\int_{0}^{1}H_9(y)u(y,t)dy+\int_{0}^{1}H_{10}(y)u_y(y,t)dy \label{bet}
\end{align}
where the gains $h_1$, $h_2$, $h_3$, $h_4$, $h_5$, $h_6$, $H_7$, $H_8$, $H_9$, $H_{10}$ are shown in Appendix\ref{h}. 
Choosing
\begin{align}
	U(t)&=\frac{1}{q_0}\bigg[ -(\acute{c}_1+h_1)\beta(1,t)+q_1\eta_t(1,t)-h_2\beta(0,t) \notag \\
	&-h_3\eta(1,t)-h_4\eta(0,t)-h_5X(t)-h_6d(t) \notag \\
	&-\int_{0}^{1}H_7(y)\beta(y,t)dy-\int_{0}^{1}H_8(y)\eta(y,t)dy \notag \\
	&-\int_{0}^{1}H_9(y)u(y,t)dy-\int_{0}^{1}H_{10}(y)u_y(y,t)dy\bigg], \label{Utag}
\end{align}
recalling \eqref{bet}, we have the right (dynamic) boundary of the target system is 
\begin{align}
	\beta_t(1,t)=-\acute{c}_1\beta(1,t)\label{eq:beta1}
\end{align}
where $\acute{c}_1$ is the positive design parameter whose condition is given by \eqref{eq:conditionbarc1}. Now we arrive at the target system \eqref{gtsbt}--\eqref{gstx}, \eqref{eq:beta1} via the proposed transformations and by the choice of the control input \eqref{Utag}. 

\begin{Rem}
	Like \cite{bib38}, the right boundary of the hyperbolic PDEs and the input $z$-ODE are captured as a dynamic right boundary of the hyperbolic PDEs in the control design, and thus the target system \eqref{gtsbt}--\eqref{gstx} is achieved with a dynamic boundary \eqref{eq:beta1}, which encapsulates the dynamics of the input ODE. 
\end{Rem}

Substituting \eqref{cxi2bt}, \eqref{gseta} and \eqref{gsxi} into \eqref{Utag}, we get the controller expressed by the original states
\begin{align}
	U(t)=&n_1\xi(1,t)+n_2\eta(1,t)+n_3X(t)+n_4d(t) \notag \\	&+n_5\xi(0,t)+n_6\eta_x(1,t)+\int_{0}^{1}N_7(y)\xi(y,t)dy \notag \\
	&+\int_{0}^{1}N_8(y)\eta(y,t)dy+\int_{0}^{1}N_9(y)u(y,t)dy \notag \\
	&+\int_{0}^{1}N_{10}(y)u_y(y,t)dy \label{U(t)}
\end{align}
where $n_1$, $n_2$, $n_3$, $n_4$, $n_5$, $n_6$, $N_7$, $N_8$, $N_9$, $N_{10}$ are given in Appendix \ref{n}.
\subsection{Result of the state-feedback control}
The property of the state-feedback closed-loop system is presented as follows. 
\begin{Thm}\label{tosb}
	If  initial value $(z(0),\xi(x,0),\eta(x,0),u(x,0)$, $ X(0)) \in \mathcal H$, the closed-loop system consists of the plant \eqref{gsz}--\eqref{gxd} and the controller \eqref{U(t)} has the following properties:
	
	1) There exits a unique solution in the sense of in \eqref{eq:space} in Definition \ref{def1}.
	
	2) For the distal ODE, $|X(t)|$ is exponentially convergent to zero. The decay rate can be adjusted by the design parameters $K$ and $\acute{c}_1$.
	Besides, all states in the closed-loop system are uniformly ultimately bounded in the sense that there exist positive constants $\Upsilon_0,\acute{C}_0,\mathcal{C}_0$ given by \eqref{mathC0} such that
	\begin{align}
		\Omega_0(t)< \Upsilon_0\Omega_0(0)e^{-\acute{C}_0 t}+\mathcal{C}_0 \label{Omeg0}
	\end{align}
	where
	$$\Omega_0(t)= | z(t)|^2+\| \xi(\cdot,t)\|^2_{H^1}+\| \eta(\cdot,t)\|^2_{H^1}+| X(t)|^2+\|u(\cdot,t)\|^2_{H^1},$$ and where $\mathcal{C}_0$ is bounded by $\acute D_d$. If $\acute D_d=0$ (i.e., no external disturbance), the exponential stability is obtained in the sense of 
		\begin{align}
			\Omega_0(t)< \Upsilon_0\Omega_0(0)e^{-\acute{C}_0 t}.
	\end{align}
	3) The control input $U(t)$ \eqref{U(t)} is well-defined and  bounded.\end{Thm}

\begin{proof}
	The proof of Theorem $\ref{tosb}$ is shown in Appendix\ref{The of controller}.
\end{proof}

\section{State Observer}\label{observerds}
Since the distributed states required by controller \eqref{U(t)} are often inaccessible in practice, in this section, we design an observer for the unmeasured states only using the available measurements at PDE boundaries. 
\subsection{Observer design}
Relying on the measurements $\xi(0,t)$, $u(0,t)$ and $CX(t)$, like \cite{bib40} the observer is built as a copy of the plant \eqref{gsz}--\eqref{gxd} with output error injections:
\begin{align}
	\dot{\hat{z}}(t)=&c_0\hat{z}(t)+c_1\hat{\xi}(1,t)-\Gamma^z(\hat{\xi}(0,t)-\xi(0,t))\notag \\
	&+U(t), \label{obz} \\
	\varepsilon_1\hat{\eta}_t(x,t)=&-\hat{\eta}_x(x,t)+c_1(x)\hat{\xi}(x,t)+D_1(x)\hat{X}(t) \notag \\
	&+\int_{0}^{x}f_{12}(x,y)\hat{\xi}(y,t)dy+g_1(x)\hat{\xi}(0,t) \notag \\
	&+\int_{0}^{x}f_{11}(x,y)\hat{\eta}(y,t)dy+\mu_1(x)\hat{u}(x,t) \notag \\
	&+\int_{0}^{x}f_{13}(x,y)\hat{u}(y,t)dy+G_1(x)\hat{d}(t)  \notag \\
	&-\Gamma^\eta(x)(\hat{\xi}(0,t)-\xi(0,t)),	\label{obeta}  \\
	\varepsilon_2\hat{\xi}_t(x,t)=&\hat{\xi}_x(x,t)+c_2(x)\hat{\eta}(x,t)+D_2(x)\hat{X}(t) \notag \\
	&+\int_{0}^{x}f_{22}(x,y)\hat{\xi}(y,t)dy+g_2(x)\hat{\xi}(0,t) \notag \\
	&+\int_{0}^{x}f_{21}(x,y)\hat{\eta}(y,t)dy+\mu_2(x)\hat{u}(x,t) \notag \\
	&+\int_{0}^{x}f_{23}(x,y)\hat{u}(y,t)dy+G_2(x)\hat{d}(t) \notag \\
	&-\Gamma^\xi(x)(\hat{\xi}(0,t)-\xi(0,t)),  \label{obxi} \\
	\hat{\xi}(1,t)=&-q_1\hat\eta(1,t)+q_0\hat{z}(t),  \label{obxib}  \\
	\hat{\eta}(0,t)=&\xi(0,t)+CX(t),  \label{obetab} \\
	\dot{\hat{X}}(t)=&A\hat{X}(t)+B\xi(0,t)-L_xC(\hat{X}(t)-X(t)), \label{obx} \\
	\hat{u}_t(x,t)=&\kappa_0\hat{u}_{xx}(x,t), \\
	\hat{u}(0,t)=&q\hat{d}(t)+\acute{p}_2\hat{X}(t), \quad \hat{u}(1,t)=0, \label{obu} \\
	\dot{\hat{d}}(t)=&A_d\hat{d}(t)-L_d(\hat{u}(0,t)-u(0,t)) \label{obd}
\end{align}
where the observer gains \begin{align}
	L_z>c_0 \label{eq:Lz}
\end{align}
and $L_x$
is chosen such that $A-L_xC$, $A_d-L_dq$ are Hurwitz,
and where observer gains $\Gamma^z$, $\Gamma^\eta(x)$, $\Gamma^\xi(x)$ are  to be determined later.
\begin{Rem}
	In the practical application of the turbine blade, all the measurements used in the observer are available. The signal $\xi(0,t)$ can be obtained via \eqref{up0om}--\eqref{up2om}, \eqref{trm0}--\eqref{trm3}, where $\upsilon_t(0,t)$, $\upsilon_x(0,t)$, and $\theta(t)$, can be measured by an accelerometer, a strain gauge offset potentiometer, which are placed at the distal end, and an encoder associated with the motor, respectively. Also, $CX(t)$ can be derived via \eqref{up0om} and \eqref{trm3}, where $\upsilon(0,t)$ is obtained by integrating $\upsilon_t(0,t)$ obtained above.
	Temperature difference $u(1,t)$ is obtained by a temperature transducer. Using the observer built in this section, the distributed temperature gradients and structural vibrations in the Timoshenko beam can be estimated.
\end{Rem}

Define the observer error as
\begin{align}
		&[\widetilde{X}(t),\widetilde{d}(t),\widetilde{\xi}(x,t),\widetilde{\eta}(x,t),
		\widetilde{u}(x,t)]=[\hat{X}(t),\hat{d}(t),\hat{\xi}(x,t), \notag \\
		&\hat{\eta}(x,t),\hat{u}(x,t)]-[X(t),d(t),\xi(x,t),\eta(x,t),u(x,t)]. \label{widSt}
\end{align}
According to \eqref{obz}--\eqref{obd}, the observer error dynamics can be obtained as
\begin{align}
	\dot{\widetilde{z}}(t)=&c_0\widetilde{z}(t)+c_1\widetilde{\xi}(1,t)-\Gamma^z\widetilde{\xi}(0,t), \label{erz} \\
	\varepsilon_1\widetilde{\eta}_t(x,t)=&-\widetilde{\eta}_x(x,t)+c_1(x)\widetilde{\xi}(x,t)+D_1(x)\widetilde{X}(t)  \notag \\
	&+\int_{0}^{x}f_{12}(x,y)\widetilde{\xi}(y,t)dy+g_1(x)\widetilde{\xi}(0,t) \notag \\
	&+\int_{0}^{x}f_{11}(x,y)\widetilde{\eta}(y,t)dy+\mu_1(x)\widetilde{u}(x,t) \notag \\
	&+\int_{0}^{x}f_{13}(x,y)\widetilde{u}(y,t)dy+G_1(x)\widetilde{d}(t) \notag \\
	&-\Gamma^\eta(x)\widetilde{\xi}(0,t),	\label{ereta} \\ 	\varepsilon_2\widetilde{\xi}_t(x,t)=&\widetilde{\xi}_x(x,t)+c_2(x)\widetilde{\eta}(x,t)+D_2(x)\widetilde{X}(t)  \notag \\
	&+\int_{0}^{x}f_{22}(x,y)\widetilde{\xi}(y,t)dy+g_2(x)\widetilde{\xi}(0,t) \notag \\
	&+\int_{0}^{x}f_{21}(x,y)\widetilde{\eta}(y,t)dy+\mu_2(x)\widetilde{u}(x,t) \notag \\
	&+\int_{0}^{x}f_{23}(x,y)\widetilde{u}(y,t)dy+G_2(x)\widetilde{d}(t)\notag \\
	&-\Gamma^\xi(x)\widetilde{\xi}(0,t),  \label{erxi}  \\
	\widetilde{\xi}(1,t)=&-q_1\widetilde{\eta}(1,t)+q_0\widetilde{z}(t),\quad
	\widetilde{\eta}(0,t)=0, \label{eretab} \\
	\dot{\widetilde{X}}(t)=&(A-L_xC)\widetilde{X}(t), \label{erx} \\
	\widetilde{u}_t(x,t)=&\kappa_0\widetilde{u}_{xx}(x,t), \\ 
	\widetilde{u}(0,t)=&q\widetilde{d}(t)+\acute{p}_2\widetilde{X}(t), \quad \widetilde{u}(1,t)=0, \label{eru}  \\
	\dot{\widetilde{d}}(t)=&(A_d-L_dq)\widetilde{d}(t)-L_d\acute{p}_2\widetilde{X}(t). \label{erd}
\end{align}
To decouple the PDE $\widetilde{\xi}(x,t)$ \eqref{erxi} with the ODE state $\widetilde{z}(t)$ \eqref{erz} and eliminate the in-domain coupling between PDEs in \eqref{ereta}, \eqref{erxi}, we introduce the invertible backsteppting transformations:
\begin{align}
	\widetilde{z}(t)=&\widetilde{Y}(t)+\int_{0}^{1}M(x)\widetilde{\beta}(x,t)dx, \label{z2Y} \\
	\widetilde{\xi}(x,t)=&\widetilde{\beta}(x,t)+\int_{0}^{x}\psi(x,y)\widetilde{\beta}(y,t)dy, \label{xitob} \\
	\widetilde{\eta}(x,t)=&\widetilde{\alpha}(x,t)+\int_{0}^{x}\phi(x,y)\widetilde{\beta}(y,t)dy \label{ettoa}
\end{align}
where the kernel functions $\psi(x,y),\phi(x,y)$ on $\acute{\mathcal{D}}=\{0 \leq x \leq y \leq 1 \}$ satisfy
\begin{align}
	&-\phi_x(x,y)+\frac{\varepsilon_1}{\varepsilon_2}\phi_y(x,y)+f_{12}(x,y)-\int_{x}^{y}\phi(z,y)f_{11}(x,z)dz\notag \\
	&-\int_{x}^{y}\psi(z,y)f_{12}(x,z)dz+c_1(x)\psi(x,y)=0, \label{phi} \\
	&\phi(x,x)=\frac{\varepsilon_1}{\varepsilon_1+\varepsilon_2}c_1(x), \label{phib} \\
	&\psi_x(x,y)+\psi_y(x,y)+f_{22}(x,y)-\int_{x}^{y}\phi(z,y)f_{21}(x,z)dz\notag \\
	&-\int_{x}^{y}\psi(z,y)f_{22}(x,z)dz+c_2(x)\phi(x,y)=0 . \label{psi}
\end{align}
with the boundary condition 
\begin{align}
	\psi(1,y)&=-q_1\phi(1,y)+q_0M(y) \label{psib}
\end{align}
and where the function $M(x)$ on $\{0\leq x\leq 1\}$ satisfies
\begin{align}
	M_x(x)&+\varepsilon_2 c_0M(x)+\varepsilon_2 c_1\psi(1,x)=0, \label{K0(x)} \\
	M(1)&=\varepsilon_2( \frac{L_z}{q_0}+c_1). \label{K0b}
\end{align}
The well-posedness of the above kernel conditions is given as follows.
\begin{Lem}\label{obkr}
	The equation set \eqref{phi}--\eqref{K0b} has a unique solution $\phi\in C^1(\acute{\mathcal{D}})$, $\psi\in C^1(\acute{\mathcal{D}})$, and $M\in C^1(0,1)$.
\end{Lem}
\begin{proof}
	By considering the alternate variables $\bar{M}(\chi,y)$ and $\bar{N}(\chi,y)$ introduced in \cite{bib44}, which transform the domain from $\acute{\mathcal{D}}$ to $\mathcal{D}$, the kernel conditions of  $\phi(x,y), \psi(x,y)$ have the simplified structure as the control kernel system \eqref{kerK}--\eqref{kerLGmab}. Moreover, the explicit solution of $M(x)$ can be easily obtained from the initial value problem \eqref{K0(x)}, \eqref{K0b}.
\end{proof}

Applying the transformations \eqref{z2Y}--\eqref{ettoa}, choosing the observer gain $\Gamma^z(t)$, $\Gamma^\eta(x)$ and $\Gamma^\xi(x)$  as 
\begin{align}
	\Gamma^z &= \frac{M(0)}{\varepsilon_2},  \label{Gamz}\\
	\Gamma^\eta(x)&=\frac{\varepsilon_1}{\varepsilon_2}\phi(x,0)+g_1(x),\label{Gamet} \\
	\Gamma^\xi(x)&=\psi(x,0)+g_2(x), \label{Gamxi}
\end{align} 
we convert \eqref{erz}--\eqref{eretab} to the following system:
\begin{align}
	\dot{\widetilde{Y}}(t)=&(c_0-L_z)\widetilde{Y}(t)-\int_{0}^{1}\frac{M(x)}{\varepsilon_2}N_4(x)dx\widetilde{d}(t) \notag \\
	&+\int_{0}^{1}\mathcal{G}_{3}(x)\widetilde{u}(x,t)dx+\int_{0}^{1}\mathcal{G}_{1}(x)\widetilde{\alpha}(x,t)dx \notag \\
	&-\int_{0}^{1}\frac{M(x)}{\varepsilon_2}N_2(x)dx\widetilde{X}(t)+\frac{L_z}{q_0}\widetilde{\alpha}(1,t), \label{tgY} \\
	\varepsilon_1\widetilde{\alpha}_t(x,t)=&-\widetilde{\alpha}_x(x,t)+\int_{0}^{x}S_{11}(x,y)\widetilde{\alpha}(y,t)dy \notag \\
	&+\int_{0}^{x}S_{13}(x,y)\widetilde{u}(y,t)dy+\mu_1(x)\widetilde{u}(x,t) \notag \\
	&+N_1(x)\widetilde{X}(t)+N_3(x)\widetilde{d}(t), \label{tgal}  \\
	\varepsilon_2\widetilde{\beta}_t(x,t)=&\widetilde{\beta}_x(x,t)+\int_{0}^{x}S_{21}(x,y)\widetilde{\alpha}(y,t)dy+c_2(x)\widetilde{\alpha}(x,t) \notag \\
	&+\int_{0}^{x}S_{23}(x,y)\widetilde{u}(y,t)dy+\mu_2(x)\widetilde{u}(x,t) \notag \\
	&+N_2(x)\widetilde{X}(t)+N_4(x)\widetilde{d}(t), \label{tgbt} \\
	\widetilde{\beta}(1,t)=&-q_1\widetilde{\alpha}(1,t)+q_0\widetilde{Y}(t), \quad \widetilde{\alpha}(0,t)=0 \label{tgalbtb} 
\end{align}
where $\mathcal{G}_{1}(x)$, $\mathcal{G}_{3}(x)$, $S_{11}(x,y)$, $S_{13}(x,y)$, $S_{21}(x,y)$, $S_{23}(x,y)$, $N_1(x)$, $N_2(x)$, $N_3(x)$, $N_4(x)$ are shown in Appendix\ref{S11}.

Finally, the target observer error system is \eqref{erx}--\eqref{eru}, \eqref{tgY}--\eqref{tgalbtb}.
\subsection{Stability of the Observer Error System}
\begin{Thm}\label{ertg}
	Considering the observer system \eqref{obz}--\eqref{obd} with observer gains \eqref{Gamz}--\eqref{Gamxi} and initial values $(\hat{z}(0),\hat{\xi}(x,0),\hat{\eta}(x,0)$, $\hat{u}(x,0), \hat{X}(0), \hat{d}(0)) \in \mathcal H$, the observer error system \eqref{erz}--\eqref{eru} has the following properties:
	
	1) There exits a unique solution in the sense of $(\widetilde{z},\widetilde{\xi},\widetilde{\eta}, \widetilde{u}, \widetilde{X}, \widetilde{d}) \in H^1([0,\infty);\mathbb R)\times H^{1}([0,1]\times [0,\infty))^3 \times H^1([0,\infty);\mathbb R^n)\times H^1([0,\infty);\mathbb R^m)$.
	
	2) The observer error system are exponentially stable in the sense that there exist positive constants $\varUpsilon_e,\acute{C}_e$ \eqref{Upe} such that
	\begin{align}
		\Omega_e(t)\leq \varUpsilon_e\Omega_e(0)e^{-\acute{C}_et} \label{Ome}
	\end{align}
	where $\Omega_e(t)=| \widetilde{z}(t)|^2+\| \widetilde{\eta}(\cdot,t)\|^2_{H^1}+\| \widetilde{\xi}(\cdot,t)\|^2_{H^1}+| \widetilde{X}(t)|^2$ $ +$$\| \widetilde{u}(\cdot,t)\|^2_{H^1}+|\widetilde{d}(t)|^2$ and $\acute{C}_e$ can be adjusted by the design parameters $L_z$, $L_d$ and $L_x$.
\end{Thm}
\begin{proof}
	The proof is shown in Appendix\ref{The 3}.
\end{proof}
\section{Output-Feedback Closed-Loop System}\label{outputfd}
\begin{figure}[!t]
	\centerline{\includegraphics[width=\columnwidth]{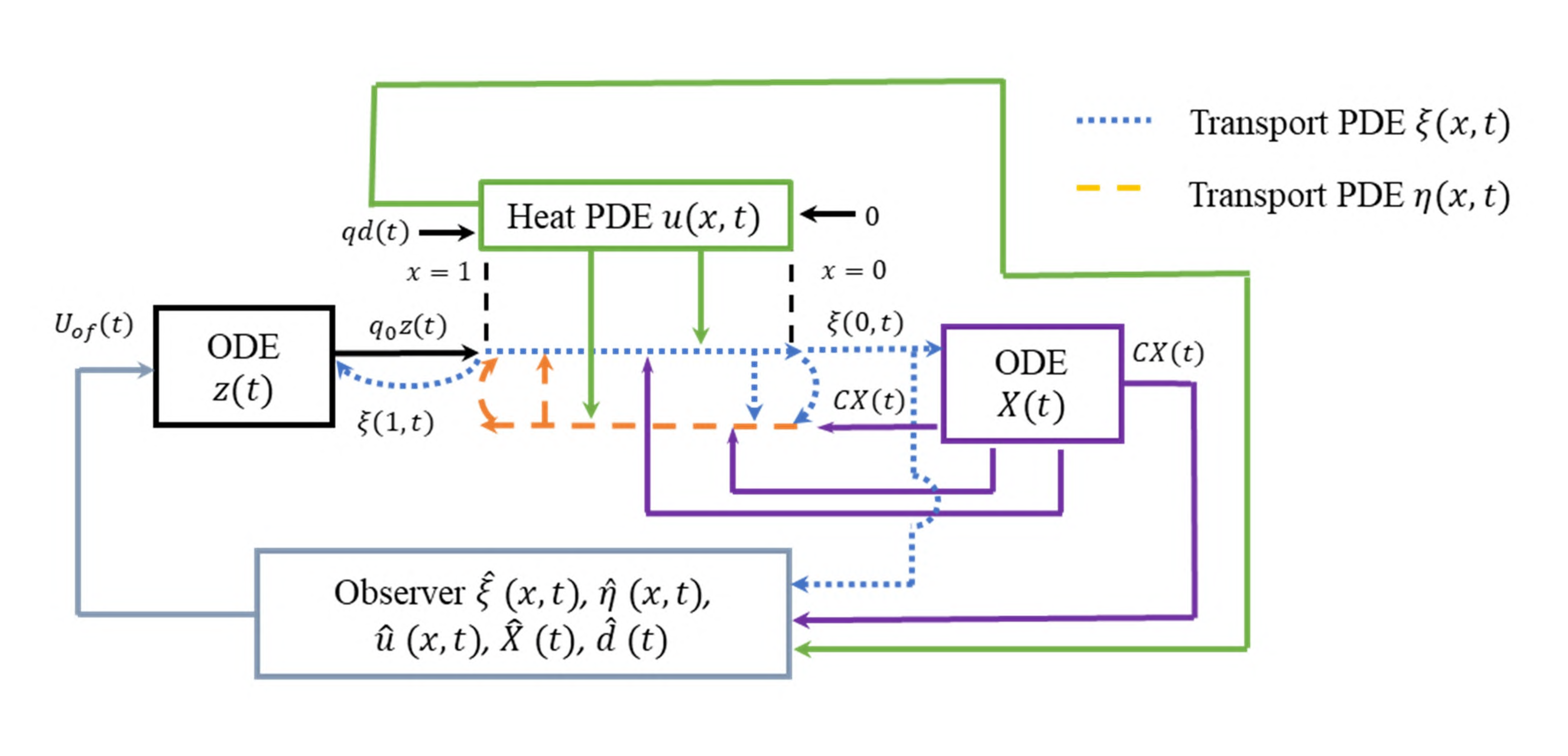}}
	\caption{The diagram of closed-loop system.}
	\label{Fg2}
\end{figure}
Substituting all the unmeasurable states in the state feedback control law \eqref{U(t)} with the observer states, the output-feedback controller can be obtained as
\begin{align}
	U_{of}(t)=&n_1\hat{\xi}(1,t)+n_2\hat{\eta}(1,t)+n_3\hat{X}(t)+n_4\hat{d}(t) \notag \\
	&+n_5\xi(0,t)+n_6\hat{\eta}_x(1,t)+\int_{0}^{1}N_7(y)\hat{\xi}(y,t)dy \notag \\
	&+\int_{0}^{1}N_8(y)\hat{\eta}(y,t)dy+\int_{0}^{1}N_9(y)\hat{u}(y,t)dy \notag \\
	&+\int_{0}^{1}N_{10}(y)\hat{u}_y(y,t)dy. \label{Uof(t)}
\end{align}
The diagram of the closed-loop system is shown in Fig. \ref{Fg2}, whose properties are given as follows.
\begin{Thm}\label{uof}
	For all initial data  $(z(0),\xi(x,0),\eta(x,0),$ $u(x,0), X(0)) \in \mathcal H$
	and $(\hat{z}(0),\hat{\xi}(x,0),\hat{\eta}(x,0),\hat{u}(x,0),$ $\hat{X}(0),\hat{d}(0)) \in \mathcal H$, the output-feedback closed-loop system consisting of the plant \eqref{gsz}--\eqref{gxd}, the observer \eqref{obz}--\eqref{obd}, and the controller \eqref{Uof(t)},  has the following properties:
	
	1)There exits a week unique solution in the sense of $(z,\hat{z},\xi,\hat{\xi},\eta,\hat{\eta},u,\hat{u}, X,\hat{X},\hat{d}) \in L^{\infty}([0,\infty);\mathbb R)^2\times L^{\infty}([0,\infty);H^1(0,1))^6 \times H^1([0,\infty);$ $\mathbb R^n)^2 \times L^{\infty}([0,\infty);$ $\mathbb R^m)$.
	
	2) For the distal ODE, $| X(t)|$ is exponentially convergent to zero, at a rate determined by the design parameters $\acute{c}_1$, $K$, $L_x$, $L_d$ and $L_z$.
	Moreover, there exist a positive constant $\Upsilon_{a}$ such that
	\begin{align}
		\Omega_{a}(t)\leq& \Upsilon_{a}{\Omega_{a}(0)}e^{-\acute{C}_{a}t}+\mathcal{C}_{a}
	\end{align}
	where $	\Omega_{a}(t)$ $=\| \eta(\cdot,t)\|^2_{H^1}+\| \hat{\eta}(\cdot,t)\|^2_{H^1}+\| \xi(\cdot,t)\|^2_{H^1}+\| \hat{\xi}(\cdot,t)\|^2_{H^1}+| z(t)|^2+| \hat{z}(t)|^2+\lvert X(t)\lvert^2+| \hat{X}(t)|^2+| \hat{d}(t)|^2$. The delay rate $\acute{C}_{a}=\min\{ \acute{C}_{0}, \acute{C}_{e} \}$ depends on the decay rate $\acute{C}_{0}$ of the state-feedback closed-loop system and that $\acute{C}_{e}$ of the observer error system. The positive constant $\mathcal{C}_{a}=2\mathcal{C}_0+{(1+\Upsilon_{a})\bar{D}_d^2}$ is bounded by $\bar{D}_d$, where $\mathcal C_0$ is given by \eqref{mathC0}.
		If $\bar{D}_d=0$ (i.e., no external disturbance), the exponential stability is obtained in the sense of
		\begin{align}
			\Omega_{a}(t)\leq& \Upsilon_{a}\Omega_{a}(0)e^{-\acute{C}_{a}t}.
	\end{align}
	
	3) The output-feedback control law \eqref{Uof(t)} is well-defined and bounded.\end{Thm}
\begin{proof}
	Inserting the output-feedback control law \eqref{Uof(t)} into \eqref{gsz}, we have
	\begin{align}
		\dot{z}(t)=c_0z(t)+c_1\xi(1,t)+U(t)+\delta(t)
	\end{align}		
	where
	\begin{align}
		\delta(t)=&n_1\widetilde{\xi}(1,t)+n_2\widetilde{\eta}(1,t)+n_3\widetilde{X}(t)+n_4\widetilde{d}(t)+n_6\widetilde{\eta}_x(1,t) \notag \\	
		&+\int_{0}^{1}N_7(y)\widetilde{\xi}(y,t)dy+\int_{0}^{1}N_8(y)\widetilde{\eta}(y,t)dy \notag \\
		&+\int_{0}^{1}N_9(y)\widetilde{u}(y,t)dy+\int_{0}^{1}N_{10}(y)\widetilde{u}_y(y,t)dy \label{det(t)}
	\end{align}
	with $n_1$, $n_2$, $n_3$, $n_4$, $n_6$, $N_7$, $N_8$, $N_9$, $N_{10}$ defined in Appendix\ref{n}.
	
	\textit{1):} Following the proof of Theorem \ref{tosb}, we have $\|\widetilde{\eta}_x(\cdot,t)\|^2$ is bounded by $\|\widetilde{\alpha}(\cdot,t)\|^2+\|\widetilde{u}(\cdot,t)\|^2_{H^1}+|\widetilde{X}(t)|^2+|\widetilde{d}(t)|^2$. Therefore, using Cauchy-Schwarz inequality and Young's inequality, $\delta(t)\in H^1([0,\infty);R)$ is obtained from Property 2 of Theorem \ref{ertg} and \eqref{det(t)}. Then, recalling Property 1 of Theorem \ref{tosb}, it is straightforward to obtain Property 1.
	
	\textit{2):} Following the proof of Theorem \ref{tosb}, we obtain the exponential convergence to zero of $|X(t)|$. Recalling \eqref{det(t)}, \eqref{ereta}, \eqref{eretab}, and the exponential convergence in terms of the norm $\|\widetilde{\xi}(\cdot,t)\|^2_{H^1}+\|\widetilde{\eta}(\cdot,t)\|^2_{H^1}+\|\widetilde{u}(\cdot,t)\|^2_{H^1}+|\widetilde{X}(t)|^2+|\widetilde{d}(t)|^2$ proved in Theorem \ref{ertg}, using Cauchy-Schwarz inequality, Young's inequality and Sobolev inequality, we can obtain the exponential convergence of $\delta(t)$. Then, applying Theorem \ref{tosb}, \ref{ertg} and \eqref{widSt},
	property 2 is obtained for some positive $\Upsilon_{a}$.
	
	\textit{3):}  From properties 1, 2, It is clear that the output-feedback control input $U_{of}(t)$ is well-defined and bounded.
	
	The proof of Theorem $\ref{uof}$ is complete.
\end{proof}

\begin{Rem}
	According to \eqref{shearbon}--\eqref{trm3}, the vibration displacement $\varpi(x,t)$ can be expressed as $\varpi(x,t)=\int_{0}^{x}(\varphi_2(y)\xi(y,t)-\varphi_1(y)\eta(y,t))dx-\frac{(\sqrt{\epsilon}+\mathcal{Q}_2)}{2\sqrt{\epsilon}\mathcal{Q}_1}\xi(0,t)+$ $\frac{(\sqrt{\epsilon}-\mathcal{Q}_2)}{2\sqrt{\epsilon}\mathcal{Q}_1}\eta(0,t)$.  Recalling \eqref{trm0}--\eqref{trm3}, using Young's inequality, Theorem \ref{tosb} physically implicates that 1) the vibration displacement of the uncontrolled tip of the Timoshenko Beam $\varpi(0,t)$ converges to zero; 2) The angular velocity tracking error $\Delta\omega(t)$, the vibration displacement $\varpi(x,t)$, and the angle of rotation of the cross-section $\Phi(x,t)$ in \eqref{up0om}--\eqref{up3om}, and vibration energy of the overall system \eqref{st0}--\eqref{st3} exponentially decay to a bound that depends on the size $\bar{D}_d$ of the heat flux and is adjustable via the design parameters $\acute c_1$.
\end{Rem}
\section{Application}\label{simulation}
We apply our control design to a flexible blade rotor subjected to extreme thermal loads and unstable aerodynamic pressure, with the aim of suppressing the vibration energy of the overall plant and keeping the tip of the blade (the farthest point from the control) tracking the target motion, i.e., the elastic displacement converges to zero. 

\subsection{Model}
The simulation is conducted on the plant \eqref{ht}, \eqref{htbd}, \eqref{st3}, \eqref{xi}--\eqref{Xt}, where, according to \cite{bib9}, \cite{bib5}, \cite{bib1}, the physical parameters of the flexible blade rotor are given in Tab. \ref{smp}. We set the damping of the rotation disk as anti-damping $c^{\ast}=5$  here to evaluate the controller in a challenging scenario, where not only the plant PDE domain and the uncontrolled boundary but also the controlled (dynamic) boundary experience instabilities. Other parameters about the thermal difference and the aerodynamics instabilities are $\acute{\kappa}=1$, $q=[1$ $1]$, $\mathcal{Q}_1=1$, $\mathcal{Q}_2=1$, $k_1=1$, $k_2=1$, $\omega_d=0$, $A_d=[0, \pi;
-\pi, 0 ]$. The initial values are set as $d(0)=[S_0, S_0]'$, $z(0)=1$, $X(0)=2$, $\xi(x,0)=2\sin(2\pi x)$, $\eta(x,0)=2\sin(2\pi x)$.  
\begin{table}
	\centering
	\caption{Physical parameters of the blade-rotor.}
	\begin{tabular}{lll}
		\toprule
		Name & Value  & Unit   \\
		\midrule
		Young's modulus $E^\ast$    & 200  & $Gpa$ \\
		Modulus of rigidity $G^\ast$ & 77.5   & $Gpa$  \\
		Density $\rho^\ast$   & 7833 & $kg/m^3$  \\
		Cross-sectional area $A^\ast$ & 0.005 & $m^2$ \\
		Area moment of inertia $I^\ast$ & 0.00013876 &  $m^4$ \\
		Shear coefficient $k'$ & 0.53066 & $-$ \\
		Blade length $L^\ast$ & 1 & $m$ \\
		Disk radius $R^\ast$ & 0.5 & $m$ \\
		Inertia moment of the disk $J^{\ast}$ & 23.4375 & $kg\cdot m^2$ \\
		Thermal expansion $\alpha_0$ & 0.0000118 & $1$ $K^{-1}$ \\
		Average heat flux level $S_0$ & 7400 & $W/m^2$ \\
		Heat parameter $\beta^{\ast}$ & 0.0897 & $m$ \\
		\bottomrule
	\end{tabular}
	\label{smp}
\end{table}
\subsection{Controller}
The controller is \eqref{Uof(t)}, where the design parameter is $\acute{c}_1=10$, and the initial values of observer are set as $\hat{d}(0)=d(0)+[S_0,S_0]'$, $\hat{X}(0)=X(0)+1$, $\hat{\xi}(x,0)=\xi(x,0)+\sin(2\pi x)$, $\hat{\eta}(x,0)=\eta(x,0)+\sin(2\pi x)$. The simulation is performed using the finite-difference method for spatial and temporal discretization, with a time step and space step set at 0.001 and 0.05. Kernels defined by \eqref{kerK}--\eqref{kerPUpb}, \eqref{phi}--\eqref{K0b} are solved by power series \cite{bib45}, \cite{bib46} in the sense of 
\begin{align}
	k(x,y)=\sum_{i=0}^{N}\sum_{j=0}^{i}k_{ij}x^{i-j}y^j
\end{align}
where $N$ is chosen as 10.
\subsection{Results}
\begin{figure}[htbp]
	\centering
	\begin{subfigure}[b]{.45\linewidth}
		\centering
		\includegraphics[width=1.05\linewidth]{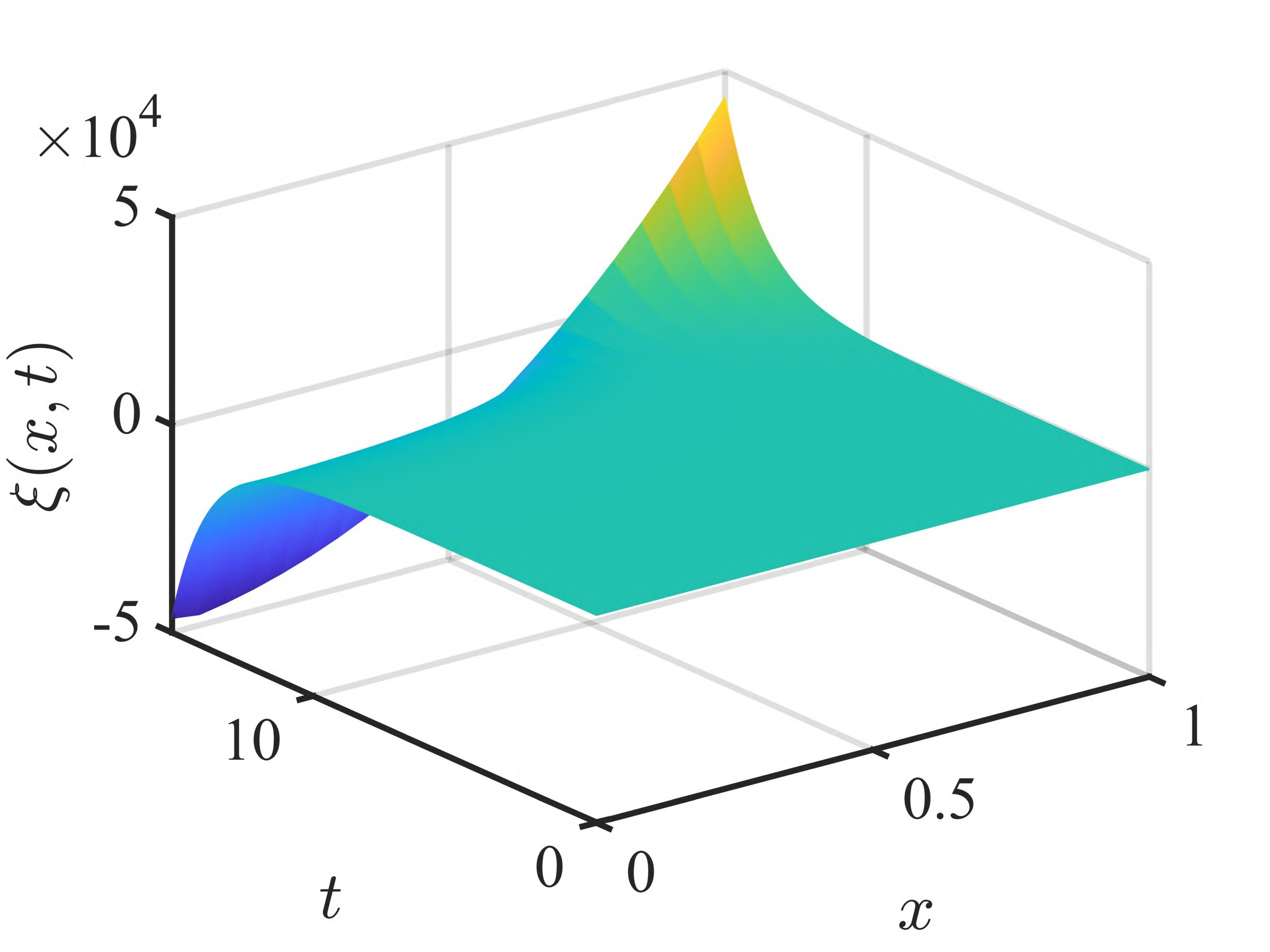}
	\end{subfigure}
	\begin{subfigure}[b]{.45\linewidth}
		\centering
		\includegraphics[width=1.05\linewidth]{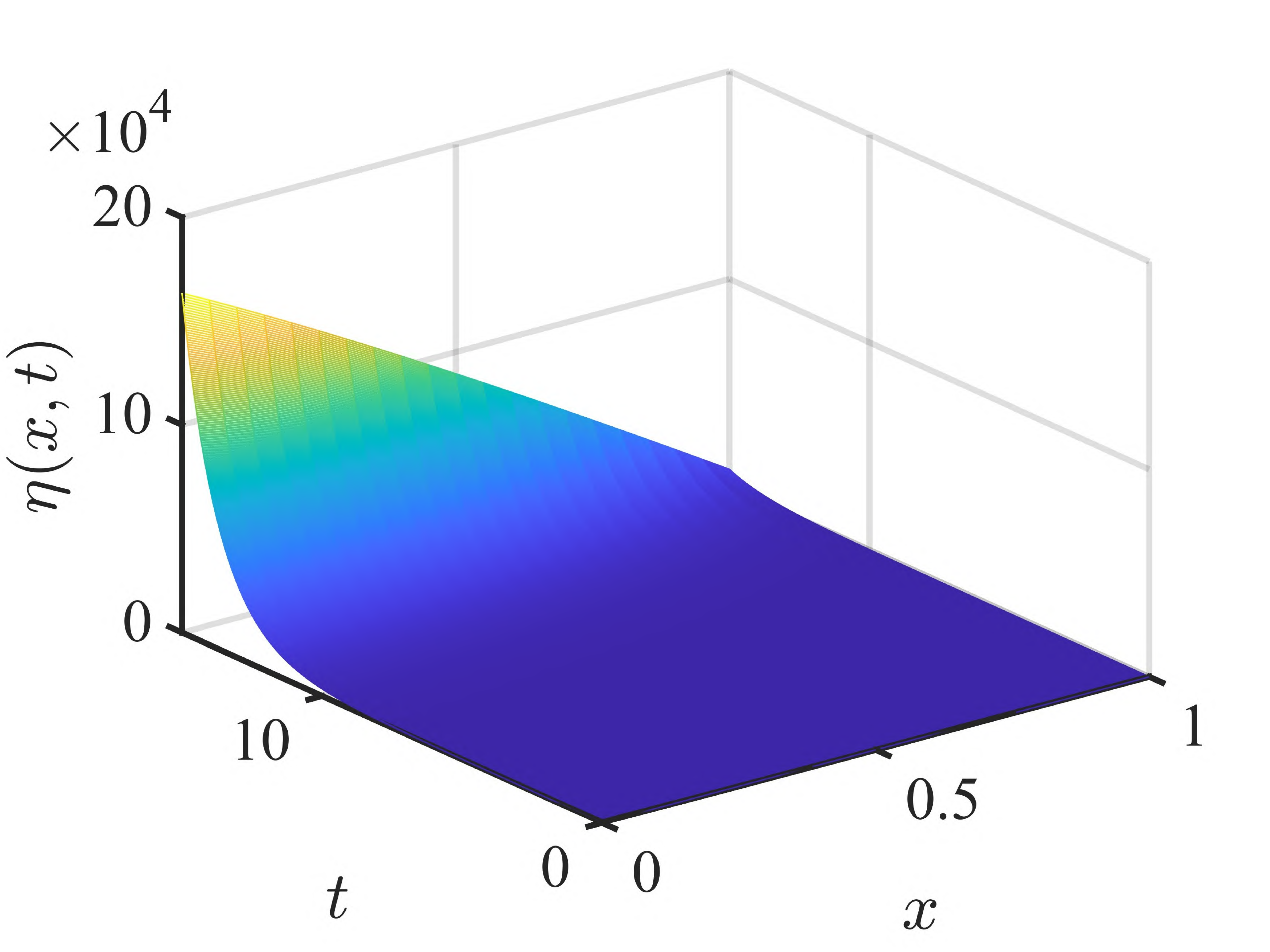}
	\end{subfigure}
	\caption{Open-loop responses of $\xi(x,t),\eta(x,t)$.}
	\label{openlp}
\end{figure}
\begin{figure}[htbp]
	\centering
	\begin{subfigure}[b]{.45\linewidth}
		\centering
		\includegraphics[width=1.05\linewidth]{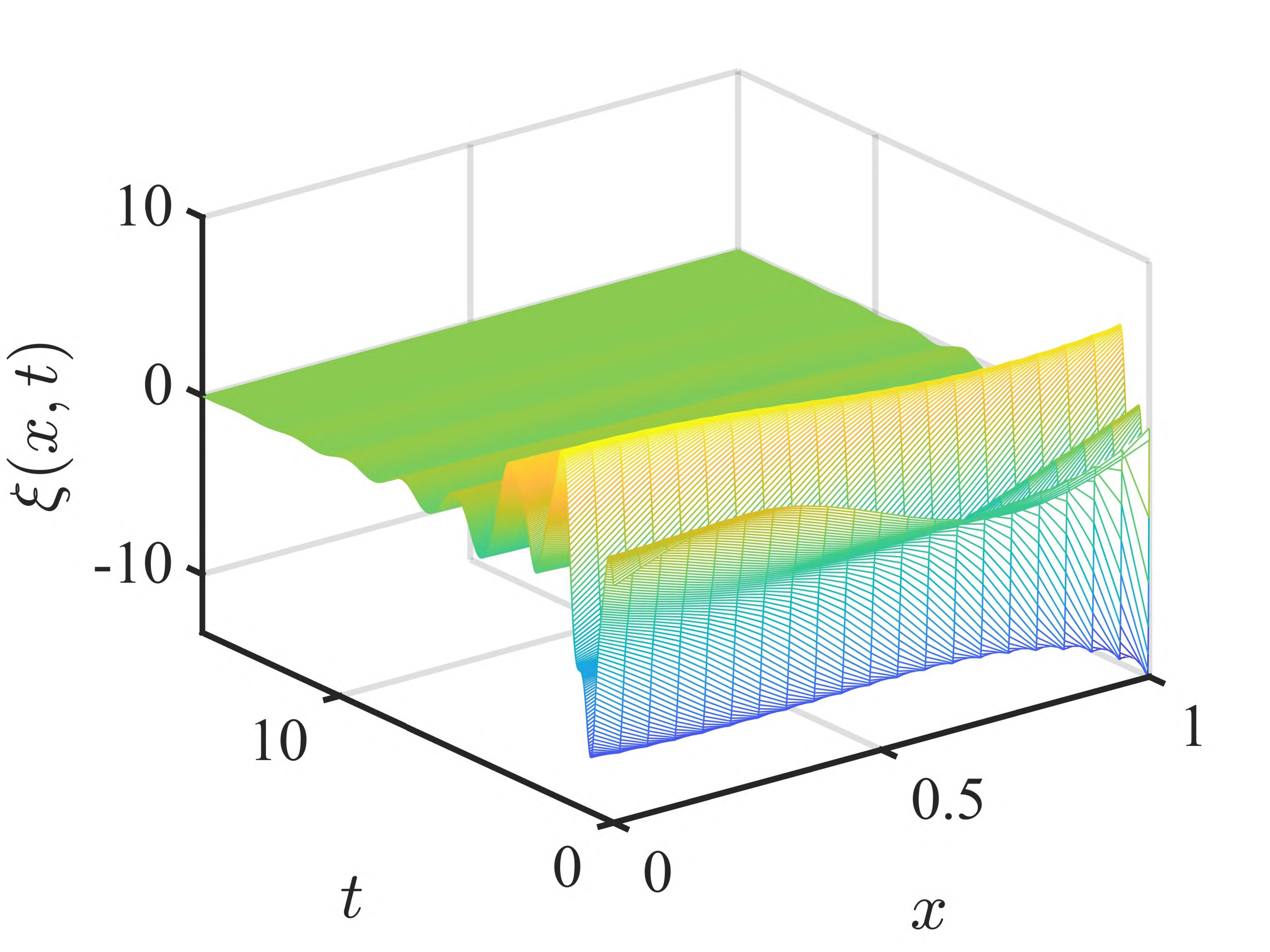}
	\end{subfigure}
	\begin{subfigure}[b]{.45\linewidth}
		\centering
		\includegraphics[width=1.05\linewidth]{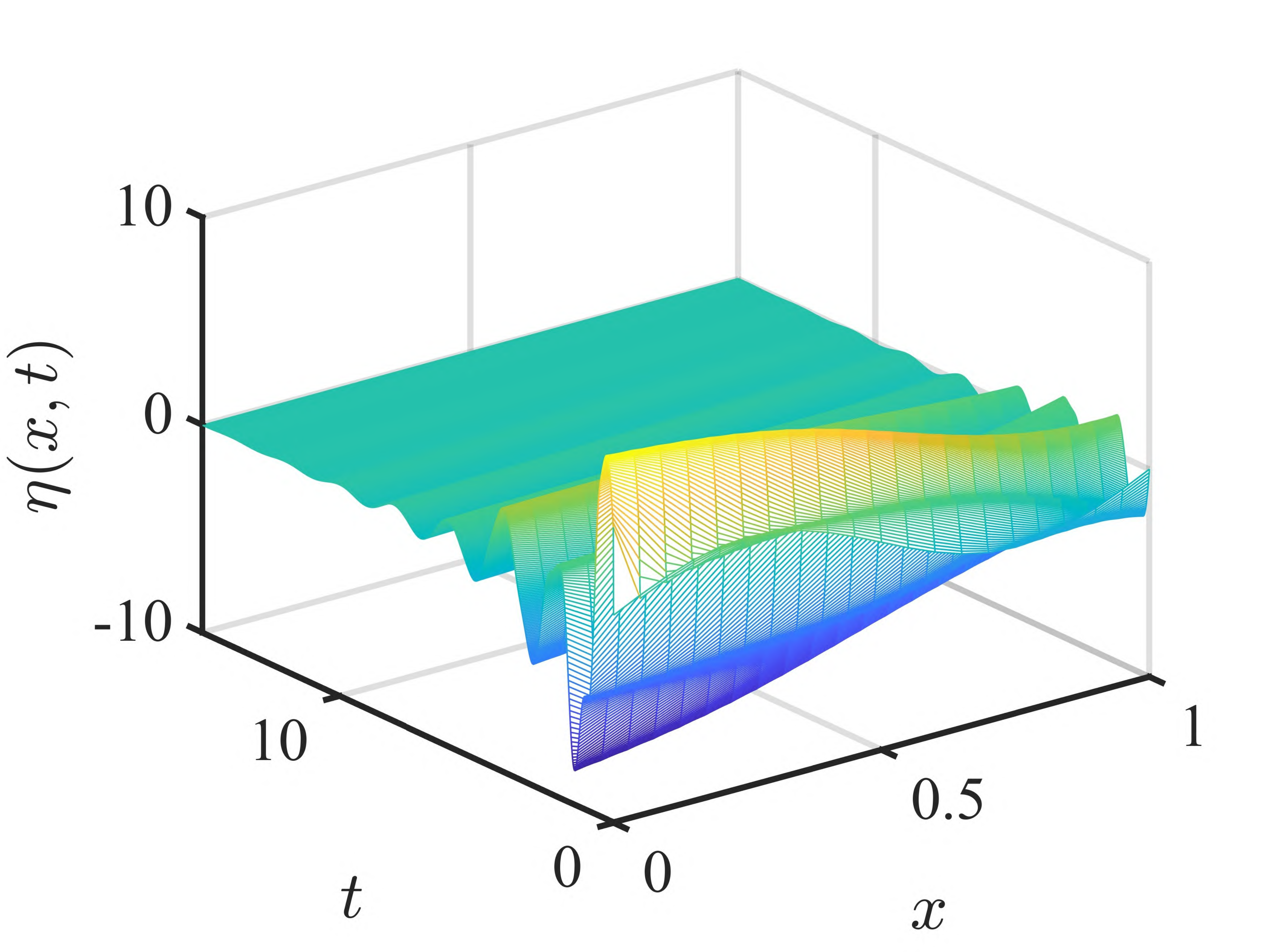}
	\end{subfigure}
	\caption{Responses of $\xi(x,t),\eta(x,t)$ under the proposed output-feedback controller.}
	\label{outptfk}
\end{figure} 
\begin{figure}[htbp]
	\centering
	\begin{subfigure}[b]{.45\linewidth}
		\centering
		\includegraphics[width=1.05\linewidth]{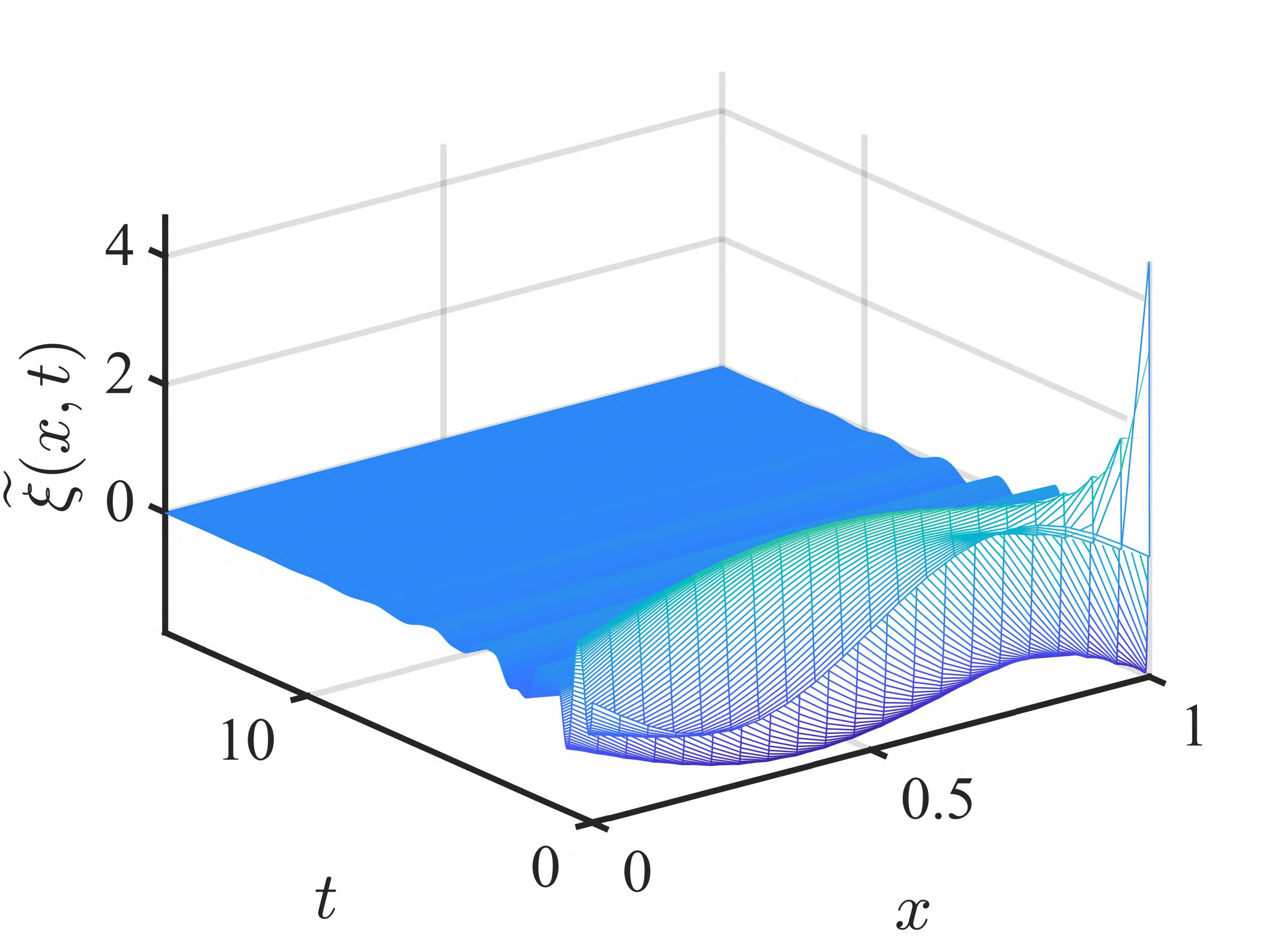}
	\end{subfigure}
	\begin{subfigure}[b]{.45\linewidth}
		\centering
		\includegraphics[width=1.05\linewidth]{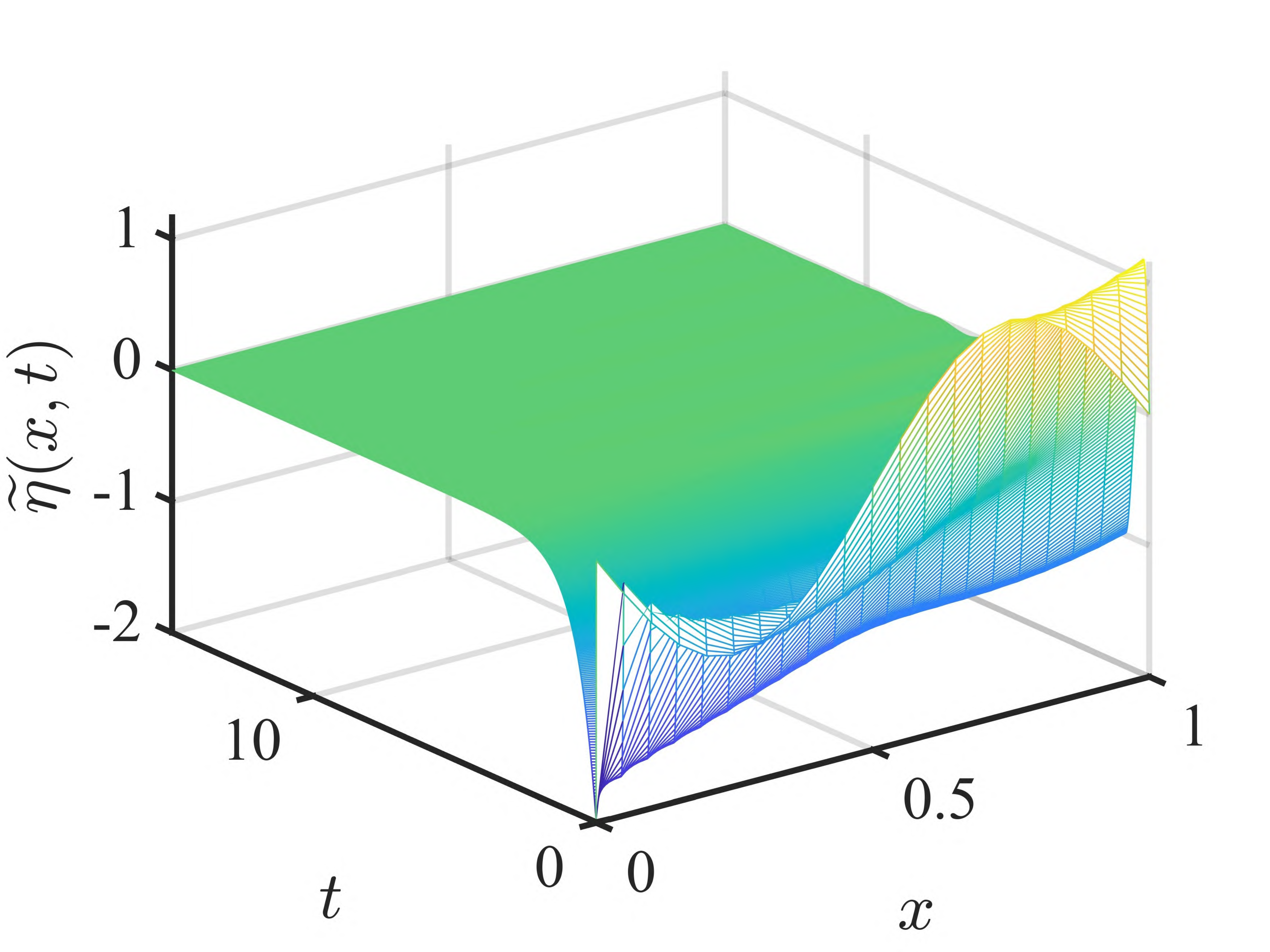}
	\end{subfigure}
	\caption{Observer error of $\xi(x,t),\eta(x,t)$.}
	\label{observer}
\end{figure} 
\begin{figure}[htbp]
	\centering
	\begin{subfigure}[b]{.45\linewidth}
		\centering
		\includegraphics[width=1.05\linewidth]{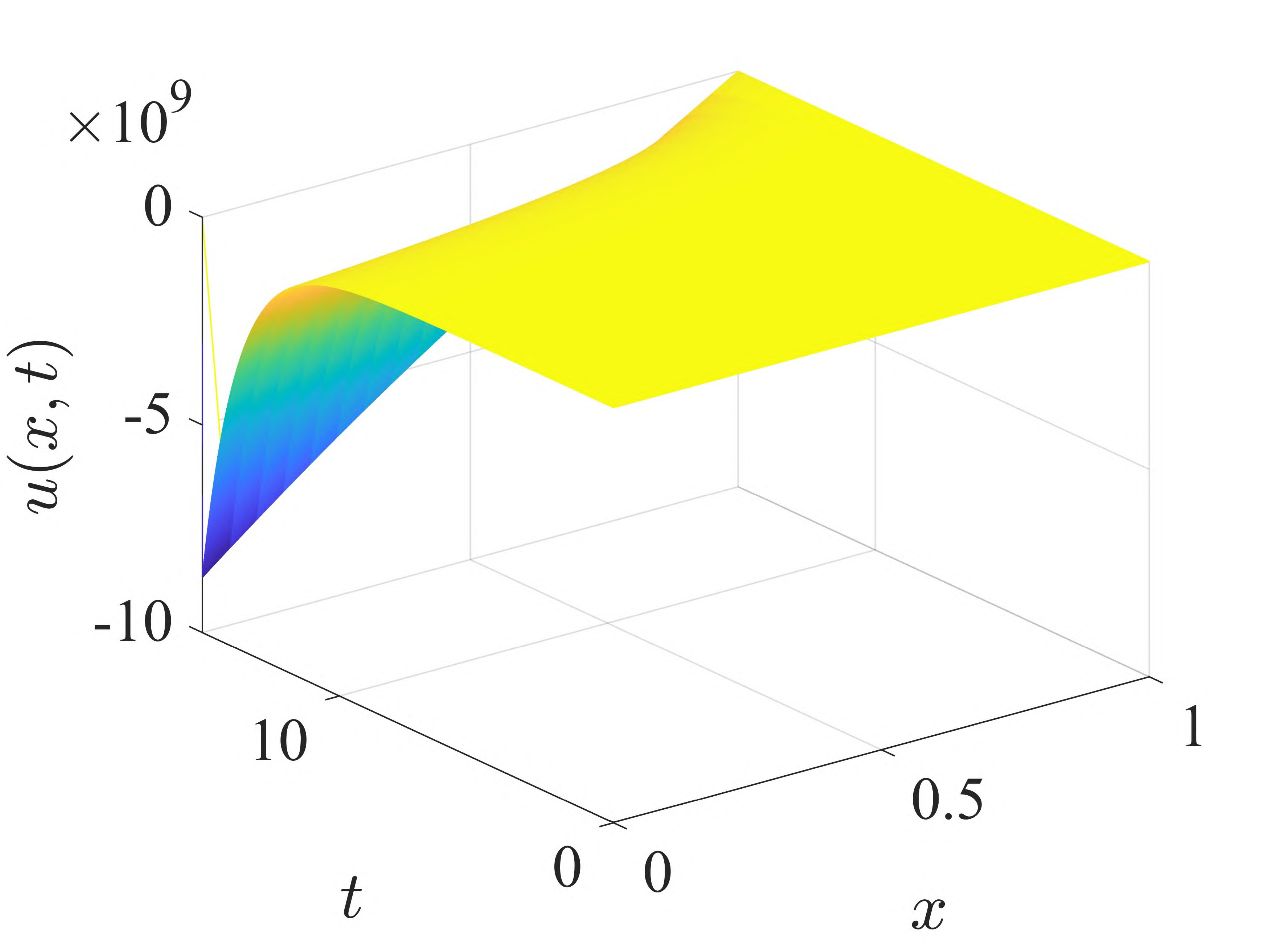}
	\end{subfigure}
	\begin{subfigure}[b]{.45\linewidth}
		\centering
		\includegraphics[width=1.05\linewidth]{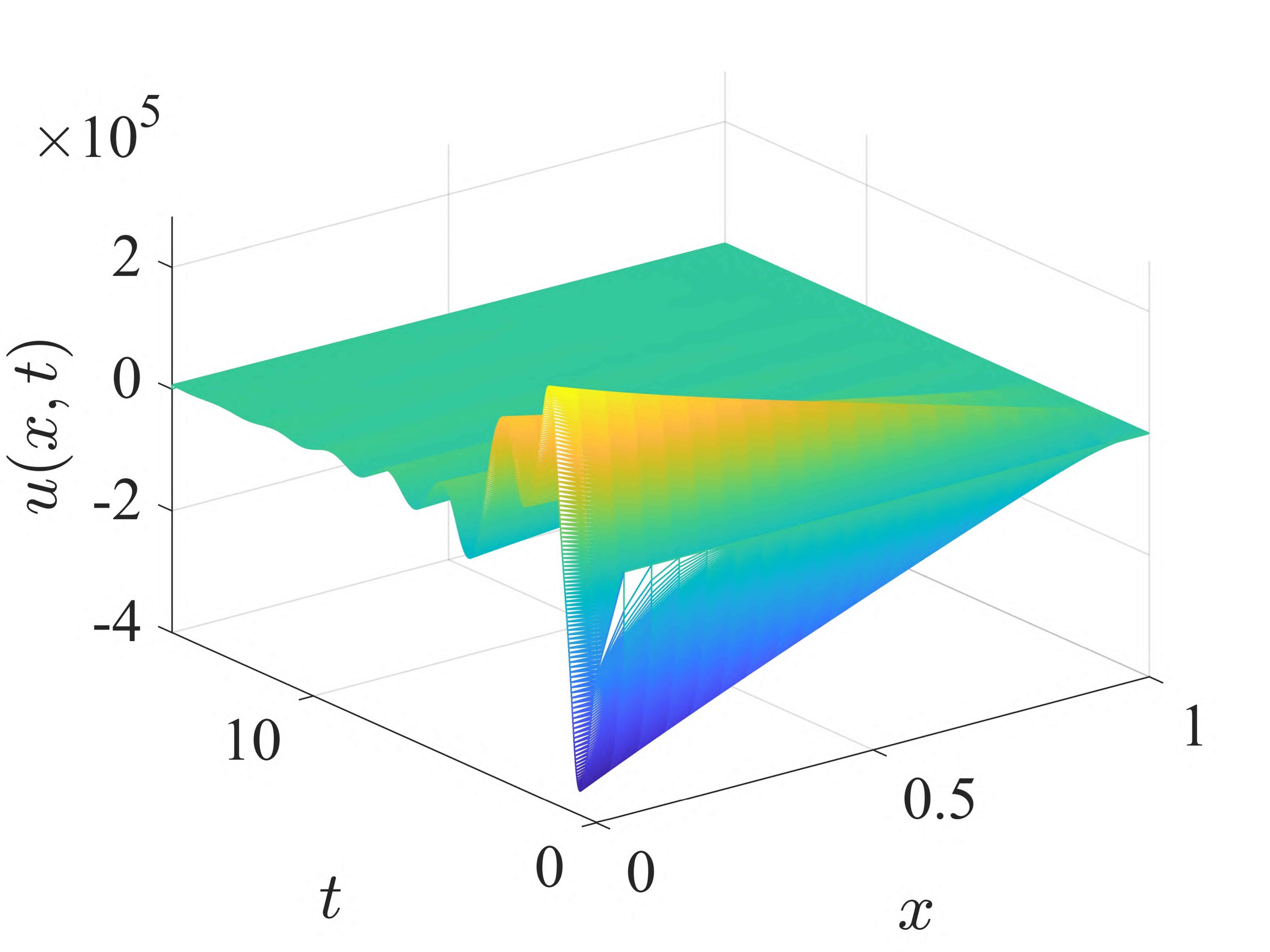}
	\end{subfigure}
	\caption{Responses of $u(x,t)$ under open loop control and output-feedback control.}
	\label{openoutu}
\end{figure} 
\begin{figure}[!ht]
	\centering
	\begin{subfigure}[b]{.45\linewidth}
		\centering
		\includegraphics[width=1.05\linewidth]{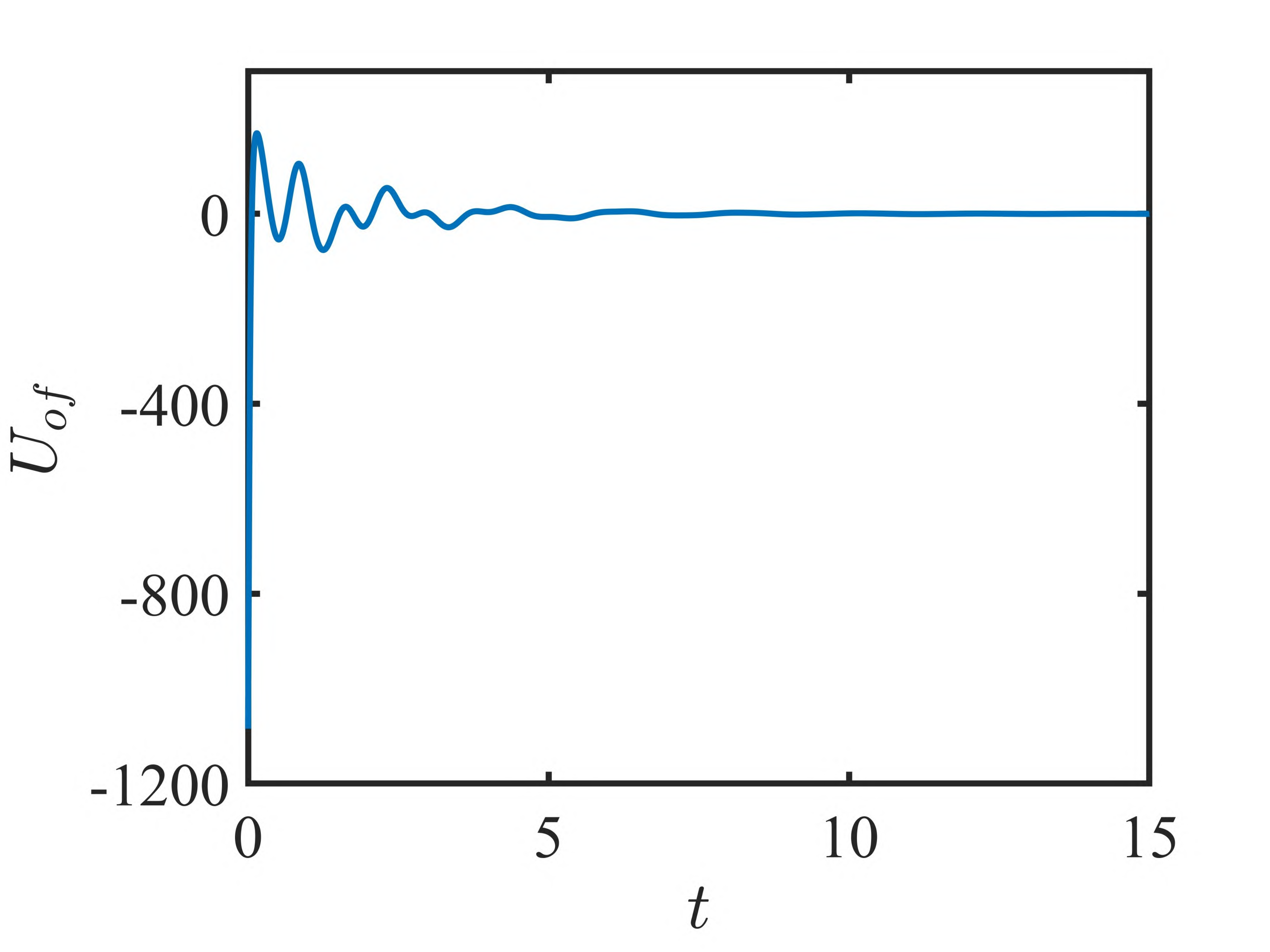}
	\end{subfigure}
	\begin{subfigure}[b]{.45\linewidth}
		\centering
		\includegraphics[width=1.05\linewidth]{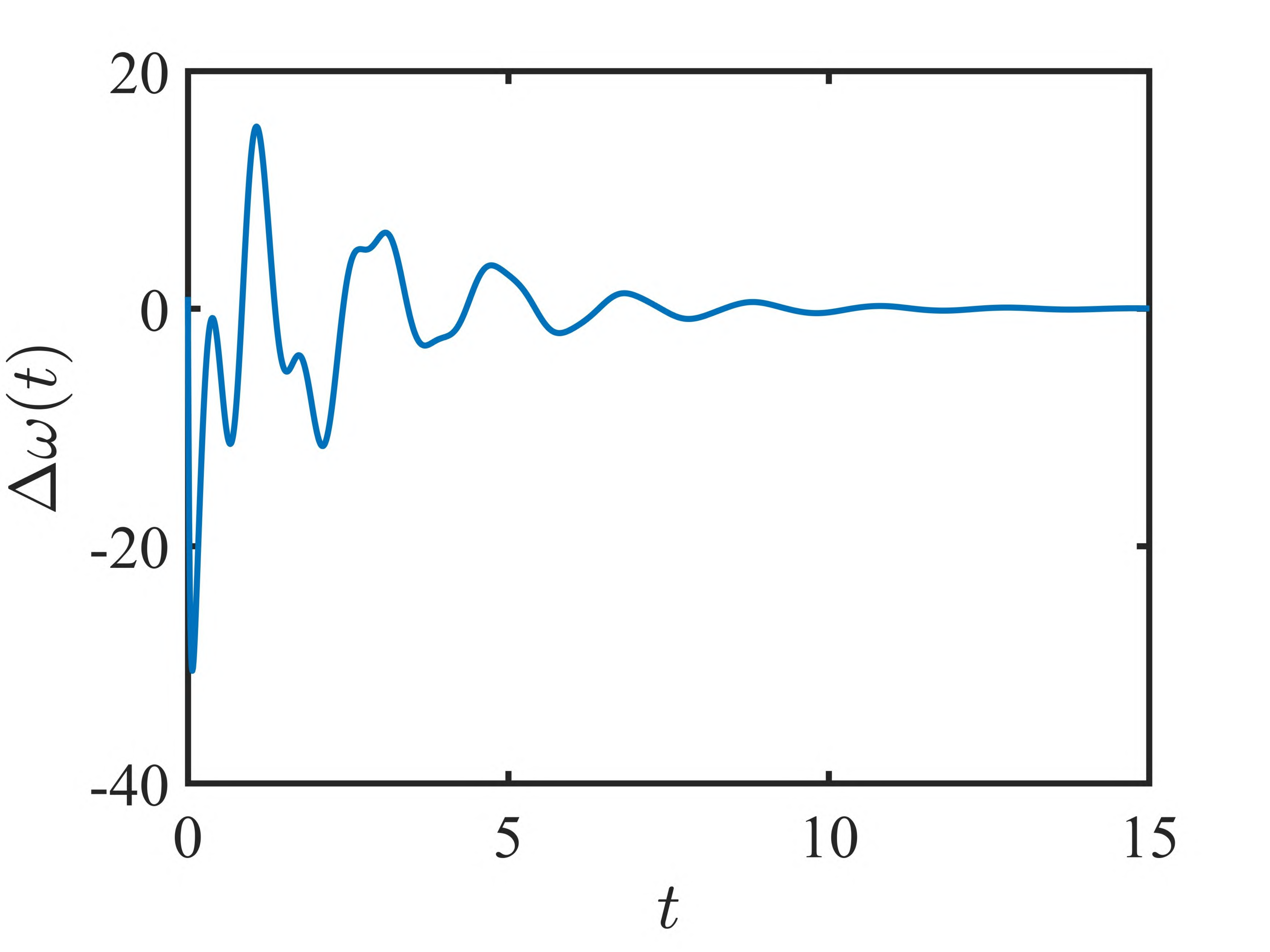}
	\end{subfigure}
	\caption{Output-feeedback control input $U_{of}$ and the disk angular velocity tracking error $\Delta\omega(t)$.}
	\label{outptU}
\end{figure}
\begin{figure}[!ht]
	\centering
	\includegraphics[width=0.8\linewidth]{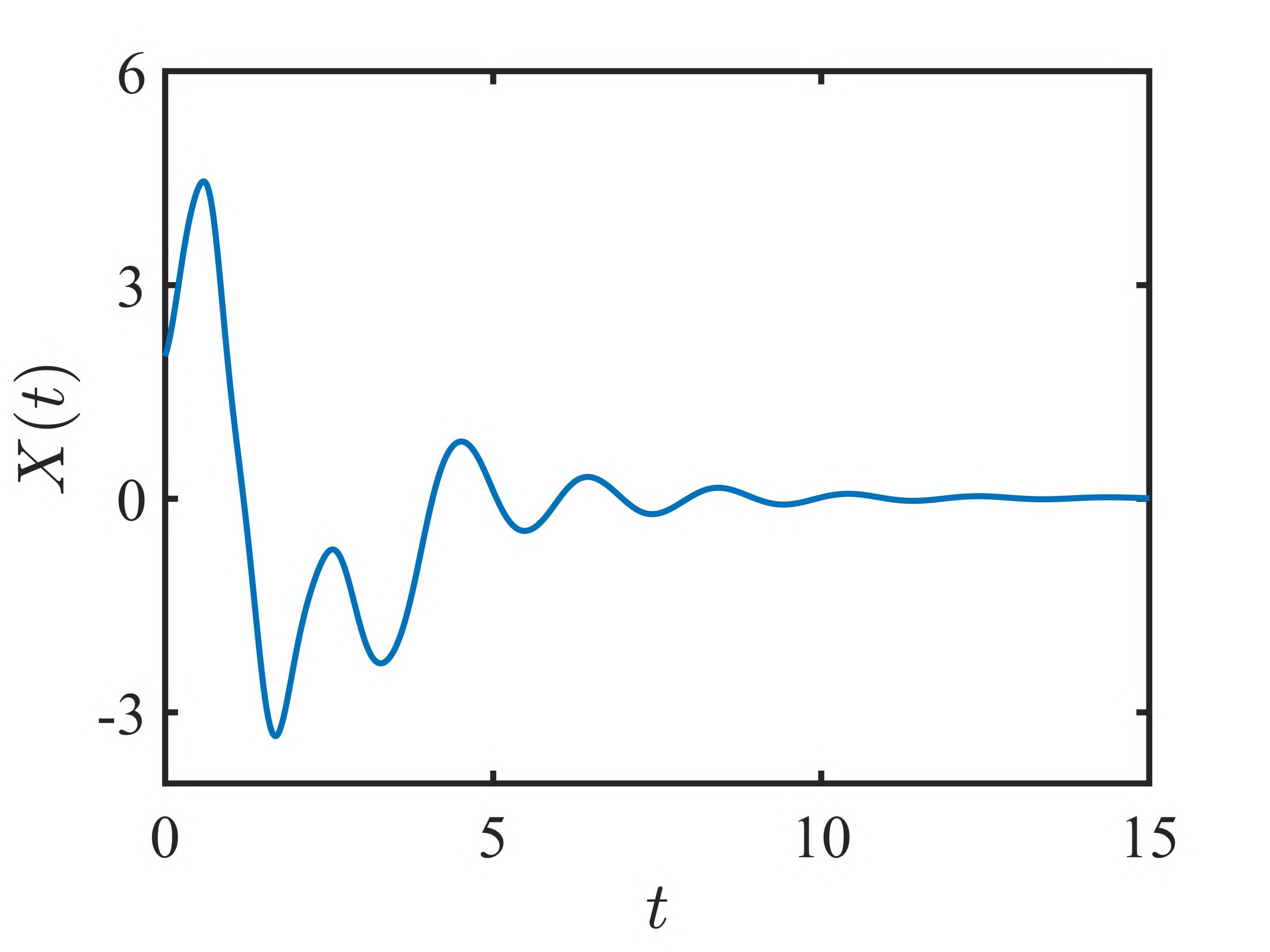}
	\caption{The tip vibration displacement $X(t)$ under the proposed output-feedback controller.}
	\label{outptxi}
\end{figure}
\begin{figure}[!ht]
	\centering
	\includegraphics[width=0.8\linewidth]{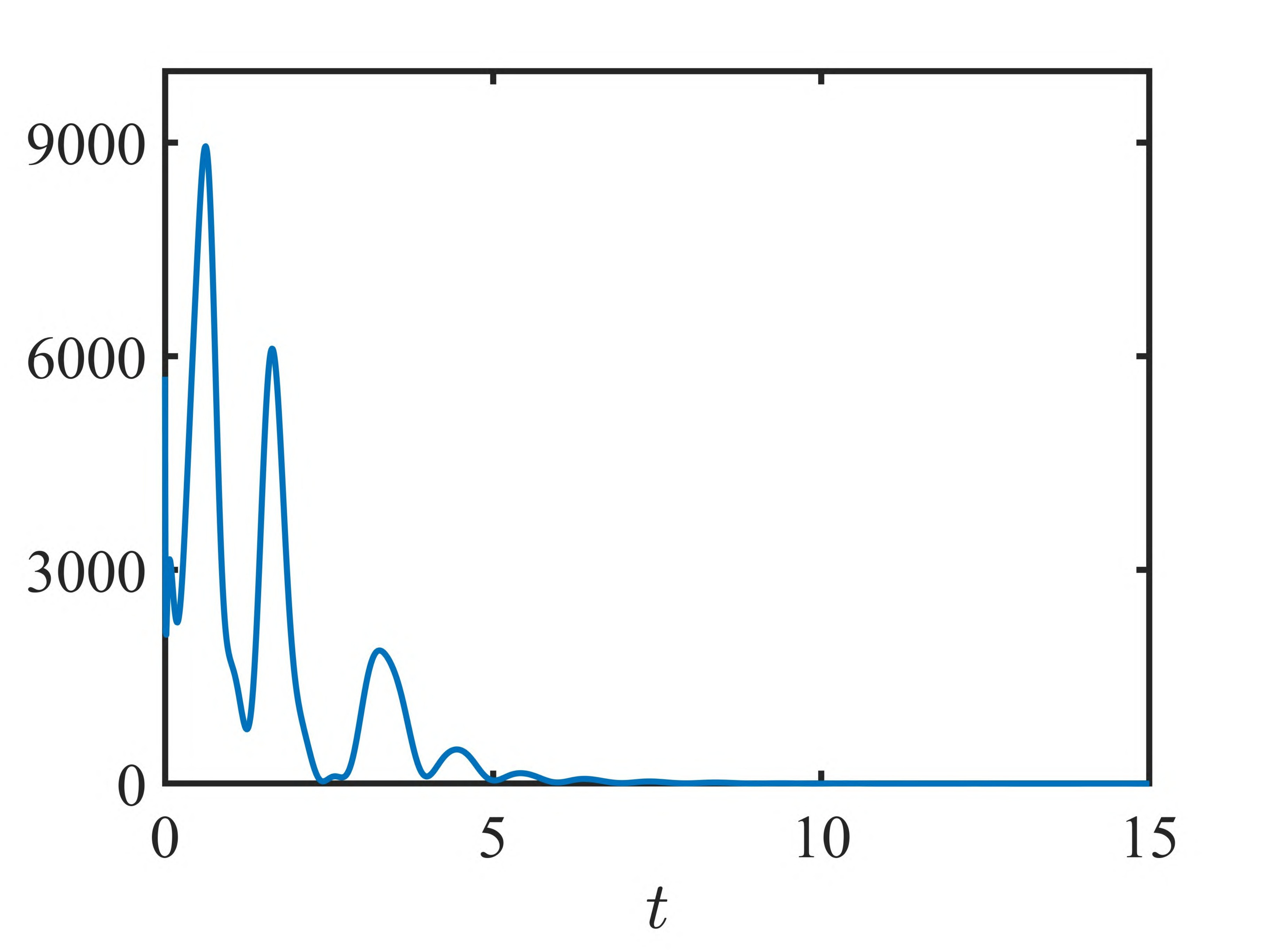}
	\caption{The sum of potential energy and kinetic energy of the overall physical model.}
	\label{Energy}
\end{figure}

Comparing Fig. $\ref{openlp}$ which illustrates the open-loop responses of $\xi(x,t),\eta(x,t)$ with Fig. $\ref{outptfk}$ which gives the responses of $\xi(x,t),\eta(x,t)$ under the proposed output-feedback controller, as one can be observed, in the latter case the fast decay is achieved, whereas the states grow unbounded in the former case. Additionally, Fig. $\ref{observer}$ demonstrates that the observer errors of $\xi(x,t)$, $\eta(x,t)$ are convergent to zero, i.e., the proposed observer converges to the actual plant for PDE states, and the responses of $u(x,t)$ under open loop control and output-feedback control are shown in Fig. \ref{openoutu}. Moreover, Fig. \ref{outptU} and Fig.\ref{outptxi} show the observer-based output-feedback control law $U_{of}$ applied in the simulation, under which the tip vibration displacement $X(t)$ and the angular velocity tracking error of the disk $\Delta \omega(t)$, i.e., the states of the two ODEs sandwiching the PDE,
where the fast decay is observed in both results as well. Fig. \ref{Energy} shows that the total energy of the overall physical model, including the potential energy from bending and shear as well as the kinetic energy from transverse deflection, the rotation of the cross-section, and the rotating disk, decays fast under the proposed controller.
\section{Conclusion}\label{conclusion}
In this article, we modeled the dynamics of thermally and flow-induced vibrations in a slender Timoshenko beam as a coupled system of hyperbolic and parabolic PDEs, using Hamilton's variational principle. A generalized system formulation is then adopted for the control design.  The state-feedback controller is designed by using the backstepping method. Relying solely on available boundary measurements, the observer for the hyperbolic-parabolic PDEs is constructed, and the exponential stability of the observer error system is shown. We then establish the observer-based output-feedback controller, under which we show that the distal ODE state, i.e., the furthest state from the control input, converges exponentially to zero, and all signals are uniformly ultimately bounded. Simulation results demonstrate that the proposed controller effectively and rapidly suppresses the vibration energy of an aero-engine flexible blade operating under extreme thermal loads and unstable aerodynamic pressure.

Future work will incorporate adaptive control design to address the thermally-induced and
flow-induced vibrations in the Timoshenko beam under uncertainties.

\section*{Appendix}\label{appendix}
\setcounter{subsection}{0}
\setcounter{section}{0}
\subsection{Expression of $\mathcal{F}_{11}\thicksim \mathcal{D}_2(x)$}\label{F}
\setcounter{equation}{0}
\renewcommand{\theequation}{A.\arabic{equation}}
The functions $\mathcal{F}_{11}(x,y)$, $\mathcal{F}_{12}(x,y)$, $\mathcal{F}_{13}(x,y)$, $\mathcal{D}_1(x)$, $\mathcal{D}_2(x)$ in \eqref{gtsbt}--\eqref{gstx} are shown as follows
\begin{align}
	\mathcal{F}_{11}(x,y)=&f_{11}(x,y)+\int_{y}^{x}f_{12}(x,z)\sigma(z,y)dz\notag \\
	&+c_1(x)\sigma(x,y), \label{mF11} \\
	\mathcal{F}_{12}(x,y)=&f_{12}(x,y)+\int_{y}^{x}f_{12}(x,z)\rho(z,y)dz\notag \\
	&+c_1(x)\rho(x,y), \label{mF12} \\
	\mathcal{F}_{13}(x,y)=&h(x-y)f_{13}(x,y)+\int_{0}^{x}f_{12}(x,z)\varrho(z,y)dz\notag \\
	&+c_1(x)\varrho(x,y), \label{mF13} \\
	\mathcal{D}_1(x)=&D_1(x)-\int_{0}^{x}f_{12}(x,y)\lambda(y)dy+g_1(x)K\notag \\
	&-c_1(x)\lambda(x), \label{mD1} \\
	\mathcal{D}_2(x)=&G_1(x)-c_1(x)\vartheta(x)-\int_{0}^{x}f_{12}(x,y)\vartheta(y)dy. \label{mD2}
\end{align}
\subsection{Expression of $h_1\thicksim H_{10}$}\label{h}
\setcounter{equation}{0}
\renewcommand{\theequation}{B.\arabic{equation}}
Expressions of  $h_1$, $h_2$, $h_3$, $h_4$, $h_5$, $h_6$, $H_7$, $H_8$, $H_9$, $H_{10}$ in \eqref{bet} are shown as follows
\begin{align}
	h_1=&c_0+q_0c_1-\frac{\rho(1,1)}{\varepsilon_2}, \label{h1} \\
	h_2=&\frac{1}{\varepsilon_2}\rho(1,0)+\lambda(1)B-\frac{1}{\varepsilon_1}\int_{0}^{1}\sigma(1,y)g_1(y)dy, \\
	h_3=&\frac{1}{\varepsilon_1}\sigma(1,1)+c_0, \quad h_4=-\frac{1}{\varepsilon_1}\sigma(1,0), \\
	h_5=&\lambda(1)(A+BK)-\frac{1}{\varepsilon_1}\int_{0}^{1}\sigma(1,y)\mathcal{D}_1(y)dy \notag \\
	&-(q_0c_1+c_0)\lambda(1), \\
	h_6=&-\frac{1}{\varepsilon_1}\int_{0}^{1}\sigma(1,y)\mathcal{D}_2(y)dy-(q_0c_1+c_0)\vartheta(1) \notag \\
	&+\vartheta(1)A_d \\
	H_7(y)=&(q_0c_1+c_0)\rho(1,y)+\frac{1}{\varepsilon_2}\rho_y(1,y)-\frac{1}{\varepsilon_1}c_1(y)\sigma(1,y)\notag \\
	&-\frac{1}{\varepsilon_1 }\int_{0}^{1}\mathcal{F}_{12}(z,y)\sigma(1,z)dz ,  \\
	H_8(y)=&-\frac{1}{\varepsilon_1}\sigma_y(1,y)-\frac{1}{\varepsilon_1}\int_{y}^{1}\mathcal{F}_{11}(z,y)\sigma(1,z)dz \notag \\
	&+(q_0c_1+c_0)\sigma(1,y), \\
	H_9(y)=&-\frac{1}{\varepsilon_1}\int_{0}^{1}\mathcal{F}_{13}(z,y)\sigma(1,z)dz-\frac{1}{\varepsilon_1}\sigma(1,y)\mu_1(y) \notag \\
	&+(q_0c_1+c_0)\varrho(1,y), \\
	H_{10}(y)=&\varrho_y(1,y)\kappa_0.
\end{align}
\subsection{Expression of $n_1\thicksim N_{10}$}\label{n}
\setcounter{equation}{0}
\renewcommand{\theequation}{C.\arabic{equation}}
Expressions of  $n_1$, $n_2$, $n_3$, $n_4$, $n_5$, $n_6$, $N_7$, $N_8$, $N_9$ in \eqref{U(t)} are given as follows
\begin{align}
	n_1=&\frac{1}{q_0}(-(\acute{c}_1+h_1)+\frac{q_1}{\varepsilon_1}c_1(1)),\quad n_2=-\frac{1}{q_0}h_3, \label{eq:n1}\\
	n_3=&\frac{1}{q_0}\bigg(-(\acute{c}_1+h_1)\gamma(1)+\frac{q_1}{\varepsilon_1}D_1(1)-h_2\gamma(0)-h_4C\notag \\
	&-h_5-\int_{0}^{1}H_7(y)\gamma(y)dy\bigg), \\
	n_4=&\frac{1}{q_0}\bigg(-(\acute{c}_1+h_1)\varUpsilon(1)-h_6-\int_{0}^{1}H_7(y)\varUpsilon(y)dy \notag \\
	&-h_2\varUpsilon(0)\bigg), \\
	n_5=&\frac{q_1}{\varepsilon_1q_0}(g_1(1)-h_2-h_4),\quad n_6=-\frac{q_1}{\varepsilon_1q_0}, \\
	N_7(y)=&\frac{1}{q_0}\bigg((\acute{c}_1+h_1)k(1,y)+\frac{q_1}{\varepsilon_1}f_{12}(1,y)-H_7(y)\notag \\
	&+\int_{y}^{1}k(z,y)H_7(z)dz\bigg), \\
	N_8(y)=&\frac{1}{q_0}\bigg((\acute{c}_1+h_1)l(1,y)+\frac{q_1}{\varepsilon_1}f_{11}(1,y)-H_8(y)\notag \\
	&+\int_{y}^{1}l(z,y)H_7(z)dz\bigg), \\
	N_9(y)=&\frac{1}{q_0}\bigg((\acute{c}_1+h_1)p(1,y)+\frac{q_1}{\varepsilon_1}f_{13}(1,y)\notag \\
	&+\int_{0}^{1}p(z,y)H_7(z)dz-H_9(y)\bigg), \\
	N_{10}(y)=&\frac{1}{q_0}H_{10}(y). \label{eq:n10}
\end{align}
\subsection{Proof of Lemma \ref{kernel klp}}\label{pfk1}
\setcounter{equation}{0}
\renewcommand{\theequation}{D.\arabic{equation}}
 In this section, we will prove the well-posedness of \eqref{kerK}--\eqref{kerPUpb} in a weak sense. 
	Recalling \eqref{kerGma}, \eqref{kerLGmab}, $\gamma(x)$ can be expressed as  
	\begin{align}
		&\gamma(x)=-Ke^{\varepsilon_2Ax}-\int_{0}^{x}\varepsilon_2\kappa_0\acute{p}_2p_y(s,0)e^{\varepsilon_2A(x-s)}ds \notag \\
		&-\int_{0}^{x}\int_{0}^{s}\bigg(k(s,y)D_2(y)+l(s,y)\frac{\varepsilon_2}{\varepsilon_1}D_1(s)\bigg)e^{\varepsilon_2A(x-s)}dyds \notag \\
		&-\frac{\varepsilon_2}{\varepsilon_1}C\int_{0}^{x}e^{\varepsilon_2A(x-s)}l(s,0)ds+\int_{0}^{x}D_2(s)e^{\varepsilon_2A(x-s)}ds. \label{eq:gamma}
\end{align}

Then we consider the following sequence of functions
	\begin{align}
		l(x,y)&=\sum_{m=0}^{\infty}l_m(x,y),\quad k(x,y)=\sum_{n=0}^{\infty}k_m(x,y),\notag \\
		p(x,y)&=\sum_{m=0}^{\infty}p_m(x,y). \label{eq:lkp}
\end{align}

Recalling \eqref{kerK}--\eqref{kerLGmab} and \eqref{eq:gamma}, for m=0, we get the PDEs
	\begin{align}
		k_{0,x}(x,y)&=-k_{0,y}(x,y)-f_{22}(x,y), \label{eq:k0} \\
		l_{0,x}(x,y)&=\frac{\varepsilon_2}{\varepsilon_1}l_{0,y}(x,y)-f_{21}(x,y), \label{eq:l0} \\
		k_0(x,0)&=-g_2(x)-\varepsilon_2\int_{0}^{x}D_2(s)e^{\varepsilon_2A(x-s)}ds B \notag \\
		&+\varepsilon_2Ke^{\varepsilon_2Ax}B, \label{eq:k0b} \\
		l_0(x,x)&=-\frac{\varepsilon_1}{\varepsilon_1+\varepsilon_2}c_2(x), \label{eq:l0b}  \\
		p_{0,x}(x,y)&=\varepsilon_2 \kappa_0 p_{0,yy}(x,y)+h(x-y)f_{23}(x,y)\notag \\
		&-\delta(y-x)\mu_2(y), \label{eq:p0} \\ 
		p_0(x,1)&=0,\quad p_0(x,0)=0,\quad p_0(0,y)=0, \label{eq:p0b}
	\end{align}
	and for $m=1,2,\dots,\infty$, the following functions are defined in the terms of $k_{m-1},l_{m-1},p_{m-1},$
	\begin{align}
		&k_{m,x}(x,y)=-k_{m,y}(x,y)+\int_{y}^{x}f_{22}(z,y)k_{m-1}(x,z)dz\notag \\
		&+\frac{\varepsilon_2}{\varepsilon_1}\int_{y}^{x}f_{12}(z,y)l_{m-1}(x,z)dz+\frac{\varepsilon_2}{\varepsilon_1}l_{m-1}(x,y)c_1(y), \label{eq:km} \\
		&l_{m,x}(x,y)=\frac{\varepsilon_2}{\varepsilon_1}l_{m,y}(x,y)+\int_{y}^{x}f_{21}(z,y)k_{m-1}(x,z)dz\notag \\
		&+\frac{\varepsilon_2}{\varepsilon_1}\int_{y}^{x}f_{11}(z,y)l_{m-1}(x,z)dz+k_{m-1}(x,y)c_2(y), \label{eq:lm} \\
		&k_m(x,0)=\int_{0}^{x}(\frac{\varepsilon_2}{\varepsilon_1}l_{m-1}(x,y)g_1(y)+k_{m-1}(x,y)g_2(y))dy \notag \\
		&+\frac{\varepsilon_2}{\varepsilon_1}l_{m-1}(x,0)+\frac{\varepsilon_2^2}{\varepsilon_1}C\int_{0}^{x}e^{\varepsilon_2A(x-s)}l_{m-1}(s,0)Bds \notag \\
		&+\int_{0}^{x}\int_{0}^{s}(\varepsilon_2k_{m-1}(s,y)D_2(y)+\frac{\varepsilon_2^2}{\varepsilon_1}l_{m-1}(s,y)D_1(y))\notag \\
		&e^{\varepsilon_2A(x-s)}Bdyds+\varepsilon_2^2\kappa_0\acute{p}_2\int_{0}^{x}e^{\varepsilon_2A(x-s)}p_{m-1,y}(s,0)Bds, \label{eq:kmb} \\
		&l_m(x,x)=0, \label{eq:lmb}  \\
		&p_{m,x}(x,y)=\varepsilon_2 \kappa_0 p_{m,yy}(x,y)+h(x-y)( k_{m-1}(x,y)\mu_2(y) \notag \\
		&\frac{\varepsilon_2}{\varepsilon_1}l_{m-1}(x,y)\mu_1(y)+\frac{\varepsilon_2}{\varepsilon_1}\int_{y}^{x}f_{13}(z,y)l_{m-1}(x,z)dz \notag \\
		&+\int_{y}^{x}f_{23}(z,y)k_{m-1}(x,z)dz ), \label{eq:pm} \\ 
		&p_m(x,1)=0,\quad p_m(x,0)=0,\quad p_m(0,y)=0. \label{eq:pmb}
\end{align}
The explicit solution of $p_0(x,y)$ is found based on the method of separation of variables in \cite{bib28} by using a Fourier sine series $p_0(x,y)=\sum_{n=0}^{\infty}A_n(x)\sin(n\pi y)$. Recalling \eqref{kerP}, \eqref{kerPb}, $p_0(x,y)$ is obtained as
	\begin{align}
		&p_0(x,y)=2\sum_{n=0}^{\infty}e^{-\varepsilon_2\kappa_0n^2\pi^2x}\sin(n\pi y) \notag \\
		&\bigg( 
		\int_{0}^{x}\frac{\int_{0}^{s}f_{23}(s,\tau)\sin(n\pi\tau)d\tau-\mu_2(s)\sin(n\pi s)}{e^{-\varepsilon_2\kappa_0n^2\pi^2s}} ds\bigg) \label{eq:p}
	\end{align}
	which is convergent since $|p_0(x,y)|\leq \frac{2(\bar{f}+\bar{\mu})}{\varepsilon_2\kappa_0\pi^2}\sum_{n=0}^{\infty}\frac{1}{n^2}$ where
	$\bar{f}=\max \{ |\frac{\varepsilon_2}{\varepsilon_1}\mu_2(x)|_{\infty}, \|f_{23}(x,y)\|_{\infty}, \|\frac{\varepsilon_2}{\varepsilon_1}f_{13}(x,y)\|_{\infty},$ $|\frac{\varepsilon_2}{\varepsilon_1}\mu_1(x)|_{\infty}, \}$ and $\bar{\mu}=\max\{ |c_1(x)|_{\infty}, |c_2(x)|_{\infty}, |\mu_2(x)|_{\infty},$ $|\frac{\varepsilon_2}{\varepsilon_1}c_1(x)|_{\infty},\| f_{22}(x,y)\|_{\infty},\|f_{21}(x,y)\|_{\infty},\|\frac{\varepsilon_2}{\varepsilon_1}f_{12}(x,y)\|_{\infty},$ $\|\frac{\varepsilon_2}{\varepsilon_1}f_{11}(x,y)\|_{\infty} \}.$
	Therefore, the expression $p_0(x,y)$ in \eqref{eq:p} make sense. Similarly, $p_{0,y}(x,y)$ is expressed as
	\begin{align}
		&p_{0,y}(x,y)=2\sum_{n=0}^{\infty}n\pi e^{-\varepsilon_2\kappa_0n^2\pi^2x}\cos(n\pi y) \notag \\
		&\bigg( 
		\int_{0}^{x}\frac{\int_{0}^{s}f_{23}(s,\tau)\sin(n\pi\tau)d\tau-\mu_2(s)\sin(n\pi s)}{e^{-\varepsilon_2\kappa_0n^2\pi^2s}} ds\bigg) \label{eq:p0y}
	\end{align}
	and we have $p_{0,y}\in L^2([0,1]\times[0,1])$ since $\int_{0}^{1}p_{0,y}^2(x,y)dy \leq \frac{4(\bar{f}+\bar{\mu})^2}{\varepsilon_2^2\kappa_0^2\pi^2}\sum_{n=0}^{\infty}\frac{1}{n^2}$ by using Parseval inequality.

According to \cite{bib24} and using the method of characteristic line, we get the explicit solutions of $l_0(x,y),k_0(x,y)$. Along the line $x=-\frac{\varepsilon_1}{\varepsilon_2}y+\tau_1$, we have
	\begin{align}
		l_0(x,y)=&-\frac{\varepsilon_1}{\varepsilon_1+\varepsilon_2}c_2(\frac{\varepsilon_2}{\varepsilon_1+\varepsilon_2}(x+\frac{\varepsilon_1}{\varepsilon_2}y))\notag \\
		&-\int_{\frac{\varepsilon_2}{\varepsilon_1+\varepsilon_2}}^{x}f_{21}(s,\frac{\varepsilon_2}{\varepsilon_1}(x+\frac{\varepsilon_2}{\varepsilon_1}y-s)) \label{eq:l}
	\end{align}
	and along the line $x=y+\tau_2$, we have
	\begin{align}
		k_0(x,y)=&-g_2(x-y)-\varepsilon_2\int_{0}^{x-y}D_2(s)e^{\varepsilon_2A(x-y-s)}Bds \notag \\
		&+\varepsilon_2Ke^{\varepsilon_2A(x-y)}B-\int_{x-y}^{x}f_{22}(s,x-y+s)ds. \label{eq:k}
\end{align}

Recalling \eqref{eq:p}--\eqref{eq:k}, we get that $l_0\in L^2(\mathcal{D})$, $k_0\in L^2(\mathcal{D})$, $p_0\in L^2(0,1,H^1(0,1))$. According to \eqref{eq:km}--\eqref{eq:pmb}, the explicit solutions of $l_m(x,y)$, $k_m(x,y)$, $p_m(x,y)$, $m=1,2,\dots,\infty$ which are determined by $l_{m-1}(x,y)$, $k_{m-1}(x,y)$, $p_{m-1}(x,y)$, $p_{m-1,y}(x,0)$ are obtained in the same way as for $l_0(x,y)$, $k_0(x,y)$, $p_0(x,y)$ and thus are omitted. Moreover, we can also iteratively infer that $l_m\in L^2(\mathcal{D})$, $k_m\in L^2(\mathcal{D})$, $p_m\in L^2(0,1;H^1(0,1))$, $m=1,2,\dots,\infty$.

In the following part of the section, the convergence of the series of \eqref{eq:lkp} in the same space of  $l(x,y),k(x,y),p(x,y)$ is proved. 
	Considering a functional as
	\begin{align}
		\widetilde{V}_1(x)=&\frac{1}{2}\int_{0}^{1}(p_1^2(x,y)+p_{1,y}^2(x,y))dy+\int_{0}^{x}(\acute{\theta}_1k_1^2(x,y) \notag \\
		&+\acute{\theta}_2l_1^2(x,y))dy \label{eq:bV1}
	\end{align}
	where $\acute{\theta}_1$, $\acute{\theta}_2$ are positive parameters. Differentiating \eqref{eq:bV1} with respect to $x$, we have
	\begin{align}
		\dot{\widetilde{V}}_1(x)=&\int_{0}^{1}(p_1(x,y)p_{1,x}(x,y)+p_{1,y}(x,y)p_{1,yx}(x,y))dy \notag \\
		&+\acute{\theta}_1k_1^2(x,x)+2\acute{\theta}_1\int_{0}^{x}k_1(x,y)k_{1,x}(x,y)dy \notag \\
		&+2\acute{\theta}_2\int_{0}^{x}l_1(x,y)l_{1,x}(x,y)dy.
\end{align}

Recalling \eqref{eq:p}--\eqref{eq:k} and using Young's inequality, we obtain
	\begin{align}
		\dot{\widetilde{V}}_1(x)	\leq&-\int_{0}^{1}\varepsilon_2\kappa_0p_{1,y}^2(x,y)dy-\int_{0}^{1}\varepsilon_2\kappa_0p_{1,yy}^2(x,y)dy \notag \\
		&+\acute{\theta}_1k_1^2(x,0)-\acute{\theta}_2\frac{\varepsilon_2}{\varepsilon_1} l_1^2(x,0)+2\int_{0}^{x}p_1^2(x,y)dy \notag \\
		&+(\bar{f}^2+\frac{\bar{f}^2}{\acute{\lambda}_0})\int_{0}^{x}(l_0^2(x,y)+k_0^2(x,y))dy\notag \\
		&+2\acute{\lambda}_0\int_{0}^{x}p_{1,yy}^2(x,y)dy+3\acute{\theta}_1\bar{\mu}\int_{0}^{x}k_1^2(x,y)dy \notag \\
		&+3\acute{\theta}_2\bar{\mu}\int_{0}^{x}l_1^2(x,y)dy+\bar{\mu}(\acute{\theta}_1+\acute{\theta}_2)\int_{0}^{x}k_0^2(x,y)dy \notag \\
		&+2\bar{\mu}(\acute{\theta}_1+\acute{\theta}_2)\int_{0}^{x}l_0^2(x,y)dy \notag \\
		\leq& 2\int_{0}^{x}p_1^2(x,y)dy-\varepsilon_2\kappa_0\int_{0}^{1}p_{1,y}^2(x,y)dy-(\varepsilon_2\kappa_0\notag \\
		&-2\acute{\lambda}_0)\int_{0}^{1}p_{1,yy}^2(x,y)dy+3\acute{\theta}_1\bar{\mu}\int_{0}^{x}k_1^2(x,y)dy\notag \\
		&+3\acute{\theta}_2\bar{\mu}\int_{0}^{x}l_1^2(x,y)dy-\acute{\theta}_2\frac{\varepsilon_2}{\varepsilon_1}l_1^2(x,0)+\acute{\theta}_1\acute{\kappa}_1\bar{p}_{0,y} \notag \\
		&+(\bar{f}^2+\frac{\bar{f}^2}{\acute{\lambda}_0}+2(\acute{\theta}_1+\acute{\theta}_2)\bar{\mu}+\acute{\kappa}_1\acute{\theta}_1)\bar{l}_0+(\bar{f}^2+\frac{\bar{f}^2}{\acute{\lambda}_0} \notag \\
		&+(\acute{\theta}_1+\acute{\theta}_2)\bar{\mu}+\acute{\kappa}_1\acute{\theta}_1)\bar{k}_0+\acute{\kappa}_1\acute{\theta}_1l_0^2(x,0) \label{dbV1}
	\end{align}
	where $\acute{\lambda}_0$ is positive, $\bar{k}_0=\|k_0(x,\cdot)\|^2$, $\bar{l}_0=\|l_0(x,\cdot)\|^2$, $\bar{p}_0=\|p_{0,y}(x,\cdot)\|^2$,
	and $k_1^2(x,0)\leq \acute{\kappa}_1(l_0^2(x,0)+\bar{k}_0+\bar{l}_0+\bar{p}_{0,y})$ for some positive $\acute{\kappa}_1$ according to \eqref{eq:p0y}--\eqref{eq:k}. We then choose $\acute{\lambda}<\frac{\varepsilon_2\kappa_0}{2}$ so that we get $\dot{\widetilde{V}}_1(x)\leq \bar{r}_1\widetilde{V}_1(x)+\acute{m}_1$ where $\bar{r}_1=\max\{ 2, \varepsilon_2\kappa_0, 3\acute{\theta}_1\bar{\mu}, 3\acute{\theta}_2\bar{\mu} \}$ and $\acute{m}_1=\acute{\kappa}_1\acute{\theta}_1\bar{p}_{0,y}+(\bar{f}^2+\frac{\bar{f}^2}{\acute{\lambda}_0}+2(\acute{\theta}_1+\acute{\theta}_2)\bar{\mu}+\acute{\kappa}_1\acute{\theta}_1)\bar{l}_0+(\bar{f}^2+\frac{\bar{f}^2}{\acute{\lambda}_0}+(\acute{\theta}_1+\acute{\theta}_2)\bar{\mu}+\acute{\kappa}_1\acute{\theta}_1)\bar{k}_0+\acute{\kappa}_1\acute{\theta}_1l_0^2(x,0)$.
	Next, we consider the following sequence of functionals
	\begin{align}
		\widetilde{V}_m(x)=&\widetilde{V}_{m-1}(x)+\frac{1}{2}\int_{0}^{1}(p_{m-1}^2(x,y)+p_{m-1,y}^2(x,y))dy \notag \\
		&+\int_{0}^{x}(\acute{\theta}_1k_{m-1}^2(x,y)+\acute{\theta}_2l_{m-1}^2(x,y))dy \label{bVm}
	\end{align}
	where $m=2,3,\dots,\infty$. For $m=2$, recalling \eqref{dbV1} and using Poincar$\acute{e}$ inequality, we have
	\begin{align}
		&\dot{\widetilde{V}}_2(x)\leq  2\int_{0}^{x}p_2^2(x,y)dy-\varepsilon_2\kappa_0\int_{0}^{1}p_{2,y}^2(x,y)dy \notag \\
		&+3\acute{\theta}_1\bar{\mu}\int_{0}^{x}k_2^2(x,y)dy+3\acute{\theta}_2\bar{\mu}\int_{0}^{x}l_2^2(x,y)dy+\acute{\theta}_1\acute{\kappa}_1l_0^2(x,0) \notag \\
		&+2\int_{0}^{x}p_1^2(x,y)dy-(\varepsilon_2\kappa_0-2\acute{\theta}_1\acute{\kappa}_1)\int_{0}^{1}p_{1,y}^2(x,y)dy \notag \\
		&+(\bar{f}^2+\frac{\bar{f}^2}{\acute{\lambda}_0}+2(\acute{\theta}_1+\acute{\theta}_2)\bar{\mu}+\acute{\theta}_1\acute{\kappa}_1+3\acute{\theta}_2\bar{\mu})\int_{0}^{x}l_1^2(x,y)dy \notag \\
		&+(\bar{f}^2+\frac{\bar{f}^2}{\acute{\lambda}_0}+(\acute{\theta}_1+\acute{\theta}_2)\bar{\mu}+\acute{\theta}_1\acute{\kappa}_1+3\acute{\theta}_1\bar{\mu})\int_{0}^{x}k_1^2(x,y)dy \notag \\
		&+(\bar{f}^2+\frac{\bar{f}^2}{\acute{\lambda}_0}+2(\acute{\theta}_1+\acute{\theta}_2)\bar{\mu}+\acute{\theta}_1\acute{\kappa}_1)\bar{l}_0+\acute{\theta}_1\acute{\kappa}_1\bar{p}_{0,y}+(\bar{f}^2 \notag \\
		&+\frac{\bar{f}^2}{\acute{\lambda}_0}+(\acute{\theta}_1+\acute{\theta}_2)\bar{\mu}+\acute{\theta}_1\acute{\kappa}_1)\bar{k}_0+(\acute{\theta}_1\acute{\kappa}_1-\acute{\theta}_2\frac{\varepsilon_2}{\varepsilon_1})l_1^2(x,0) \notag \\
		&-(\varepsilon_2\kappa_0-2\acute{\lambda}_0-2\acute{\theta}_1\acute{\kappa}_1)\int_{0}^{1}p_{1,yy}^2(x,y)dy . \label{dbV2} 
\end{align}

 We set $\acute{\theta}_1\leq \frac{\varepsilon_2\kappa_0-2\acute{\lambda}_0}{2\acute{\kappa}_1}$ and $\acute{\theta}_2 \geq \frac{\varepsilon_1}{\varepsilon_2}\acute{\kappa}_1$. Then we obtain $\dot{\widetilde{V}}_2(x)\leq \bar{r}_2 \widetilde{V}_2(x)+\acute{m}_1$ where $\bar{r}_2=\max\{ \bar{f}^2+\frac{\bar{f}^2}{\acute{\lambda}_0}+(\acute{\theta}_1+\acute{\theta}_2)\bar{\mu}+\acute{\theta}_1\acute{\kappa}_1+3\acute{\theta}_1\bar{\mu},\bar{f}^2+\frac{\bar{f}^2}{\acute{\lambda}_0}+2(\acute{\theta}_1+\acute{\theta}_2)\bar{\mu}+\acute{\theta}_1\acute{\kappa}_1+3\acute{\theta}_2\bar{\mu},2,\varepsilon_2\kappa_0 \}$.

 Following the same steps mentioned above and iterating for all values of m, we get $\dot{\widetilde{V}}_m(x)\leq \bar{r}_2 \widetilde{V}_m(x)+\acute{m}_1$ which means $\widetilde{V}_m(x)\leq \widetilde{V}_m(0)e^{\bar{r}_2x}+\frac{\acute{m}_1}{\bar{r}_2}(e^{\bar{r}_2x}-1)$. Recalling \eqref{eq:p0b}, \eqref{eq:pmb}, \eqref{bVm}, it is clear that $\lim\limits_{m\to\infty}\widetilde{V}_m(x) \leq \frac{\acute{m}_1}{\bar{r}_2}(e^{\bar{r}_2}-1)$ for all $x\in[0,1]$ since $V_m(0)=0$. Therefore, the series of \eqref{eq:lkp} is convergent and define valid functions $l\in L^2(\mathcal{D})$, $k\in L^2(\mathcal{D})$ and $p\in L^2([0,1],H^1([0,1]))$.

 Once solutions of $p(x,y)$, $l(x,y)$, $k(x,y)$ are obtained, according to \eqref{kerUp}, \eqref{kerPUpb} and \eqref{eq:gamma}, the explicit solution of $\Upsilon(x)$ is then obtained as
	\begin{align}
		\Upsilon(x)=\int_{0}^{x}(-\varepsilon_2\kappa_0p_y(s,0)q+G_2(s))e^{\varepsilon_2A_d(x-s)}ds
	\end{align}
	and we directly get $\Upsilon\in H^1(0,1)$, $\gamma \in H^1(0,1)$,

	Define $k_y=\acute k$ and $l_y=\acute l$. Then we consider the sequences of functions
	\begin{align}
		\acute l(x,y)&=\sum_{m=0}^{\infty}\acute l_{m}(x,y),\quad \acute k(x,y)=\sum_{m=0}^{\infty}\acute k_{m}(x,y). \label{eq:lky}
	\end{align}

Differentiating \eqref{kerK}, \eqref{kerL} with respect to $y$, differentiating \eqref{kerKb}, \eqref{kerLGmab}  with respect to $x$, we have
	\begin{align}
		\acute k_{x}(x,y)&=-\acute k_{x}(x,y)+\frac{\varepsilon_2}{\varepsilon_1}c_1(y)\acute l(x,y)-f_{22}(y,y)k(x,y) \notag \\
		&-\int_{y}^{x}(f_{22y}(z,y)k(x,z)+\frac{\varepsilon_2}{\varepsilon_1}f_{12y}(z,y)l(x,z))dz \notag \\
		&-\frac{\varepsilon_2}{\varepsilon_1}f_{12}(y,y)l(x,y)+\frac{\varepsilon_2}{\varepsilon_1}c_1^{(1)}(y)l(x,y)-f_{22y}(x,y), \label{eq:ky} \\
		k_x(x,0)&=\frac{\varepsilon_2}{\varepsilon_1}l_x(x,0)-\varepsilon_2\gamma_x(x)B+\frac{\varepsilon_2}{\varepsilon_1}l(x,x)g_1(x)-g_2^{(1)}(x) \notag \\
		&+\int_{0}^{x}(\frac{\varepsilon_2}{\varepsilon_1}l_x(x,y)g_1(y)+k_x(x,y)g_2(y))dy\notag \\
		&+k(x,x)g_2(x), \label{eq:kyb} \\
		\acute l_{x}(x,y)&=\frac{\varepsilon_2}{\varepsilon_1}\acute l_{y}(x,y)+c_2(y)\acute k(x,y)-f_{21}(y,y)k(x,y)\notag \\
		&-\int_{y}^{x}(f_{21y}(z,y)k(x,z)+\frac{\varepsilon_2}{\varepsilon_1}f_{11y}(z,y)l(x,z))dz \notag \\
		&-\frac{\varepsilon_2}{\varepsilon_1}f_{11}(y,y)l(x,y) +c_2^{(1)}(y)k(x,y)-f_{21y}(x,y),  \label{eq:ly}\\
		l_x(x,x)&=-\acute l(x,x)-+\frac{\varepsilon_1}{\varepsilon_1+\varepsilon_2}c_2^{(1)}(x) \label{eq:lyb}
	\end{align}
	and define functionals as
	\begin{align}
		\acute{V}_1(x)=&\int_{0}^{x}\acute k_1^2(x,y)dy+2\acute{\theta}_2\int_{0}^{x}\acute l_1^2(x,y)dy, \notag \\
		\acute{V}_m(x)=&\int_{0}^{x}\acute k_{m-1}^2(x,y)dy+2\acute{\theta}_2\int_{0}^{x}\acute l_{m-1}^2(x,y)dy \notag \\
		&+\acute{V}_{m-1}(x)
	\end{align}where $m=2,3,\dots,\infty$. 

For $m=2$, following the according proof for $l(x,y)$ and $k(x,y)$, we obtain
	\begin{align}
		&\dot{\acute{V}}_2(x)\leq \int_{0}^{x}\acute k_{2}^2(x,y)dy+\acute{\theta}_4\int_{0}^{x}\acute l_{2}^2(x,y)dy+\acute{\kappa_2}\acute l_{0}^2(x,0) \notag \\
		&+(\bar{f}^2+\acute{\kappa}_2+\acute{\theta}_4)\int_{0}^{x}\acute l_{1}^2(x,y)dy+(\acute{\kappa}_2-\frac{\varepsilon_2}{\varepsilon_1}\acute{\theta}_4)\acute l_{1}^2(x,0)+ \notag \\
		&(\bar{f}^2+\acute{\kappa}_2+\acute{\theta}_4)\int_{0}^{x}\acute k_{1}^2(x,y)dy+(\bar{f}^2+\acute{\kappa}_2)\int_{0}^{x}\acute l_{0}^2(x,y)dy \notag \\
		&+(\bar{f}^2+\acute{\kappa}_2)\int_{0}^{x}\acute k_{0}^2(x,y)dy
	\end{align}
	where $\acute k_{1}^2(x,0)\leq \acute{\kappa}_2(\acute l_{0}^2(x,0)+\|\acute k_{0}(x,\cdot)\|^2+\|\acute l_{0}(x,\cdot)\|^2)$ for some positive $\acute{\kappa}_2$. Notice, from the explicit solutions of $\acute l_0$, $k$ and $l$, we get $\acute l_{0}(x,0)\in L^{\infty}(0,1)$. Choosing $\acute{\theta}_4\geq \frac{\varepsilon_1}{\varepsilon_2}\acute{\kappa_2}$, we have $\dot{\acute{V}}_2(x)\leq \bar{r}_3\acute{V}_2(x)+\acute{m}_3$ where $\bar{r}_3=\max\{ 1,\bar{f}^2+\acute{\kappa}_2+\acute{\theta}_4 \}$ and $\acute{m}_3=\acute{\kappa}_2\acute l_{0}^2(x,0)+(\bar{f}^2+\acute{\kappa}_2)(\|\acute l_{0}(x,\cdot)\|^2+\|\acute k_{0}(x,\cdot)\|^2)$. Similarly, it is straightforward to get $\lim\limits_{m\to\infty}\acute{V}_m(x) \leq \frac{\acute{m}_3}{\bar{r}_3}(e^{\bar{r}_3}-1)$. Therefore, we obtain $l\in H^1(\mathcal{D})$, $k\in H^1(\mathcal{D})$

The proof of Lemma \ref{kernel klp} is complete.

\subsection{Proof of Theorem $\ref{tosb}$}\label{The of controller}
\setcounter{equation}{0}
\renewcommand{\theequation}{E.\arabic{equation}}
We start by investigating the stability of the target system. The equivalent stability property between the target system and the original system is guaranteed by the invertibility of the transformations  \eqref{cxi2bt}, \eqref{cbt2xi}.

\textit{1):}From $(z(0),\xi(x,0),\eta(x,0), u(x,0), $ $X(0))\in \mathcal H$, we obtain $(z(0),\beta(x,0),\eta(x,0), u(x,0), $ $X(0))\in \mathcal H$ recalling the transformation \eqref{cxi2bt}, Assumption \ref{d(t)} and Lemma \ref{kernel klp}.

Timing $e^{\acute{c}_1t}$ into both side of \eqref{eq:beta1} and integrating from $0$ to $t$, we have the solution $\beta(1,t)=e^{-\acute{c}_1t}\beta(1,0)$ on $t\in [0,\infty)$, and thus $\beta(1,t)\in H^1([0,\infty))$. Recalling initial data $\beta(x,0)\in H^1(0,1)$ of the $\beta$-PDE \eqref{gtsbt}, we then obtain $\beta(x,t)\in L^{\infty}([0,\infty);H^1(0,1))$ and $\beta(0,t)\in H^1([0,\infty))$. Then, we time $e^{-(A+BK)t}$ into both sides of \eqref{gstx} and integrating it from $0$ to $t$, we get $X(t)=X(0)e^{(A+BK)t}+\int_{0}^{t}e^{(A+BK)(t-s)}B\beta(0,s)ds$. Because of $\beta(0,t)\in H^1([0,\infty))$, we have  $X(t)\in H^1([0,\infty);\mathbb R^n)$. Similarly, we have $d(t)\in H^1([0,\infty);\mathbb R^m)$ according to \eqref{gxd}.

Next, we introduce a transformation
	\begin{align}
		\acute{u}(x,t)=u(x,t)-(1-x)(qd(t)+\acute{p}_2X(t))\label{eq:tu1} 
	\end{align}
	to convert \eqref{gxu}, \eqref{gxub} as
	\begin{align}
		\acute{u}_t(x,t)=&\kappa_0\acute{u}_{xx}(x,t)+(x-1)\acute{p}_2(A+BK)X(t)\notag \\
		&+(x-1)qA_dd(t)+(x-1)\acute{p}_2B\beta(0,t), \label{eq:bu} \\
		\acute{u}(0,t)=&0,\quad \acute{u}(1,t)=0. \label{eq:bub}
\end{align}We thus obtain $\acute{u}(x,t)\in  H^{1}([0,1]\times[0,\infty))$ and $u(x,t)\in L^{\infty}([0,\infty);H^1(0,1))$ recalling \eqref{eq:tu1}. Then, in the same way as for $\beta(x,t)$, it is clear that we get $\eta(x,t)\in L^{\infty}([0,\infty);H^1(0,1))$ from \eqref{gtset}, \eqref{gtsalb}. Recalling the transformation \eqref{cbt2xi}, we have $(\xi(x,t),\eta(x,t), u(x,t), X) \in L^{\infty}([0,\infty);H^1(0,1))^3 \times H^1([0,\infty);\mathbb R^n)$. Moreover, we get $z(t)\in L^{\infty}([0,\infty);\mathbb R)$ according to \eqref{gsxib}. The proof of Property 1 is complete.

\textit{2):}
We define a Lyapunov function $V_1(t)$ as
\begin{align}
	V_1(t)=&X(t)^TP_1X(t)+\frac{\varepsilon_2}{2}\int_{0}^{1}he^{x}\beta^2(x,t)dx+\frac{1}{2}\beta(1,t)^2 \label{lyfV1}
\end{align} 
where the positive parameter $h$ is to be chosen later. Recalling $A+BK$ is Hurwitz, there exists a matrix $P_1=P_1^T>0$ being the solution to the Lyapunov equation
$P_1(A+BK)+(A+BK)^TP_1=-Q_1$
for some $Q_1=Q_1^T>0$.
Defining 
\begin{align}
	\Omega_1(t)&=| \beta(1,t)|^2+| X(t) |^2+\| \beta(\cdot,t) \|^2, \label{Ome1} 
\end{align}
we also have
$
\theta_{1,1}\Omega_1(t)\leq V_1(t)\leq \theta_{1,2}\Omega_1(t) 
$
where
\begin{align}
	\theta_{1,1}&=\min\{ \frac{1}{2},\lambda_{\min}(P_1),\frac{\varepsilon_2}{2}h \}, \\
	\theta_{1,2}&=\max\{ \frac{1}{2},\lambda_{\max}(P_1),\frac{\varepsilon_2}{2}he \}.
\end{align}
Taking the time derivative of $V_1(t)$ along \eqref{gtsbt}--\eqref{gstx}, \eqref{eq:beta1}, 
applying integral by parts, using Cauchy-Schwartz inequality, we get
\begin{align}
	\dot{V}_1(t)&\leq-\frac{\lambda_{\min}(Q_1)}{2}\lvert X(t) \lvert^2+\frac{2\lvert P_1B \lvert^2}{\lambda_{\min}(Q_1)}\beta(0,t)^2\notag \\
	&+\frac{he}{2}\beta(1,t)^2-\frac{h}{2}\beta(0,t)^2-\frac{h}{2}\int_{0}^{1}e^{x}\beta^2(x,t)dx \notag \\
	&-\acute{c}_1\beta(1,t)^2. \label{dV1}
\end{align}
Choosing $h,\acute{c}_1$ to satisfy
\begin{align}
		h&> \frac{4| P_1B|^2}{\lambda_{\min}(Q_1)}+\lambda_{\min}(Q_1),\quad
		\acute{c}_1>\frac{he}{2}++\frac{\lambda_{\min}(Q_1)}{2},\label{eq:conditionbarc1}
	\end{align}
	we obtain
	\begin{align}
		\dot{V_1}(t)&\leq  -\frac{\lambda_{\min}(Q_1)}{2\theta_{1,2}}\theta_{1,2}\Omega_1(t)-\acute{\eta}_0\beta(0,t)^2 \notag \\
		&\leq -\frac{\lambda_{\min}(Q_1)}{2\theta_{1,2}} V_1(t)
	\end{align}
	and $\acute \eta_0=\frac{h}{2}-\frac{2| P_1B|^2}{\lambda_{\min}(Q_1)}>0$.	
	From the comparison principle, we can conclude $V_1(t)< V_1(0)e^{-\frac{\lambda_{\min}(Q_1)}{2\theta_{1,2}} t}$.
	It implies that $\Omega_1(t)< \frac{\theta_{1,2}}{\theta_{1,1}}\Omega_1(0)e^{-\frac{\lambda_{\min}(Q_1)}{2\theta_{1,2}} t}$ with $\lambda_{\min}(Q_1)$ influenced by $K$ since $P_1(A+BK)+(A+BK)^TP_1=-Q_1$.

Next, the well-posedness of system \eqref{gsz}--\eqref{gxd} is obtained, and Lyapunov analysis for the overall system is provided. We choose a Lyapunov function $V_2(t)$ as
\begin{align}
	V_2(t)&=\frac{\varepsilon_1}{2}\int_{0}^{1}e^{-a_0x}\eta^2(x,t)dx+\frac{a_1}{2\kappa_0}\int_{0}^{1}u^2(x,t)dx \notag \\
	&+\frac{a_2}{2\kappa_0}\int_{0}^{1}u_x^2(x,t)dx+\mathcal{R}_1V_1(t) \label{V2(t)}
\end{align}
where $a_0,a_1,a_2,\mathcal{R}_1$ are positive. We have $\theta_{2,1}\Omega_2(t)\leq V_2(t)\leq \theta_{2,2}\Omega_2(t)$
where
\begin{align}
	\Omega_2(t)&=| \beta(1,t)|^2+| X(t) |^2+\| \beta(\cdot,t) \|^2+\| \eta(\cdot,t) \|^2 \notag \\
	&+\| u(\cdot,t) \|^2+\| u_x(\cdot,t) \|^2, \label{Ome2} \\
	\theta_{2,1}&=\min\{ \mathcal{R}_1\theta_{1,1},\frac{\varepsilon_1}{2}e^{-a_0},\frac{a_1}{2\kappa_0},\frac{a_2}{2\kappa_0} \}, \label{eq:theta21} \\ 
	\theta_{2,2}&=\max\{ \mathcal{R}_1\theta_{1,2},\frac{\varepsilon_1}{2},\frac{a_1}{2\kappa_0},\frac{a_2}{2\kappa_0} \}.
\end{align}
Taking the derivative of \eqref{V2(t)}, recalling \eqref{dV1}, applying Young's inequality and Cauchy-Schwarz inequality, we have
\begin{align}
	&\dot{V}_2(t)\leq -(\mathcal{R}_1\acute{\eta}_1-| a_1\acute{p}_2+\frac{a_2\acute{p}_2}{\kappa_0}(A+BK) |^2\frac{\ell_2}{4} -\lvert C+K\lvert^2 \notag \\
	&-\frac{L_2}{2})\lvert X(t)\lvert^2-(\mathcal{R}_1\acute{\eta}_0-|\frac{a_2\acute{p}_2}{\kappa_0}B|^2\frac{\ell_3}{2}-\frac{L_2}{2}-1 ) \beta(0,t)^2 \notag \\
	&-\mathcal{R}_1(\acute{c}_1-\frac{he}{2}) \beta(1,t)^2-(\mathcal{R}_1\acute{\eta}_4-\frac{L_2}{2a_0}-\frac{L_2}{2})\int_{0}^{1}\beta^2(y,t)dy \notag \\
	&-(\frac{a_0}{2}-\frac{L_2}{2a_0}-4L_2)\int_{0}^{1}e^{-a_0x}\eta^2(x,t)dx+\frac{L_2}{2}\int_{0}^{1}u^2(x,t)dx \notag \\
	&-(a_1-2(\frac{1}{\ell_1}+\frac{1}{\ell_2}+\frac{1}{\ell_3}))\int_{0}^{1}u_x^2(x,t)dx-(a_2-2(\frac{1}{\ell_1}+\frac{1}{\ell_2} \notag \\
	&+\frac{1}{\ell_3})))\int_{0}^{1}u_{xx}^2(x,t)dx+(|a_1q+\frac{a_2}{\kappa_0}qA_d|^2\frac{\ell_1}{4}+\frac{L_2}{2})|d(t)|^2 . 
\end{align}
According to Poincar$\acute{e}$ inequality and Property 1, we get
\begin{align}
	&\dot{V}_2(t)\leq -(\mathcal{R}_1\acute{\eta}_1-| a_1\acute{p}_2+\frac{a_2\acute{p}_2}{\kappa_0}(A+BK) |^2\frac{\ell_2}{4} -| C+K|^2 \notag \\
	&-\frac{L_2}{2})\lvert X(t)\lvert^2-(\mathcal{R}_1\acute{\eta}_0-|\frac{a_2\acute{p}_2}{\kappa_0}B|^2\frac{\ell_3}{2}-\frac{L_2}{2}-1 ) \beta(0,t)^2 \notag \\
	&-\mathcal{R}_1(\acute{c}_1-\frac{he}{2}) \beta(1,t)^2-(\mathcal{R}_1\acute{\eta}_4-\frac{L_2}{2a_0}-\frac{L_2}{2})\int_{0}^{1}\beta^2(y,t)dy \notag \\
	&-(\frac{a_0}{2}-\frac{L_2}{2a_0}-4L_2)\int_{0}^{1}e^{-a_0x}\eta^2(x,t)dx+\frac{L_2}{2}\int_{0}^{1}u^2(x,t)dx \notag \\
	&-\frac{1}{5}(a_1-2(\frac{1}{\ell_1}+\frac{1}{\ell_2}+\frac{1}{\ell_3})+\frac{5}{2}L_2)\int_{0}^{1}u^2(x,t)dx-\frac{1}{5}(a_1 \notag \\
	&-2(\frac{1}{\ell_1}+\frac{1}{\ell_2}+\frac{1}{\ell_3}))\int_{0}^{1}u_x^2(x,t)dx -(a_2-2(\frac{1}{\ell_1}+\frac{1}{\ell_2} \notag \\
	&+\frac{1}{\ell_3}))\int_{0}^{1}u_{xx}^2(x,t)dx+(|a_1q+\frac{a_2}{\kappa_0}qA_d|^2\frac{\ell_1}{4}+\frac{L_2}{2})|d(t)|^2 
	\label{dV2}		\end{align}
where $\acute{\eta}_1=\frac{\lambda_{\min}(Q_1)}{2}$, $\acute{\eta}_4=\frac{h}{2}$ and
$L_2=\max\{ $ $\| \mathcal{F}_{12}(x,y)\|_\infty,$ $\| \mathcal{F}_{13}(x,y)\|_\infty,\lvert \mathcal{D}_1(x)\lvert_\infty,
| \mathcal{D}_2(x)|_\infty,| c_1(x)|_\infty,| c_2(x)|_\infty,$ $| \mu_1(x)|_\infty,\lvert g_1(x)\lvert_\infty \}$. Please note that $-(a_1-2(\frac{1}{\ell_1}+\frac{1}{\ell_2}+{\frac{1}{\ell_3}}))\int_{0}^{1}u_x^2(x,t)dx\leq-\frac{1}{5}(a_1-2(\frac{1}{\ell_1}+\frac{1}{\ell_2}+{\frac{1}{\ell_3}}))\int_{0}^{1}(u^2(x,t)+u_x^2(x,t))dx $ according to  Poincar$\acute{e}$ inequality, \eqref{gxub}.
Choosing $a_0>4L_2$ $+\sqrt{16L_2^2+L_2}$ to make  $-\frac{a_0}{2}+\frac{L_2}{2a_0}+4L_2<0$, and 
selecting positive parameter $R_1,a_1,a_2$ to satisfy
\begin{align}
		\mathcal{R}_1>&\max \biggl\{ \frac{| a_1\acute{p}_2+\frac{a_2\acute{p}_2}{\kappa_0}(A+BK) |^2\frac{\ell_2}{4} -| C+K|^2-\frac{L_2}{2} }{\acute{\eta}_1}, \notag \\
		&\frac{|\frac{a_2}{\kappa_0}B|^2\frac{\ell_3}{2}+\frac{L_2}{2}+1}{\acute{\eta}_0}, \frac{\frac{L_2}{2a}+\frac{L_2}{2}}{\acute{\eta}_4} \biggr\}, \\
		a_1>&2(\frac{1}{\ell_1}+\frac{1}{\ell_2}+\frac{1}{\ell_3})+\frac{5}{2}L_2, \quad a_2>2(\frac{1}{\ell_1}+\frac{1}{\ell_2}+\frac{1}{\ell_3}),
\end{align}
we get $
\dot{V_2}(t)\leq -\lambda_2 \theta_{2,2}\Omega_2(t)+\acute{\eta}_5 $
for some positive $\lambda_2$ and $\acute{\eta}_5=(|a_1q+\frac{a_2}{\kappa_0}qA_d|^2\frac{\ell_1}{4}+\frac{L_2}{2})|d(t)|^2$. Thus, we get
\begin{align}
	V_2(t)\leq V_2(0)e^{-\lambda_2t}+\frac{\acute{\eta}_5}{\lambda_2}. \label{V2}
\end{align}

Next, we conduct the analysis on the higher-order terms $\|\eta(x,t)\|_{H^1}$, $\|\xi(x,t)\|_{H^1}$, which will also be required in analyzing the boundedness of the control input. 

Differentiating \eqref{gtsbt}, \eqref{gtset} with respect to $x$, differentiating \eqref{gtsalb}, \eqref{eq:beta1} with respect to $t$, we have
\begin{align}
	&\varepsilon_2\beta_{xt}(x,t)=\beta_{xx}(x,t), \label{eq:btx} \\
	&\varepsilon_1\eta_{xt}(x,t)=-\eta_{xx}(x,t)+(c_1^{(1)}(x)+\mathcal{F}_{12}(x,x))\beta(x,t) \notag \\
	&+c_1(x)\beta_x(x,t)+g_1^{(1)}(x)\beta(0,t)+\mathcal{D}_1^{(1)}(x)X(t) \notag \\
	&+\mu_1^{(1)}(x)u(x,t)+\mu_1(x)u_x(x,t)+\mathcal{D}_2^{(1)}(x)d(t) \notag \\
	&+\mathcal{F}_{11}(x,x)\eta(x,t)+\int_{0}^{x}\mathcal{F}_{11x}(x,y)\eta(y,t)dy\notag \\
	&+\int_{0}^{x}\mathcal{F}_{12x}(x,y)\beta(y,t)dy+\int_{0}^{1}\mathcal{F}_{13x}(x,y)u(y,t)dy, \label{eq:etx} \\
	&\beta_{tt}(1,t)=-\acute{c}_1\beta_t(1,t), \label{eq:btxb} \\
	&\eta_t(0,t)=\beta_t(0,t)+(c+K)\dot{X}(t). \label{eq:etxb}
\end{align}
Please note that 
	$(\beta(x,t),\eta(x,t))\in  L^{\infty}([0,\infty);H^1(0,1))^2$ obtained in the proof of Property 1 leads to $\beta_x,\eta_x\in L^{\infty}([0,\infty);L^2(0,1))^2$, which justifies $x$-differentiation of \eqref{gtsbt}, \eqref{gtset} in the weak sense.
	Besides, the fact that $\beta(1,t)\in H^1([0,\infty))$, $\beta(0,t)\in H^1([0,\infty))$, and $X\in H^1([0,\infty))$, which is obtained in the proof of Property 1,
	justifies $t$-differentiation of \eqref{gtsalb}, \eqref{eq:beta1} in the weak sense. According to \eqref{eq:btx}--\eqref{eq:etxb},
	the $\beta_{x}(x,t)$ and $\eta_x(x,t)$ can be expressed as
	\begin{align}
		&\beta_{x}(x,t)=-\acute{c}_1\varepsilon_2\beta(1,\varepsilon_2(x-1)+t) \label{eq:betax} \\
		&\eta_x(x,t)=\acute{c}_1\varepsilon_1\beta(1,-\varepsilon_2(x+1)+t) \notag \\
		&-(\varepsilon_1(C+K)(A+BK)-\mathcal D_1(0))X(-\varepsilon_1x+t)\notag \\
		&-(\varepsilon_1(C+K)B-c_1(0)-g_1(0))\beta(0,-\varepsilon_1x+t) \notag \\
		&+\int_{0}^{x}g^{(1)}(x)\beta(0,\varepsilon_1(s-x)+t)ds \notag \\
		&-\int_{0}^{x}\acute{c}_1\varepsilon_2c_1(s)\beta(1,(\varepsilon_1+\varepsilon_2)-\varepsilon_2-\varepsilon_1x+t)ds \notag \\
		&+\int_{0}^{x}(c^{(1)}_1(s)+\mathcal F_{12}(s,s))\beta(s,\varepsilon_1(s-x)+t)ds \notag \\
		&+\int_{0}^{x}(\mu^{(1)}_1(s)+\mathcal F_{13}(s,s))u(s,\varepsilon_1(s-x)+t)ds \notag \\
		&+\int_{0}^{x}\mathcal D_1^{(1)}(x)X(\varepsilon_1(s-x)+t)ds \notag \\
		&+\int_{0}^{x}\mathcal D_2^{(1)}(s)d(\varepsilon_1(s-x)+t)ds+\mathcal D_2(0)d(-\varepsilon_1x+t) \notag \\
		&+\int_{0}^{x}\int_{0}^{s}\mathcal F_{11x}(s,y)\eta(y,\varepsilon_1(s-x)+t)dyds \notag \\
		&+\int_{0}^{x}\int_{0}^{s}\mathcal F_{12x}(s,y)\beta(y,\varepsilon_1(s-x)+t)dyds \notag \\
		&+\int_{0}^{x}\int_{0}^{s}\mathcal F_{13x}(s,y)u(y,\varepsilon_1(s-x)+t)dyds \notag \\
		&+\int_{0}^{x}\mathcal F_{11}(s,s)\eta(s,\varepsilon_1(s-x)+t)ds \notag \\
		&+\int_{0}^{1}\mathcal F_{13}(0,y)u(y,-\varepsilon_1x+t)dy \notag \\
		&+\int_{0}^{x}\mu_1(x)u_x(s,\varepsilon_1(s-x)+t)ds. \label{eq:etax}
\end{align}
We define
	\begin{align}
		&V_4(t)=V_2(t)+\int_{0}^{1}\beta_x^2(x,t)dx+\frac{1}{16}\int_{0}^{1}\eta_x^2(x,t)dx, \label{V4(t)}
	\end{align}
	and consider
	\begin{align}
		\Omega_4(t)=&\lvert \beta(1,t)\lvert^2+\lvert X(t) \lvert^2+\| \beta(\cdot,t) \|^2_{H^1}+\| \eta(\cdot,t) \|^2_{H^1} \notag \\
		&+\| u(\cdot,t) \|^2+\| u_x(\cdot,t)\|^2  \label{eq:Ome4},
	\end{align}
	we have $\theta_{4,1}\Omega_4(t)\leq V_4(t)\leq \theta_{4,2}\Omega_4(t)$ for some positive $\theta_{4,1}$, $\theta_{4,2}$.
	Recalling \eqref{V2}, using Cauchy Schwarz inequality and Young's inequality, we obtain
	\begin{align}
		&V_4(t)\leq \bigg( 2\acute{c}_1^2\varepsilon_1^2+(1+\acute{c}_1^2\varepsilon_2^2)(-\varepsilon_1(C+K)B+c_1(0)g_1(0))^2\notag \\
		&+\bigg( |\mathcal D_1(0)-\varepsilon_1(C+K)(A+BK)|^2+\int_{0}^{1}|\mathcal D_1^{(1)}(x)|^2dx \bigg) \notag \\
		&|X(t)|^2+\int_{0}^{1}((1+\acute{c}_1^2\varepsilon_2^2)g^{(1)}(x)^2+\acute{c}_1^2\varepsilon_2^2c_1^2(x))dx \bigg)\beta_1^2(1,t) \notag \\
		&+ \int_{0}^{1}\bigg((c_1^{(1)}(x)+\mathcal F_{12}(x,x))^2+\int_{0}^{x}\mathcal F_{11x}^2(x,y)dy \bigg)dx\notag \\
		&\int_{0}^{1}\beta^2(x,t)dx+\int_{0}^{1}\mu_1^2(x)dx\int_{0}^{1}u_x^2(x,t)dx+V_2(t) \notag \\
		&+ \int_{0}^{1}(\mathcal F_{11}^2(x,x)+\int_{0}^{x}\mathcal F_{11x}^2(x,y)dy )dx\int_{0}^{1}\eta^2(x,t)dx \notag \\
		&+\int_{0}^{1}\bigg((\mu^{(1)}(x)+\mathcal F_{13}(x,x))^2+\int_{0}^{x}\mathcal F_{13x}^2(x,y)dy+ \notag \\
		&\mathcal F_{13}^2(0,y) \bigg)dx\int_{0}^{1}u^2(x,t)dx+(\mathcal{D}_2^2(0)+\int_{0}^{1}\mathcal{D}_2^2(x)dx)|d(t)|^2 \notag \\
		&\leq \frac{\kappa_1}{\theta_{2,1}}\theta_{2,1}\Omega_2(t)+(\mathcal{D}_2^2(0)+\int_{0}^{1}\mathcal{D}_2^2(x)dx)|d(t)|^2 \notag \\
		&+V_2(0)e^{-\lambda_2t}+\frac{\acute{\eta}_5}{\lambda_2} \label{eq:V4}
	\end{align}
	for some positive $\kappa_1$. We thus have $V_4(t)\leq (1+\frac{\kappa_1}{\theta_{2,1}})$ $V_4(0)e^{-\lambda_2t}+\mathcal{C}_1$ where $\mathcal C_1=(1+\frac{\kappa_1}{\theta_{2,1}})\frac{\acute{\eta}_5}{\lambda_2}+(\mathcal{D}_2^2(0)+\|\mathcal D(\cdot)\|^2)|d(t)|^2$.
	It means that $\Omega_4(t)\leq \frac{\theta_{4,2}}{\theta_{4,1}}(1+\frac{\kappa_1}{\theta_{2,1}})\Omega_4(0)e^{-\lambda_2 t}+\frac{\mathcal{C}_1}{ \theta_{4,1}}$.

Recalling \eqref{gsxib} and \eqref{cbt2xi}, applying Cauchy-Schwarz inequality and Young's inequality, we get
	\begin{align}
		|z(t)|^2\leq& \acute{\kappa}_3(\|\eta(\cdot,t)\|^2_{H^1}+\beta(1,t)^2+|X(t)|^2+\|\beta(\cdot,t))\|^2\notag \\
		&+\|u(\cdot,t)\|^2+|d(t)|^2)
	\end{align}
	for some positive $\acute{\kappa}_3$. Thus, recalling Assumption \ref{d(t)}, we get $z(t)$ is bounded.
	Defining 
	\begin{align}
		\Xi_1(t)=&\lvert z(t)\lvert^2+\lvert X(t)\lvert^2+\| \xi(\cdot,t)\|^2_{H^1}+\| \eta(\cdot,t)\|^2_{H^1} \notag \\
		&+\| u(\cdot,t)\|^2+\| u_x(\cdot,t)\|^2 \label{eq:Xi1},
	\end{align} 
	applying Cauchy-Schwarz inequality and recalling transformations \eqref{cbt2xi}, \eqref{cxi2bt}, we get
	$
	\theta_{4,3}(t)(\Xi_1(t)+|d(t)|^2)\leq \Omega_4(t)\leq \theta_{4,4}(\Xi_1(t)+|d(t)|^2)
	$
	for some positive $\theta_{4,3}$ and $\theta_{4,4}$. Therefore we have
	\begin{align}
		\Xi_1(t)\leq& \frac{\theta_{4,2}\theta_{4,4}}{\theta_{4,1}\theta_{4,3}}(1+\frac{\kappa_1}{\theta_{2,1}})\Xi_1(0)e^{-\lambda_2 t}+\frac{\mathcal{C}_1}{\theta_{4,1}\theta_{4,3}} \notag \\
		&+\frac{\theta_{4,2}\theta_{4,4}}{\theta_{4,1}\theta_{4,3}}(1+\frac{\kappa_1}{\theta_{2,1}}){\bar D_d}^2.
	\end{align}
	Thus, recalling \eqref{eq:V4}, \eqref{V2},  \eqref{Omeg0} is achieved with
	\begin{align}
		\Upsilon_0=&\frac{\theta_{4,2}\theta_{4,4}}{\theta_{4,1}\theta_{4,3}}(1+\frac{\kappa_1}{\theta_{2,1}}),\quad \acute{C}_0=\lambda_2, \notag \\
		\mathcal{C}_0=&\frac{1}{\theta_{4,1}\theta_{4,3}}\bigg( \frac{\theta_{2,1}+\kappa_1}{\theta_{2,1}\lambda_2}(|a_1q+\frac{a_2}{\kappa_0}qA_d|^2\frac{\ell_1}{4}+\frac{L_2}{2}) \notag \\
		&\mathcal D_2^2(0)+\int_{0}^{1}\mathcal D_2^2(x)dx \bigg)\bar{D}_d^2.  \label{mathC0}
\end{align}
Please note that according to \eqref{eq:theta21}, \eqref{eq:Ome4}, \eqref{eq:Xi1}, we know $\theta_{2,1}$, $\theta_{4,1}$, $\theta_{4,3}$ and $\lambda_2$ depend on the plant parameters in \eqref{gsz}--\eqref{gxd}. Similarly, according to \eqref{dV2}, \eqref{eq:V4}, we get that $\kappa_1$, $L_2$, $\theta_{4,2}$, $\theta_{4,4}$ depend on the plant parameters as well, and can be adjusted by the design parameters. Therefore, from \eqref{mF11}--\eqref{mD2} and \eqref{cxi2bt}, \eqref{cbt2xi}, we have that $\mathcal C_0$ is bounded by $\bar{D}_d$. That is, if $\bar{D}_d=0$ (i.e., external disturbances vanish), then $\mathcal C_0=0$, the exponential stability is obtained.

\textit{ 3):}
Applying Cauchy Schwarz inequality into \eqref{U(t)}, recalling Lemma \ref{kernel klp}, \eqref{eq:n1}--\eqref{eq:n10} and \eqref{eq:etax}, together with Properties 1, 2, Property 3 is obtained.

The proof of Theorem $\ref{tosb}$ is complete.
%\end{thm}

\subsection{Proof Theorem $\ref{ertg}$}\label{The 3}
\setcounter{equation}{0}
\renewcommand{\theequation}{F.\arabic{equation}}
\textit{1):} Following the proof of Property 1 of Theorem \ref{tosb}, the observer error system has a unique solution $(\widetilde{z},\widetilde{\xi},\widetilde{\eta}, \widetilde{u}, \widetilde{X}, \widetilde{d}) \in H^1([0,\infty);\mathbb R)\times H^{1}([0,1]\times [0,\infty))^3 \times H^1([0,\infty);\mathbb R^n)\times H^1([0,\infty);\mathbb R^m)$.

\textit{2):} Consider a Lyapunov function as
\begin{align}
		V_3(t)&=\acute{b}_1\widetilde{X}(t)^TP_2\widetilde{X}(t)+\frac{\varepsilon_1}{2}\int_{0}^{1}\acute{p}_0e^{-\acute{a}_0x}\widetilde{\alpha}^2(x,t)dx \notag \\
		&+\frac{\varepsilon_2}{2}\int_{0}^{1}\acute{h}e^{\acute{d}x}\widetilde{\beta}^2(x,t)dx+\frac{\acute{a}_1}{2\kappa_0}\int_{0}^{1}\widetilde{u}^2(x,t)dx \notag \\
		&+\frac{\acute{a}_2}{2\kappa_0}\int_{0}^{1}\widetilde{u}_x^2(x,t)dx+\frac{1}{2}\widetilde{Y}(t)^2+\acute{b}_2\widetilde{d}(t )^TP_3\widetilde{d}(t) \label{V3(t)}
\end{align}
where the positive parameters $\acute{p}_0, \acute{b}_1,\acute{b}_2,\acute{a}_0,\acute{h},\acute{d},\acute{a}_1$,$\acute{a}_2$ are to be chosen later, and where $P_2=P_2^T>0$ and $P_3=P_3^T>0$ are the solution to the Lyapunov equations $P_2(A-L_xC)+(A-L_xC)^TP_2=-Q_2$ and $P_3(A_d-L_dq)+(A_d-L_dq)^TP_3=-Q_3$, respectively, for some $Q_2=Q_2^T>0$, $Q_3=Q_3^T>0$, considering $A-L_xC$ and $A_d-L_dq$ are Hurwitz.
Defining
\begin{align}
	\Omega_3(t)=&\| \widetilde{\alpha}(\cdot,t)\|^2+\| \widetilde{\beta}(\cdot,t)\|^2+\| \widetilde{u}(\cdot,t)\|^2+\| \widetilde{u}_x(\cdot,t)\|^2  \notag \\
	&+| \widetilde{Y}(t)|^2+| \widetilde{X}(t)|^2+| \widetilde{d}(t)|^2, \label{Om3} 
\end{align}	
we have $\theta_{3,1}\Omega_3(t)\leq V_3(t) \leq\theta_{3,2}\Omega_3(t)$ where $\theta_{3,1}$ $=\min\{ \frac{1}{2}, \acute{b}_1\lambda_{\min}(P_2), \frac{\varepsilon_1}{2}\acute{p}e^{-\acute{a}_0}, \frac{\varepsilon_2\acute{h}}{2}, \frac{\acute{a}_1}{2\kappa_0}, \frac{\acute{a}_2}{2\kappa_0},\acute{b}_2\lambda_{\min}(P_3) \}$, $\theta_{3,2}=\max\{ \frac{1}{2}, \acute{b}_1\lambda_{\max}(P_2),\frac{\varepsilon_1}{2}\acute{p},\frac{\varepsilon_2\acute{h}}{2}e^{\acute{d}}, \frac{\acute{a}_1}{2\kappa_0},\frac{\acute{a}_2}{2\kappa_0},\acute{b}_1\lambda_{\max}(P_3) \}.$

Differentiating \eqref{V3(t)} in time and integrating by parts, we can arrive at
\begin{align}
	&\dot{V}_3(t)\leq (c_0-L_z)\widetilde{Y}(t)^2-\acute{b}_1\lambda_{\min}(Q_2)\lvert \widetilde{X}(t)\lvert^2 \notag \\
	&-\frac{\acute{b}_2}{2}\lambda_{\min}(Q_3)|\widetilde{d}(t)|^2+\frac{2\acute{b}_2|P_3L_d\acute{p}_2|^2}{\lambda_{\min}(Q_3)}|\widetilde{X}(t)|^2+\widetilde{Y}(t)\bigg[ \notag \\
	&\int_{0}^{1}\mathcal{G}_1(x)\widetilde{\alpha}(x,t)dx+\frac{L_z}{q_0}\widetilde{\alpha}(1,t)+\int_{0}^{1}\mathcal{G}_3(x)\widetilde{u}(x,t)dx \notag \\
	&-\int_{0}^{1}\frac{M(x)}{\varepsilon_2}N_2(x)dx\widetilde{X}(t)-\int_{0}^{1}\frac{M(x)}{\varepsilon_2}N_4(x)dx\widetilde{d}(t) \bigg] \notag \\
	&+\int_{0}^{1}\acute{p}_0e^{-\acute{a}_0x}\widetilde{\alpha}(x,t)\bigg(-\widetilde{\alpha}_x(x,t)+N_1(x)\widetilde{X}(t)\notag \\
	&+\int_{0}^{x}S_{11}(x,y)\widetilde{\alpha}(y,t)dy+\int_{0}^{x}S_{13}(x,y)\widetilde{u}(y,t)dy+ \notag \\
	&\mu_1(x)\widetilde{u}(x,t)+N_3(x)\widetilde{d}(t) \bigg)dx+\int_{0}^{1}\acute{h}e^{\acute{d}x}\widetilde{\beta}(x,t)\bigg( \widetilde{\beta}_x(x,t)\notag \\
	&+\int_{0}^{x}S_{21}(x,y)\widetilde{\alpha}(y,t)dy+\int_{0}^{x}S_{23}(x,y)\widetilde{u}(y,t)dy+ \notag \\
	&c_2(x)\widetilde{\alpha}(x,t)+\mu_2(x)\widetilde{u}(x,t)+N_2(x)\widetilde{X}(t)+N_4(x)\widetilde{d}(t) \bigg)dx \notag \\
	&+\int_{0}^{1}\acute{a}_1\widetilde{u}(x,t)\widetilde{u}_{xx}(x,t)dx+\int_{0}^{1}\acute{a}_2\widetilde{u}_x(x,t)\widetilde{u}_{xxx}(x,t)dx
\end{align}	where $c_0-L_z<0$ because of \eqref{eq:Lz}.

Applying Young's inequality, Cauchy-Schwarz inequality and Poincar$\acute{e}$ inequality, we get
\begin{align}
	&\dot{V}_3(t)\leq -\bigg(\acute{b}_1\lambda_{\min}(Q_2)-\frac{\acute{p}_0L_3}{2}-\frac{\acute{h}L_3}{2}e^{\acute{d}}-\frac{L_3}{4r_{10}}-\frac{1}{4}|a_1\acute{p}_2 \notag \\
	&+\frac{a_2}{\kappa_0}(\acute{p}_2(A-L_xC)-qL_d\acute{p}_2)|^2-\frac{2\acute{b}_2|P_3L_d\acute{p}_2|^2}{\lambda_{\min}(Q_3)} \bigg) \widetilde{X}(t)|^2 \notag \\
	&-\bigg( \frac{\acute{b}_2}{2}\lambda_{\min}(Q_3)-\frac{L_3}{4r_{11}}-\frac{1}{4}|a_1q+\frac{a_2}{\kappa_0}q(A_d-L_dq)|^2 \notag \\
	&-\frac{p_0L_3}{2}-\frac{\acute{h}L_3}{2}e^d \bigg)|\widetilde{d}(t)|^2-\int_{0}^{1}\acute{h}\bigg(\frac{\acute{d}}{2}-3L_3\bigg)\widetilde{\beta}^2(x,t)dx  \notag \\
	&-(L_z-c_0-r_5-(r_6+r_7+r_{10}+r_{11})L_3-\acute{h}e^{\acute{d}} q_0^2)\lvert \widetilde{Y}(t)\lvert^2\notag \\
	&-\int_{0}^{1}\bigg(\frac{1}{5}(\acute{a}_1-2)-\frac{\acute{p}_0L_3}{2\acute{a}}-\frac{\acute{p}_0L_3}{2}-\frac{\acute{h}L_3}{2\acute{d}}e^{\acute{d}}+\frac{\acute{h}L_3}{2\acute{d}} \notag \\
	& -\frac{\acute{h}L_3}{2}e^{\acute{d}}-\frac{L_3}{4r_7}  \bigg)\widetilde{u}^2(x,t)dx -\int_{0}^{1}\bigg(-\frac{\acute{h}L_3}{2\acute{d}}e^{\acute{d}}-\frac{\acute{h}L_3}{2}e^{\acute{d}} \notag \\
	&+(\frac{\acute{p}_0\acute{a}_0}{2}-\frac{\acute{p}_0L_3}{2\acute{a}_0}-2\acute{p}_0L_3)e^{-\acute{a}_0}+\frac{\acute{h}L_3}{2\acute{d}}-\frac{L_3}{4r_6}\bigg)\widetilde{\alpha}^2(x,t)dx \notag \\	&-\int_{0}^{1}\frac{1}{5}(\acute{a}_1-2)\widetilde{u}_x^2(x,t)dx-(\acute{a}_2-2)\int_{0}^{1}\widetilde{u}_{xx}^2(x,t)dx \notag \\
	&-\bigg(\frac{\acute{p}_0e^{-\acute{a}_0}}{2}-\frac{L_z^2}{4r_5q_0^2}-\acute{h}e^{\acute{d}}q_1^2\bigg)\widetilde{\alpha}(1,t)^2
\end{align}
where
 $	L_3=\max\{  \| S_{11}(x,y)\|_\infty,\| S_{13}(x,y)\|_\infty,\| S_{21}(x,y)\|_\infty,$ $\| S_{23}(x,y)\|_\infty, \lvert c_2(x)\lvert_\infty,|\mu_1(x)|_\infty,| \mu_2(x)|_\infty,| N_1(x)|_\infty,$ $ | N_2(x)|_\infty,| N_3(x)|_\infty,| N_4(x)|_\infty,\lvert \mathcal{G}_1(x)\lvert_\infty, \lvert\mathcal{G}_3(x)\lvert_\infty, | \frac{1}{\varepsilon_2}M(x)$ $N_2(x)|_\infty,| \frac{1}{\varepsilon_2}M(x)N_4(x)|_\infty \},
	$
and $c_0$ $-L_z+r_5+r_6L_3+r_7L_3+r_{10}L_3+\acute{h}e^{\acute{d}} q_0^2<0$ for sufficiently small $r_5$,$r_6$,$r_7$,$r_{10}$,$\acute{h}$. 

Choosing the positive $\acute{p}_0$, $\acute{d}$, $\acute{a}_0$, $\acute{a}_1$, $\acute{a}_2$, $\acute{b}_1$,  $\acute{b}_2$ to satisfy
\begin{align}
	\acute{a}_1>&\max\biggl\{ 2,5(\frac{\acute{p}_0L_3}{2\acute{a}}+\frac{\acute{p}_0L_3}{2}+\frac{\acute{h}L_3}{2\acute{d}}e^{\acute{d}}-\frac{\acute{h}L_3}{2\acute{d}}\notag \\
	&+\frac{\acute{h}L_3}{2}e^{\acute{d}}+\frac{L_3}{4r_7})+2 \biggr\}, \\
	\acute{a}_2>&2, \quad \acute{d}>6L_3, \quad \acute{a}_0>2L_3+\sqrt{4L_3^2+L_3}, \\
	\acute{p}_0>&\max\{ 2e^{\acute{a}_0}(\frac{L_z^2}{4r_5q_0^2}+\acute{h}e^{\acute{d}}q_1^2),e^{\acute{a}_0}(\frac{\acute{h}L_3}{2\acute{d}}e^{\acute{d}}-\frac{\acute{h}L_3}{2\acute{d}} \notag \\
	&+\frac{\acute{h}L_3}{2}e^{\acute{d}}+\frac{L_3}{4r_6})/(\frac{\acute{a}_0}{2}-\frac{L_3}{2\acute{a}_0}-2L_3) \},\\
	\acute{b}_1>&\bigg(\frac{\acute{p}_0L_3}{2}+\frac{\acute{h}L_3}{2}e^{\acute{d}}+\frac{L_3}{4r_{10}}+\frac{2\acute{b}_2|P_3L_d\acute{p}_2|^2}{\lambda_{\min}(Q_3)} \notag \\
	&+\frac{a_2}{\kappa_0}(\acute{p}_2(A-L_xC)-qL_d\acute{p}_2)|^2 \bigg)/\lambda_{\min}(Q_2) , \\
	\acute{b}_2 >&2\bigg( \frac{L_3}{4r_{11}}+\frac{1}{4}|a_1q+\frac{a_2}{\kappa_0}q(A_d-L_dq)|^2+\frac{p_0L_3}{2} \notag \\
	&+\frac{\acute{h}L_3}{2}e^d \bigg)/\lambda_{\min}(Q_3),
\end{align} 
we arrive at
\begin{align}
	\dot{V}_3(t)&\leq -\lambda_3\theta_{3,2}\Omega_3(t)+(-\frac{\acute{p}_0e^{-\acute{a}_0}}{2}+\frac{L_z^2}{4r_5q_0^2}+\acute{h}e^{\acute{d}})\widetilde{\alpha}(1,t)^2
\end{align}
for some positive $\lambda_3$. 

Consider another Lyapunov function
	\begin{align}
		V_5(t)=&\frac{\varepsilon_1}{2}\int_{0}^{1}B_3e^{-\delta_3x}\widetilde{\alpha}_x^2(x,t)dx+\frac{\varepsilon_2}{2}\int_{0}^{1}e^{\delta_3x}\widetilde{\beta}_x^2(x,t)dx \notag \\
		&+\mathcal{R}_3V_3(t) \label{V5(t)}
	\end{align}
	where $B_3, \delta_3$ are positive. Defining
	\begin{align}
		\Omega_5(t)=&\| \widetilde{\alpha}(\cdot,t)\|^2_{H^1}+\| \widetilde{\beta}(\cdot,t)\|^2_{H^1}+\| \widetilde{u}(\cdot,t)\|^2_{H^1}+| \widetilde{Y}(t)|^2 \notag \\
		&+| \widetilde{X}(t)|^2+| \widetilde{d}(t)|^2,
\end{align}
we have $\theta_{5,1}\Omega_5(t) \leq V_5(t)\leq\theta_{5,2}\Omega_5(t)$ for some positive $\theta_{5,1}$, $\theta_{5,2}$. Taking the derivative of \eqref{V5(t)}, we then get
\begin{align}
		&\dot{V}_5(t)\leq -\bigg( \mathcal{R}_3\acute{\beta}_6-\frac{\bar{L}_3}{2}e^{\delta_3}-\frac{B_3\bar{L}_5}{2}\bigg)\int_{0}^{1}\widetilde{u}_x^2(x,t)dx  \notag \\
		&-\bigg( \mathcal{R}_3\acute{\beta}_2-\frac{\bar{L}_3}{2}e^{\delta_3}-\frac{B_3}{2}N_1(0)-\frac{B_3\bar{L}_5}{2}-\frac{e^{\delta_3}}{2}\kappa_3 \bigg)| \widetilde{X}(t)|^2 \notag \\
		&-\bigg( \frac{B_3e^{-\delta_3}}{2}-\frac{e^{\delta_3}}{2}\kappa_3 \bigg)\widetilde{\alpha}_{x}(1,t)^2-\bigg( \mathcal{R}_3\acute{\beta}_8-\frac{\bar{L}_3e^{\delta_3}}{2}\notag \\
		&-\frac{\bar{L}_3}{2\delta_3}e^{\delta_3}-\frac{e^{\delta_3}}{2}\kappa_3-\frac{B_3\bar{L}_5}{2}-\frac{B_3\bar{L}_5}{2\delta_3}\bigg)\int_{0}^{1}\widetilde{\alpha}^2(x,t)dx \notag \\
		&-\bigg((\frac{B_3\delta_3}{2}-\frac{7}{2}B_3\bar{L}_5)e^{-\delta_3}-\frac{\bar{L}_3}{2}e^{\delta_3} \bigg)\int_{0}^{1}\widetilde{\alpha}_x^2(x,t)dx \notag \\
		& -\bigg(\mathcal{R}_3\acute{\beta}_5-\frac{\bar{L}_3}{2\delta_3}e^{\delta_3}-\frac{\bar{L}_3}{2}-\frac{e^{\delta_3}}{2}\kappa_3\bigg)\int_{0}^{1}\widetilde{u}^2(x,t)dx \notag \\
		&-\bigg(\mathcal{R}_3\acute{\beta}_1-\frac{e^{\delta_3}}{2}\kappa_3 \bigg)| \widetilde{Y}(t)|^2-\bigg(\mathcal{R}_3\acute{\beta}_3-\frac{e^{\delta_3}}{2}\kappa_3\bigg)\widetilde{\alpha}(1,t)^2 \notag \\
		&-\mathcal R_3 \acute{\beta}_7\int_{0}^{1}\widetilde{u}_{xx}^2(x,t)dx-\int_{0}^{1}\mathcal{R}_3\acute{h}\bigg(\frac{\acute{d}}{2}-3L_3\bigg)\widetilde{\beta}^2(x,t)dx  \notag \\
		&-\bigg(\frac{\delta_3}{2}-4\bar{L}_3 \bigg)\int_{0}^{1}e^{\delta_3x}\widetilde{\beta}_x^2(x,t)dx-\frac{1}{2}\widetilde{\beta}_x(0,t)^2 \notag \\
		&-\bigg( \mathcal R_3\acute{\beta}_9-\frac{B_3\bar{L}_5}{2}-\frac{\bar{L}_3}{2}-\frac{e^{\delta_3}}{2}\kappa_3 \bigg)|\widetilde{d}(t)|^2  \label{dV5}
\end{align}
where $\acute{\beta}_1=-c_0+L_z-r_5-(r_6+r_7+r_{10}+r_{11})L_3-he^{\acute{d}} q_0^2$, $\acute{\beta}_2=\acute{b}_1\lambda_{\min}(Q_2)-\frac{\acute{p}_0L_3}{2}-\frac{\acute{h}L_3}{2}e^{\acute{d}}
	-\frac{L_3}{4r_{10}}-\frac{1}{4}|a_1\acute{p}_2+\frac{a_2}{\kappa_0}(\acute{p}_2(A-L_xC)-qL_d\acute{p}_2)|^2-\frac{2\acute{b}_2|P_3L_d\acute{p}_2|^2}{\lambda_{\min}(Q_3)}$, $\acute{\beta}_3=\frac{\acute{p}_0e^{-\acute{a}}}{2}-\frac{L_z^2}{4r_5q_0^2}-\acute{h}e^{\acute{d}}q_1^2$, $\acute{\beta}_5=\frac{1}{5}(\acute{a}_1-2)+\frac{\acute{h}L_3}{2\acute{d}}-\frac{\acute{p}_0L_3}{2\acute{a}}-\frac{\acute{p}_0L_3}{2}-\frac{\acute{h}L_3}{2\acute{d}}e^{\acute{d}}-\frac{\acute{h}L_3}{2}e^{\acute{d}}-\frac{L_3}{4r_7}$, $\acute{\beta}_6=\frac{1}{5}(\acute{a}_1-2)$, $\acute{\beta}_7=\acute{a}_2-2$, $\acute{\beta}_8=(\frac{\acute{p}_0\acute{a}}{2}-\frac{\acute{p}_0L_3}{2\acute{a}}-2\acute{p}_0L_3)e^{-\acute{a}}-\frac{\acute{h}L_3}{2\acute{d}}e^{\acute{d}}+\frac{\acute{h}L_3}{2\acute{d}}-\frac{\acute{h}L_3}{2}e^{\acute{d}}-\frac{L_3}{4r_6}$, $\acute{\beta}_9=\frac{\acute{b}_2}{2}\lambda_{\min}(Q_3)-\frac{L_3}{4r_{11}}-\frac{1}{4}|a_1q+\frac{a_2}{\kappa_0}q(A_d-L_dq)|^2-\frac{p_0L_3}{2}-\frac{\acute{h}L_3}{2}e^d$, and where
\begin{align}
	\bar{L}_3=&\max\{ \lvert S_{21}(x,x)+c_2^{(1)}(x)\lvert_{\infty}, \lvert c_2(x)\lvert_{\infty}, \lvert \mu_2(x) \lvert_{\infty},\notag \\
	&\lvert \mu_2^{(1)}(x)+S_{23}(x,x) \lvert_{\infty}, \lvert N_2^{(1)}(x) \lvert_{\infty},\lvert N_4^{(1)}(x) \lvert_{\infty}, \notag \\
	& \| S_{21x}(x,y)\|_{\infty}, \| S_{23x}(x,y\|_{\infty}) \}, \\
	\bar{L}_5=&\max\{ \lvert S_{11}(x,x) \lvert_{\infty}, \lvert S_{13}(x,x)+\mu_1^{(1)}(x) \lvert_{\infty}, \lvert \mu_1(x) \lvert_{\infty},\notag \\
	& \lvert N_1^{(1)}(x) \lvert_{\infty},\lvert N_3^{(1)}(x) \lvert_{\infty}, \| S_{11x}(x,y)\|_{\infty}, \| S_{13x}(x,y)\|_{\infty} \} .
\end{align}
The inequalitie $\widetilde{\beta}_x(1,t)^2\leq \kappa_3(\widetilde{\alpha}_x(1,t)^2+\widetilde{\alpha}(1,t)^2+\int_{0}^{1}\widetilde{\alpha}^2(x,t)dx$ $+\int_{0}^{1}\widetilde{u}^2(x,t)dx+\lvert \widetilde{Y}(t)\lvert^2$ $+\lvert \widetilde{X}(t)\lvert^2+|\widetilde{d}(t)|^2)$ which is obtained from \eqref{eretab} and their time or spatial derivatives, has been used in obtaining \eqref{dV5}.

By choosing 
\begin{align}
	B_3>\max\bigg\{\frac{\bar{L}_3e^{2\delta_3}}{\delta_3-7\bar{L}_5}, e^{2\delta_3}\kappa_3 \bigg\},\quad\delta_3>\max\{7\bar{L}_5,8\bar{L}_3\},
\end{align}
and sufficiently large $\mathcal{R}_3$, we get
\begin{align}
	\dot{V}_5(t)\leq -\lambda_5 \theta_{5,2}\Omega_5(t)\leq -\lambda_5V_5(t)
\end{align}	
for some positive $\lambda_5$ and where
\begin{align}
		\lambda_5=&\min\bigg\{ \frac{1}{2}\mathcal{R}_3\acute{\beta}_1,  \frac{1}{2}\mathcal{R}_3\acute{\beta}_2,  \frac{1}{2}\mathcal{R}_3\acute{\beta}_4, \frac{1}{2}\mathcal{R}_3\acute{\beta}_5,\frac{1}{2}\mathcal{R}_3\acute{\beta}_6, \notag \\
		& \frac{1}{2}\mathcal{R}_3\acute{\beta}_8,\frac{1}{2}\mathcal{R}_3\acute{\beta}_9,(\frac{B_3\delta_3}{2}-\frac{7}{2}B_3\bar{L}_5)e^{-\delta_3}-\frac{\bar{L}_3}{2}e^{\delta_3}, \notag \\
		& -4\bar{L}_3+\frac{\delta_3}{2},\mathcal{R}_3\acute{h}(\frac{\acute{d}}{2}-3L_3) \bigg\}/\theta_{5,2}. \label{lambda5}
\end{align}
Then through the same steps in Theorem $\ref{tosb}$, defining
\begin{align}
	\Xi_3(t)=&\| \widetilde{\eta}(\cdot,t)\|^2_{H^1}+\| \widetilde{\xi}(\cdot,t)\|^2_{H^1}+\| \widetilde{u}(\cdot,t)\|^2_{H^1}+| \widetilde{z}(t)|^2 \notag \\
	&+| \widetilde{X}(t)|^2+| \widetilde{d}(t)|^2, \label{Xi1} 
\end{align}
we obtain
\begin{align}
	\Xi_3(t)\leq \frac{\theta_{5,2}\theta_{5,4}}{\theta_{5,1}\theta_{5,3}}\Xi_3(0)e^{-\lambda_5 t}
\end{align}
for some positive parameters $\theta_{5,3},\theta_{5,4}$ that are obtained by applying Cauchy-Schwarz inequality into the transformations \eqref{z2Y}--\eqref{ettoa}.

Then \eqref{Ome} is obtained with
\begin{align}
	\varUpsilon_e=\frac{\theta_{5,1}\theta_{5,3}}{\theta_{5,2}\theta_{5,4}}, \quad \acute{C}_e=\lambda_5 \label{Upe}
\end{align}
where $\lambda_5$ is given by \eqref{lambda5}. The proof of Theorem $\ref{ertg}$ is complete.
\subsection{Expression of $\mathcal{G}_1$, $\mathcal{G}_3$, $S_{11}$, $S_{13}$, $S_{21}$, $S_{23}$, $N_1$, $N_2$, $N_3$, $N_4$}\label{S11}
\setcounter{equation}{0}
\renewcommand{\theequation}{G.\arabic{equation}}
The functions $\mathcal{G}_1(x)$, $\mathcal{G}_3(x)$, $S_{11}(x,y)$, $S_{13}(x,y)$, $S_{21}(x,y)$, $S_{23}(x,y)$, $N_1(x)$, $N_2(x)$, $N_3(x)$, $N_4(x)$ in target observer error system are shown as follows
\begin{align}
	\mathcal{G}_{1}(x)=&-\frac{1}{\varepsilon_2}M(x)c_2(x)-\frac{1}{\varepsilon_2}\int_{x}^{1}S_{21}(y,x)M(y)dy, \\ \mathcal{G}_{3}(x)=&-\frac{1}{\varepsilon_2}M(x)\mu_2(x)-\frac{1}{\varepsilon_2}\int_{x}^{1}S_{23}(y,x)M(y)dy, \\
	S_{11}(x,y)=&f_{11}(x,y)-\frac{\varepsilon_1}{\varepsilon_2}\int_{y}^{x}S_{21}(z,y)\phi(x,z)dz\notag \\
	&-\frac{\varepsilon_1}{\varepsilon_2}c_2(y)\phi(x,y), \\
	S_{13}(x,y)=&f_{13}(x,y)-\frac{\varepsilon_1}{\varepsilon_2}\int_{y}^{x}S_{23}(z,y)\phi(x,z)dz\notag \\
	&-\frac{\varepsilon_1}{\varepsilon_2}\mu_2(y)\phi(x,y), \\
	S_{21}(x,y)=&f_{21}(x,y)-\int_{y}^{x}S_{21}(z,y)\psi(x,z)dz\notag \\
	&-c_2(y)\psi(x,y), \\
	S_{23}(x,y)=&f_{23}(x,y)-\int_{y}^{x}S_{23}(z,y)\psi(x,z)dz\notag \\
	&-\mu_2(y)\psi(x,y), \\
	N_1(x)=&D_1(x)-\frac{\varepsilon_1}{\varepsilon_2}\int_{0}^{x}\phi(x,y)N_2(y)dy, \\ 
	N_2(x)=&D_2(x)-\int_{0}^{x}\psi(x,y)N_2(y)dy, \\
	N_3(x)=&G_1(x)-\frac{\varepsilon_1}{\varepsilon_2}\int_{0}^{x}\phi(x,y)N_4(y)dy, \\ 
	N_4(x)=&G_2(x)-\int_{0}^{x}\psi(x,y)N_4(y)dy.
\end{align}

\end{document}